\documentclass{article} % For LaTeX2e
\usepackage{iclr2025_conference,times}

\usepackage[utf8]{inputenc} % allow utf-8 input
\usepackage[T1]{fontenc}    % use 8-bit T1 fonts
\usepackage{url}            % simple URL typesetting
\usepackage{booktabs}       % professional-quality tables
\usepackage{amsfonts}       % blackboard math symbols
\usepackage{nicefrac}       % compact symbols for 1/2, etc.
\usepackage{microtype}      % microtypography
\usepackage{xcolor}         % colors
\usepackage{floatrow}
\usepackage{amsmath,amsthm}
\usepackage{algorithm2e} % algorithm2e
\usepackage{wrapfig, booktabs}
\usepackage{multirow}
\usepackage{dsfont}
\usepackage{caption}
\usepackage{subcaption}
\usepackage{graphicx}
\usepackage{multicol}
\usepackage{enumitem}

\SetKwSty{textnormal}       % Make keywords (\For, \If, etc.) non-italic
\SetArgSty{textnormal}      % Make arguments non-italic
\SetCommentSty{textnormal}  % Make comments non-italic

\usepackage{xcolor}
\usepackage{tcolorbox}
\tcbuselibrary{skins}
\usepackage{inconsolata}    % Monospaced font
\usepackage[T1]{fontenc}
\usepackage[utf8]{inputenc}

\definecolor{softGray}{RGB}{240,240,240}

\definecolor{myDarkBlue}{RGB}{26,77,143}  % A pleasing navy-ish blue
\definecolor{myDarkGreen}{RGB}{19,113,119} % A nice teal-green
\definecolor{darkpastelblue}{rgb}{0.47, 0.62, 1}
\usepackage[hidelinks,colorlinks=true,linkcolor=purple,citecolor=darkpastelblue]{hyperref}
\usepackage{cleveref}
\newtheorem{theorem}{Theorem}

\usepackage{amsmath,amsfonts,bm}

\def\eqref#1{equation~\ref{#1}}
\def\1{\bm{1}}

\DeclareMathAlphabet{\mathsfit}{\encodingdefault}{\sfdefault}{m}{sl}
\SetMathAlphabet{\mathsfit}{bold}{\encodingdefault}{\sfdefault}{bx}{n}

\newcommand{\mpar}{\Theta}

\usepackage{url}

\title{Learning a Neural Solver for Parametric PDEs \\ to Enhance Physics-Informed Methods}

\author{Lise Le Boudec \textsuperscript{1} \hspace{2pt} \thanks{Corresponding author: \url{lise.leboudec@isir.upmc.fr}.} \\
\And Emmanuel de Bezenac \textsuperscript{2}\\
\And Louis Serrano \textsuperscript{1}\\
\And Ramon Daniel Regueiro-Espino \textsuperscript{1}\\
\And Yuan Yin \textsuperscript{3} \thanks{Work done during post-doc at Sorbonne University.} \\
\And Patrick Gallinari \textsuperscript{1, 4} \\
\AND \\
\textsuperscript{1} Sorbonne Université, CNRS, ISIR, F-75005 Paris, France \\
\textsuperscript{2} ARCHES, INRIA Paris, France \\ 
\textsuperscript{3} Valeo.ai, Paris, France \\ 
\textsuperscript{4} Criteo AI Lab, Paris, France \\ 
}

\iclrfinalcopy % Uncomment for camera-ready version, but NOT for submission.
\begin{document}

\maketitle

\begin{abstract}
% Physics-informed deep learning often faces optimization challenges due to the complexity of solving partial differential equations (PDEs), which can lead to unstable training and require numerous iterations to achieve accurate solutions. To address these issues, we propose to solve PDEs using a physics-informed iterative algorithm trained from data. This approach accelerates and stabilizes the optimization process, enabling faster convergence of physics-aware models. 
% Moreover, our method integrates the physical loss gradient and the PDE parameters to solve on parameters distribution of the PDE, including the PDE coefficient, initial or boundary condition. We demonstrate the effectiveness of our methods through empirical experiments on multiple datasets, comparing training and test-time optimization performance. 

Physics-informed deep learning often faces optimization challenges due to the complexity of solving partial differential equations (PDEs), which involve exploring large solution spaces, require numerous iterations, and can lead to unstable training. These challenges arise particularly from the ill-conditioning of the optimization problem caused by the differential terms in the loss function. To address these issues, we propose learning a solver, i.e., solving PDEs using a physics-informed iterative algorithm trained on data. Our method learns to condition a gradient descent algorithm that automatically adapts to each PDE instance, significantly accelerating and stabilizing the optimization process and enabling faster convergence of physics-aware models. Furthermore, while traditional physics-informed methods solve for a single PDE instance, our approach extends to parametric PDEs. Specifically, we integrate the physical loss gradient with PDE parameters, allowing our method to solve over a distribution of PDE parameters, including coefficients, initial conditions, and boundary conditions. We demonstrate the effectiveness of our approach through empirical experiments on multiple datasets, comparing both training and test-time optimization performance.
The code is available at \url{https://github.com/2ailesB/neural-parametric-solver}. %Project page: \url{}. 

\end{abstract}

\section{Introduction}
\label{sec:intro}
Partial Differential Equations (PDEs) are ubiquitous as mathematical models of dynamical phenomena in science and engineering. Solving PDEs is of crucial interest to researchers and engineers, leading to a huge literature on this subject \citep{PDE_evans10, PDEact_Salsa}. %While certain differential equations can be solved analytically, there are still other solutions that remain unknown. This has lead to an other field of mathematics called numerical analysis that applies tools from scientific computing to solve mathematical problems \citep{NumMaths_Quarteroni06, Alazard_ToolsPbPDE}. %For differential equation, algorithms to computes numerical and approached solutions have been developed through solvers. 
Traditional approaches to solving PDEs such as finite difference, finite element analysis, or spectral methods \citep{zienkiewicz2005finite, leveque2007finite} often come with stability and convergence guarantees but suffer from a high computational cost. Improving numerical PDE solvers through faster and more accurate algorithms remains an active research topic \citep{zienkiewicz2005finite}.  % spectral, multigrid

%leading to an interest in machine learning (ML) based alternatives \citep{KarRev}. Several methods have been proposed, focusing on solving PDEs by learning a solver and/or the solution through data and/or some prior physical knowledge on the PDE. 
% Deep Learning can also help with this task by learning solver for PDEs from data and benefit to this task by automating the discovery of some PDE solvers. 
%Depending on the type of problem to be solved, several algorithms have developed such as Gradient Descent (GD), Runge-Kutta methods (RK), Finite Element Method (FEM) \citep{Süli_Mayers_2003_IntroNumAnalysis, s2012numericalAnPDE} and many others. If these mathematical formulation, often come with convergence guarantees or interesting properties, they suffer from a very high computational cost to solve complex systems. 
% these methods iteratively refine an initial solution towards greater accuracy. Some example include like Jacobi, Gauss-Seidel, and Krylov subspace methods. Given the complexity and ill-conditioned nature of many PDEs, these iterative processes can demand extensive computational resources. 

%To date, ML-based approaches to solving PDEs 

PDE solvers usually rely on discretization and/or linearization of the problem through various techniques to simplify the computations. Iterative methods such as Jacobi, Gauss-Seidel, Conjugate Gradient, and Krylov subspace methods can then be used to solve the resulting systems. Unfortunately, many PDEs have an ill-conditioned nature, and these iterative processes can demand extensive computational resources. Preconditioning techniques are often essential to mitigate this, though they require precise customization to the specific PDE problem, making the development of effective solvers a significant research endeavor in itself. Yet, the computational demands, time, and expertise required to develop these algorithms sometimes make them impractical or sub-optimal for specific classes of problems. Instead of relying on hand-designed algorithms, researchers have investigated, as an alternative, the use of machine learning to train iterative PDE solvers \cite{hsieh2019learningNPDESconv, li23e-preconcgpde, rudikov2024fgcno, kopanivcakova2023enhancing}. These approaches usually parallel the classical numerical methods by solving a linear system resulting from the discretization of a PDE, for example, using finite differences or finite elements. A preconditioner is learned from data by optimizing a residual loss computed w.r.t. a ground truth solution obtained with a PDE solver. This preconditioner is used on top of a baseline iterative solver and aims at accelerating its convergence. Examples of baseline solvers are the conjugate gradient \citep{li23e-preconcgpde, rudikov2024fgcno} or the Jacobi method \cite{hsieh2019learningNPDESconv}. 
% We will follow a similar idea, but our model deviates from traditional preconditioning methods by acting directly on the cost function of a Physics-Informed Neural Network (PINN) without going through the discretization steps \textit{i.e.}, we focus on directly optimizing the non-linear residual loss of the PDE. % This avoids to reduce solving the PDE to solving linear systems.  % XXXest ce vraiment la différence ??? Il faudrait insister sur la différence XXXXXX

Another recent research direction investigates the use of neural networks for building surrogate models in order to accelerate the computations traditionally handled by numerical techniques. These methods fall into two main categories: \textit{supervised} and \textit{unsupervised}. The \textit{supervised} methodology consists of first solving the PDE using numerical methods to generate input and target data and then regressing the solution using neural networks in the hope that this surrogate could solve new instances of the PDE. Many models, such as Neural Operators, lie within this class \citep{li2020fno, cno, reno} and focus on learning the solution operator directly through a single neural network pass. 
%This method is very efficient, at the downside of relying on quite large quantities of data for training in order to ensure adequate generalization. Additionally, the neural network does not have access to the PDE is in itself never used, only indirectly through the data.
\textit{Unsupervised} approaches, involve considering a neural network as a solution of the PDE. The neural network parameters are found by minimizing the PDE residual with gradient descent. Methods such as Physics-Informed Neural Networks (PINNs) \citep{PINNs_Raissi19}, or DeepRitz \citep{DR1} fall under this category. This family of methods is attractive as it does not rely on any form of data, but only on information from the PDE residual. However, they exhibit severe difficulties during training \citep{krishnapriyan_characterizing, deryckManu2023PreconditioningPINNs}, often requiring many optimization steps and sophisticated training schemes \citep{krishnapriyan_characterizing,RathoreICML2024}. The ill-conditioned nature of PDE residual loss appears again in this context, making standard optimizers such as Adam inappropriate (see \cref{app:losslandscape} and \cref{ssec:app_vismethod} for a visualization of the ill conditioning of this loss landscape). A detailed review of the existing literature is described in \cref{app:rw}.
%This means that is holds the promise to generalize to any PDE, initial, and boundary conditions. 

In this work, we consider having access to the PDE as in unsupervised approaches and also to some data for training our neural solver. Our objective is to solve the optimization issues mentioned above by \textit{learning an iterative algorithm that solves the PDE} from its residual, defined as in the PINNs framework (see \cref{fig:idea}). This \textit{neural solver} is trained from data, either simulations or observations. Different from the classical ML training problem which aims at learning the parameters of a statistical model from examples, the problem we handle is learning to learn, \cite{andrychowicz2016l2lbygdbygd}, i.e. learning an iterative algorithm that will allow us to solve a learning problem. When vanilla PINNs handle a single PDE instance, requiring retraining for each new instance, we consider the more complex setting of solving parametric PDEs, the parameters may include boundary or initial conditions, forcing terms, and PDE coefficients. Each specific instance of the PDE, sampled from the PDE parameter distribution, will then be considered as a training example. The objective is then to learn a solver from a sample of the parametric PDE distribution in order to accelerate inference on new instances of the equation.  With respect to unsupervised approaches, our model implements an improved optimization algorithm, tailored to the parametric PDE problem at hand instead of using a hand-defined method such as stochastic gradient descent (SGD) or Adam. As demonstrated in the experimental section, this approach proves highly effective for the ill-posed problem of optimizing the PINNs objective, enabling convergence in just a few steps. This is further illustrated through gradient trajectory visualizations in \cref{app:losslandscape}. In the proposed methodology, the neural solver will make use of the gradient information computed by a baseline gradient method to accelerate its convergence. In our instantiation, we will use SGD as our baseline algorithm, but the method could be easily extended to other baselines. Our model deviates from the traditional preconditioning methods by directly optimizing the non-linear PDE residual loss \citep{PINNs_Raissi19} without going through the discretization steps. % \textit{i.e.}, we focus on directly optimizing the non-linear residual loss of the PDE.
%we introduce a neural network-based PDE solvers, designed to address the later issues of physics-informed methods.
%The hope is that by learning the algorithm rather that directly predicting the solution as in the supervised methods, we obtain improved generalization, and a smaller number of iterations w.r.t. unsupervised approaches. 
%This framework is somewhat different from both classes of methods discussed above: With respect to supervised approaches, instead of directly learning to predict the solution, we learn the iterative algorithm that solves the PDE. As it has access to the PDE, the hope is that it is able to assess whether the ansatz is close to the solution and refine its predictions in the following steps, as opposed to one step (i.e. single forward of the model). Instead of learning to predict the solution, it learns to solve the PDE. Our intuition is that the proposed algorithm may require less data w.r.t. these methods as it has access to the PDE. With respect to unsupervised approaches where our model %$\mathcal{A}$ would correspond to the optimization algorithm e.g. SGD or Adam. Our updates are learned using a neural network, thus requiring much less iterations.
Our contribution includes :
\begin{itemize}
    \item Setting an optimization framework for learning to iteratively solve parametric PDEs from physics-informed loss functions. We develop an instantiation of this idea using an SGD baseline formulation. We detail the different components of the framework as well as training and inference procedures.  % (\cref{ssec:instantiation})
    \item Evaluating this method on challenging PDEs for physics-informed method, including failure cases of classical PINNs and showing that it solves the associated optimization issues and accelerates the convergence. 
    \item Extending the comparison to several parametric PDEs with varying parameters from $1$d static to $2$d+time problems. We perform a comparison with  baselines demonstrating a significant acceleration of the convergence w.r.t. baselines. % Finally, we conduct experiments to highlight some key components. 
\end{itemize}

\begin{figure}[htbp]
    \centering
    \includegraphics[width=0.8\textwidth, trim={0.5cm 5cm 1.cm 0.5cm}, clip]{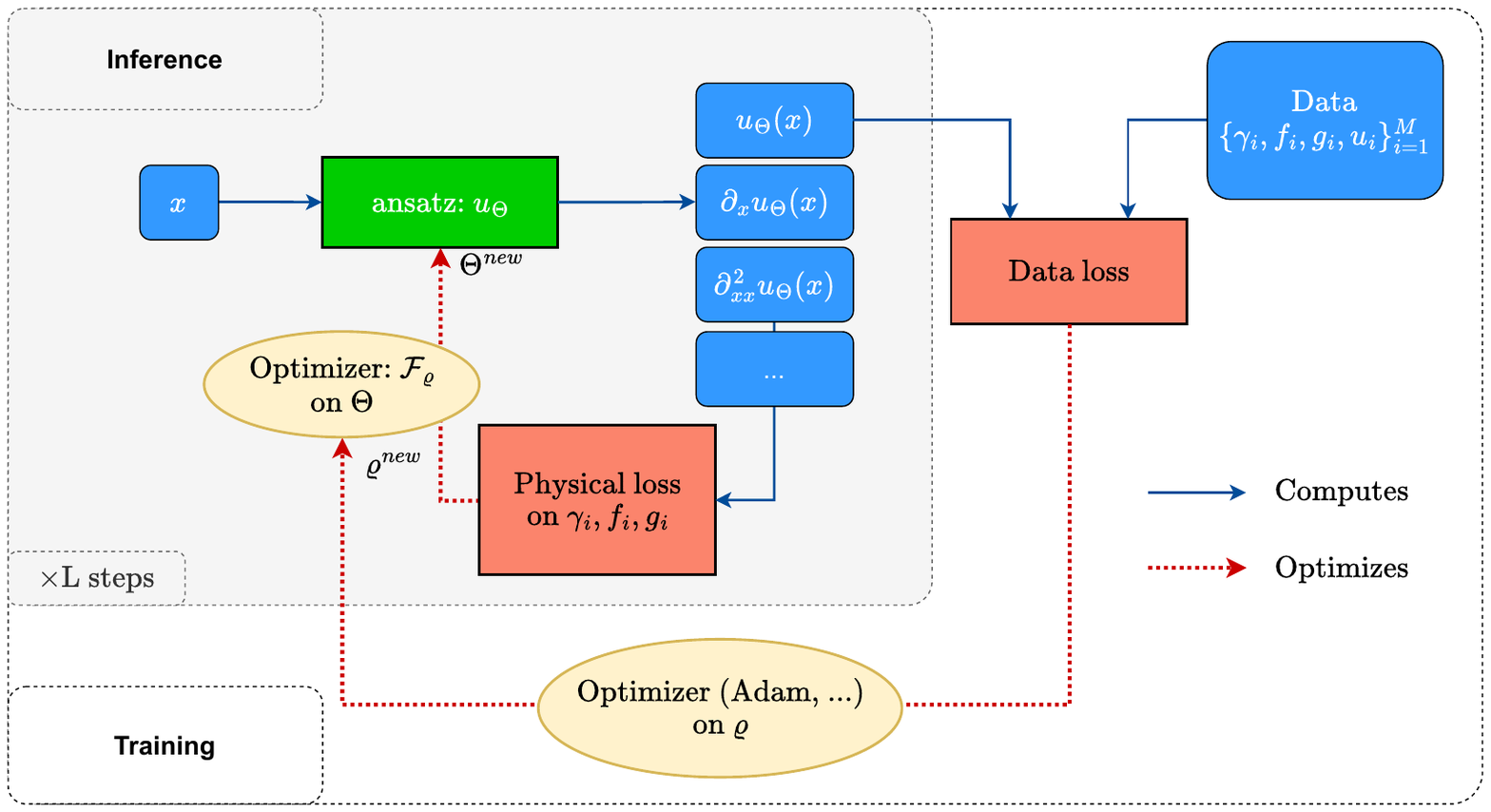}
    \caption{Optimization scheme of a physics-informed method with our framework. }
    \label{fig:idea}
\end{figure}

\section{Motivation\label{sec:motivation}}

Let us first motivate our objective with a simple example. Physics-informed neural networks (PINNs) are a promising tool for optimizing neural networks in an unsupervised way to solve partial differential equations (PDEs). However, these methods are notoriously difficult to train \citep{krishnapriyan_characterizing, deryckManu2023PreconditioningPINNs}. As an illustrative example of this challenge, let us solve the zero-boundary Poisson equation in $1$d on $\Omega = [-\pi, \pi]$. Note that this section is intentionally informal, we provide rigorous statements and proofs in Appendix \ref{app:proof}.

\paragraph{Poisson equation, 1d.} The solution is given by $u(x)= \sin(k x)$.
\begin{align}
  &u''(x) = -k^2 \sin(kx), \notag \\
  &u(-\pi) = 0, \, u(\pi) = 0.
  \label{eq:poisson_zero_bcs}
\end{align}
Physics-informed machine learning relies on an \emph{ansatz space} of \emph{parametric functions}, e.g. neural networks $u_\mpar:\Omega \mapsto \mathbb{R}$, minimizing the following loss in order to satisfy the constraints in \eqref{eq:poisson_zero_bcs}:
\begin{equation*}
    \mathcal{L}_{\textnormal{PDE}} = \mathcal{L}_{_{\textnormal{Res}}} + \lambda\mathcal{L}_{\textnormal{BC}}, \quad \mathcal{L}_{_{\textnormal{Res}}} = \int_{\Omega} \left| u_{\mpar}''(x) - f(x) \right|^2 dx , 
\end{equation*}
\begin{equation*}
    \mathcal{L}_{\textnormal{BC}} = \frac{1}{2} \left[ u_{\mpar}(-\pi) ^2 + u_{\mpar}(\pi)^2 \right].
\end{equation*}

As a simple example, consider the parametrization given by considering a linear combination of Fourier features widely used \citep{FF}\footnote{Note that even though the ansatz is linear in $\mpar$, it is not linear in $x$.}, $u_\mpar(x) = \sum_{k=-K}^K \theta_k\phi_k(x)$, with $\phi_0(x)= \tfrac{1}{\sqrt{2\pi}}$, $\phi_{-k}(x)=\tfrac{1}{\sqrt{\pi}}\cos(kx)$ and $\phi_k(x) = \tfrac{1}{\sqrt{\pi}}\sin(kx)$ for $1\leq k\leq K$. 

This simple but informative example yields a tractable gradient descent algorithm, as the associated updates are linear in the parameters, governed by a matrix $A$ and constant $b$:
\begin{align}
    \mpar_{l+1} &= \mpar_l - \eta \nabla \mathcal{L}_{\textnormal{PDE}}(\mpar_l)\notag \\
    &= (I - \eta A) \mpar_l + b
    \label{eq:gradient_descent}
\end{align}

with $A$ whose condition number is $\kappa(A) := {\lambda_{\max}(A)}/{\lambda_{\min}(A)} \geq K^4$:
\begin{align}
    A &= \begin{bmatrix}
    0^4 & 0 & \cdots & 0 \\
    0 & 1^4 & \cdots & 0 \\
    \vdots & \vdots & \ddots & \vdots \\
    0 & 0 & \cdots & K^4
    \end{bmatrix} + \lambda 
    \begin{bmatrix}
    \phi_1(\pi) \\ 
    \phi_2(\pi) \\ 
    \vdots \\ 
    \phi_K(\pi)
    \end{bmatrix}
    \begin{bmatrix}
    \phi_1(\pi) & \phi_2(\pi) & \cdots & \phi_
    K(\pi)
    \end{bmatrix}.
\label{eq:A}
\end{align}

This implies that the condition number of $A$ increases extremely rapidly in the ratio between the highest and lowest frequencies of the network. Given that the rate of convergence to the optimum $\mpar^\ast = \mpar_0 + A^{-1}b$ can be bounded as
\begin{align}
        \|\mpar_l - \mpar^\ast\|_2 \leq \left(1-{c}/{\kappa(A)}\right)^{l}\| \mpar_0 - \mpar^\ast \|_2,
\end{align}

the number of steps $N(\varepsilon)$ required to obtain an error of size at most $\varepsilon$, i.e., $ \|\mpar_l - \mpar^\ast\|_2 \leq \varepsilon$ increases linearly in the condition number, i.e. as the fourth power of the maximal frequency $K$:
\begin{equation}
\label{eq:nsteps}
N(\varepsilon) = O\left(\kappa(A)\ln \tfrac{1}{\epsilon}\right) = O\left(K^4\ln \tfrac{1}{\epsilon}\right).
\end{equation}

We believe that this simple example clearly illustrates and highlights the fact that PINNs--even when considering a linear basis, and when the PDE is linear--suffer heavily from ill-conditioning: if $500$ steps are required in order to achieve a given error when $K=5$, roughly speaking, $312\:500$ steps are required for only $K=25$. This result extends to more general linear systems of equations and linear ansatz, as explained in \cref{app:proof}. %This is illustrated in \cref{app:losslandscape}. 

Our objective in the following will be to accelerate the convergence of such systems in this context--as well as extend them to the non-linear setting. To do so, in the following section, we will learn how to transform the optimization problem in such a way that the number of gradient descent iterations is small. The resulting method can be seen as a standalone, iterative solver as it is not only applicable to different PDEs but can handle a wide range of initial/boundary conditions and parameters.

\section{Approach}

In order to optimize PDE-based losses, we propose to learn a physics-based optimizer that will fulfill two objectives: (i) allowing a fast test-time optimization given a new PDE and (ii) solving without retraining parametric PDEs, with varying PDE coefficients $\gamma$ \footnote{Note that PDE coefficients can be functions, an example is the Darcy PDE in \cref{subsec:datasets}}, forcing terms $f$, and initial/boundary conditions $g$ using the same model. We present the general framework below and propose an instantiation that leverages a linear combination of basis functions as the ansatz. 

\subsection{Problem statement\label{sec:problem_statement}}

Let us consider the following family of boundary value problems parameterized by $\gamma$ with domain $\Omega$, representing both space and time, with $\mathcal{N}$ a potentially nonlinear differential operator, $\mathcal{B}$ the boundary operator, $g$ the initial/boundary conditions, and source term $f$:
\begin{align}
    \mathcal{N}(u; \gamma) &= f \quad \text{in } \Omega, \label{eq:PDE}\\
    \mathcal{B}(u) &= g \quad \text{on } \partial\Omega. \label{eq:BC}
\end{align}

Note that different PDEs can be represented in this form, amounting to changing the parameters $\gamma$. The goal here is to develop a generic algorithm that is able to solve the above problem, yielding an approximate solution $u$ given the PDE and different sets of inputs  $(\gamma, f, g)$.

For training, we assume access to a dataset of $M$ problem instances, represented by the PDE parameters $(\gamma_i, f_i, g_i)_{i=1}^M$ and to associated target solution $(u_i)_{i=1}^M$ given on a $m$ point grid $(x_j)_{j=1}^m$. The solutions  $(u_i)_{i=1}^M$ will be used to train the neural solver. At inference, for a new PDE instance, only the PDE parameters are provided and we do not have access to solution points $(u_i)_{i=1}^M$. %In the following sections, we first present our iterative algorithm in \cref{subsec:ngdinf}, and then show how this algorithm is learned, in \cref{subsec:trainngd}. 

\subsection{Methodology}
\label{ssec:methodology}

Physics-informed neural networks consider an ansatz $u_\mpar$ parametrized by some finite-dimensional $\mpar$. The parameters $\mpar$ are iteratively updated by minimizing a criterion $\mathcal{L}_{\textnormal{PDE}}$ (e.g. the PDE residual), which assesses how well the ansatz $u_\mpar$ meets the conditions specified in equations \ref{eq:PDE} and \ref{eq:BC}. As introduced for example for PINNs \citep{PINNs_Raissi19} or Deep Galerkin method \citep{Sirignano2018}, we consider $\mathcal{L}_{\textnormal{PDE}}$ to be given by the strong formulation of the residual $\mathcal{L}_{\textnormal{Res}}$, plus a boundary discrepancy term\footnote{Note that other formulations of the loss may also be considered in a straightforward manner.} $\mathcal{L}_{\textnormal{BC}}$: $\mathcal{L}_{\textnormal{PDE}} = \mathcal{L}_{_{\textnormal{Res}}} + \lambda\mathcal{L}_{\textnormal{BC}}, \quad, \lambda > 0$. 

\begin{equation}
\mathcal{L}_{_{\textnormal{Res}}} = \sum_{x_j \in \Omega} |\mathcal{N}(u_{\mpar}; \gamma)(x_j) - f(x_j)|^2 , \quad \mathcal{L}_{\textnormal{BC}} = \sum_{x_j \in \partial \Omega} |\mathcal{B}(u_{\mpar})(x_j) -g(x_j)|^2 
\label{eq:pinnsloss}
\end{equation}

As illustrated in \cref{sec:motivation}, performing gradient descent, or alternatives such as Adam and L-BFGS on such a highly ill-conditioned loss $\mathcal{L}_{\textnormal{PDE}}$ leads to severe training difficulties \citep{krishnapriyan_characterizing}.
The key idea in our work is to improve a baseline gradient descent algorithms with the neural solver. More precisely starting from a baseline gradient algorithm, SGD in our instantiation, instead of considering the classical update, we first \textit{transform} the gradient using a neural network $\mathcal{F}_{\varrho}$ with parameters $\varrho$, depending on the values of the PDE parameters as well as on other inputs such as the residual gradient provided by SGD: $\nabla_{\mpar}\mathcal{L}_{\textnormal{PDE}}$. The objective is to transform, through the neural solver $\mathcal{F}_{\varrho}$,  the ill-conditioned problem into a new, simpler problem that requires fewer steps in order to achieve a given error.

Once the neural solver $\mathcal{F}_{\varrho}$ has been learned, inference can be performed on any new PDE as follows (see inference \cref{alg:infngd}). Starting from an initial ansatz parameter $\mpar_0$, it is iteratively updated by this solver. At iteration $l$, the steepest direction of the loss $\mathcal{L}_{\textnormal{PDE}}$ is first computed with autograd. Then, the gradient is transformed, in a PDE parameter dependant way, with $\mathcal{F}_{\varrho}$:
\begin{equation}
    \mpar_{l+1} = \mpar_l - \eta \mathcal{F}_{\varrho}(\nabla_{\mpar}\mathcal{L}_{\textnormal{PDE}}(\mpar_l), \; \gamma, f, g)
\label{eq:NGinfstep}
\end{equation}

The objective is to iteratively refine the ansatz to closely approximate the true solution after a series of $L$ iterations, \textit{ideally small} for efficiency.

This approach can be seen as learning the iterates in a PDE solver to achieve a low loss, similar to the residual minimization methods in PDEs \citep{elman2014finite}. By design, this solver is intended to be applicable to different PDEs, as well as various sources, boundary conditions, and initial conditions. This flexibility allows for a broad range of applications, making it a versatile tool in solving complex PDEs with varying characteristics.

Designed as a \textit{parametric PDE solver}, $\mathcal{F}_{\varrho}$  \footnote{In the following we use "solver" with $\mathcal{F}_{\varrho}$ as a short-hand to refer to our proposed method.} is trained with input target data from different sets of PDE parameters, as outlined in \cref{sec:problem_statement}. Once trained, it will be used without retraining on new instances of the PDE, i.e. with new values of the PDE parameters. The underlying hypothesis is that even though the solutions may be different for different inputs and parameters, the solution methodology remains relatively consistent. This consistency is expected to enhance the algorithm's ability to generalize across novel scenarios effectively.

\subsection{Training of a physics-informed solver\label{sec:training}}

\textbf{Choice of Ansatz $u_\mpar$.} A very common choice \citep{SpectralMethodsShen} is to consider a family of basis functions $\Psi(x)=\{\psi_i(x)\}_{i=1}^N$ and consider the ansatz to be given by its linear span ${u_{\mpar}(x) = \sum_{i=0}^N \theta_i\psi_i(x)}$. In the following, we consider this linear reconstruction, although our formulation is generic in the sense that it can also accommodate nonlinear variants. \footnote{Although we have found this may further complicate training.}
% Where $\Omega \subset \mathbb{R}^d$ is the spatial domain, and $T$ denotes the temporal horizon (when involved in the PDE). $\gamma$ refers to the parameter of the PDE, $g, f$ are functions representing the boundary condition and the forcing terms respectively, $\mathcal{P}, \mathcal{B}$ are differential operator modeling the PDE. 

As indicated in \cref{eq:dataloss}, the solver $\mathcal{F}_{\varrho}$ will be trained from samples of the PDE parameter distribution $(\gamma, f, g)$ and from the associated samples of the solution $u$. We first describe below the inference step aiming at iteratively updating the parameters $\mpar$ of the solution function $u_\mpar$ while $\mathcal{F}_{\varrho}$ is held fixed. We then describe how the ${\varrho}$ parameters of the solver are trained. Please refer to \cref{fig:idea} that illustrates the interaction between the two steps.

% \begin{minipage}{0.32\textwidth}
% \RestyleAlgo{ruled}
% \begin{algorithm}[H]
%     \caption{Inference using the neural PDE solver.}
%     \label{alg:infngd}
%     \KwData{$\mpar_0 \in \mathbb{R}^n$, PDE ($\gamma, f, g)$}
%     %\KwResult{$\mpar_L \in \mathbb{R}^n$}
%         \For{l = 0...L-1}{
%             $\mpar_{l+1} = \mpar_l - \eta \mathcal{F}_{\varrho}( \nabla\mathcal{L}_{\textnormal{PDE}}(\mpar_l), \gamma, f, g)$
%         }
%     \Return $\mpar_{L}$
% \end{algorithm}
% \end{minipage}
% \hfill
% \begin{minipage}{0.65\textwidth}
% \RestyleAlgo{ruled}
% \begin{algorithm}[H]
%     \caption{Training algorithm for learning to optimize physics-informed losses.}
%     \label{alg:trngd}
%     \KwData{$\mpar_0 \in \mathbb{R}^n$, PDE ($\gamma, f, g)$, sample values $u(x)$}
%     % \KwResult{$\mathcal{F}_{\varrho}$}
%     \For{$e = 1... $ epochs}  {
%  \For{(\textnormal{PDE}, x, u) in dataset}{
          
%              Initialize $\mpar_0$\\
%              Estimate $\mpar_L$ from $\mpar_0, (\gamma, f, g)$ using Algorithm \ref{alg:infngd} \\
%              Reconstruct $u_{\mpar_L}(x)$ \\
%              Update $\varrho$ with gradient descent from $\mathcal{L}_{\text{DATA}}$.

%          }
%      }
%      \Return
% $\mathcal{F}_{\varrho}$
% \end{algorithm}
% \end{minipage}
\begin{wrapfigure}[26]{R}{0.5\textwidth}
\RestyleAlgo{ruled}
\begin{algorithm}[H]
    \caption{Inference using the neural PDE solver.}
    \label{alg:infngd}
    \KwData{$\mpar_0 \in \mathbb{R}^n$, PDE ($\gamma, f, g)$}
    \KwResult{$\mpar_L \in \mathbb{R}^n$}
        \For{l = 0...L-1}{
            $\mpar_{l+1} = \mpar_l - \eta \mathcal{F}_{\varrho}( \nabla\mathcal{L}_{\textnormal{PDE}}(\mpar_l), \gamma, f, g)$
        }
        \Return $\mpar_{L}$
    \end{algorithm}
% \end{wrapfigure}
% \begin{wrapfigure}[20]{R}{0.5\textwidth}
\RestyleAlgo{ruled}
\begin{algorithm}[H]
    \caption{Training algorithm for learning to optimize physics-informed losses.}
    \label{alg:trngd}
    \KwData{$\mpar_0 \in \mathbb{R}^n$, PDE ($\gamma, f, g)$, sample values $u(x)$}
    \KwResult{$\mathcal{F}_{\varrho}$}
    \For{$e = 1... $ epochs}  {
 \For{(\textnormal{PDE}, x, u) in dataset}{
          
             Initialize $\mpar_0$\\
             Estimate $\mpar_L$ from $\mpar_0, (\gamma, f, g)$ using \cref{alg:infngd} \\
             Reconstruct $u_{\mpar_L}(x)$ \\
             Update $\varrho$ with gradient descent from the data loss in \cref{eq:dataloss}

         }
     }
     \Return
$\mathcal{F}_{\varrho}$
\end{algorithm}
% \vspace{-1.5cm}
\end{wrapfigure}

\textbf{Inference}
The inference step is performed at fixed values of the solver parameters $\mathcal{F}_{\varrho}$. It  consists, for a given instance of the PDE with parameters $(\gamma_i, f_i, g_i)$, in finding the best $\mpar$ with a few steps of the solver using \cref{eq:NGinfstep}. It is illustrated in \cref{fig:idea} - grey box and formalized in \cref{alg:infngd}: starting from  initial parameters $\mpar_0$, we compute the Physical loss $\mathcal{L}_{\textnormal{PDE}}$ using the ansatz $u_{\mpar_0}$. The PDE derivatives in $\mathcal{L}_{\textnormal{PDE}}$, can be computed by hand or automatic differentiation depending on the application \footnote{In our experiments, we computed the derivative by hand when possible since it fastens computations.}. 
Then, \cref{eq:NGinfstep} is used to update the parameters $\mpar$ for a given number of steps $L$.
The final solution is reconstructed using the linear combination, ${u_{\mpar}(x) = \sum_{i=0}^N \theta_i\psi_i(x)}$, introduced in \cref{ssec:methodology} with the computed coefficients $\mpar_L$. Note that inference does not make use of the sampled
 target solutions $(u_i)_{i=1}^M$ computed on the grid points$(x_j)_{j=1}^m$.These targets are used exclusively for training the neural solver.
%In the end, we can reconstruct the solution $u_{\theta}$ and query it at every point $x$ with: \begin{equation} 
%     u_{\theta_L}(x) = \sum_{i=0}^N \theta^L_i\psi_i(x)
% \label{eq:utheta}
% \end{equation}
%During inference and for a new PDE instance to solve, we apply the iterative algorithm with new parameters, IC and BC to compute the loss $\mathcal{L}_{\textnormal{\textnormal{PDE}}}$ and the optimized descent steps. %
%no additional data for optimization. Indeed, only the PDE parameters and BC expression are required to compute the loss $\mathcal{L}_{\textnormal{\textnormal{PDE}}}$ and the optimized descent steps. 

\textbf{Training the neural solver}
Training amounts to learning the parameters of the solver $\mathcal{F}_{\varrho}$ and is performed with a training set of PDE parameters and simulation data considered as ground truth $(\gamma_i, f_i, g_i, u_i)_{i=1}^M$, corresponding to PDE instances (\textit{i.e.} with \textbf{different} parameters $\gamma$ and/or forcing terms $f$ and/or initial/boundary conditions $g$). See \cref{fig:idea} - white box and \cref{alg:trngd}. The objective is to learn a solver $\mathcal{F}_{\varrho}$ able, at inference, to converge to a target solution in a small ($2$ to $5$ in our experiments) number of steps. For that, an optimizer (Adam in our experiments) is used to update the $\mathcal{F}_{\varrho}$ parameters. The training algorithm makes use of the data associated with the different PDE instances by sampling PDEs in  batches and running \cref{alg:infngd} on several PDE instances. For each PDE instance, one starts from an initial parameter value $\mpar_0$ and then performs two optimization steps (see \cref{alg:trngd}): (i) one consists in solving in the ansatz parameters $\mpar$ using the neural solver using \cref{alg:infngd}, leading to $\mpar_L$; (ii) the second one is the optimization of the solver parameters $\varrho$.
%In order to assess that the iterative process converge in a fixed small number of steps, we run \cref{alg:infngd} for a given instance of a PDE with parameters $\gamma_i, f_i, g_i$ and obtain an approximate $u_{\theta}$ of the solution. 
We train the outputs directly to match the associated ground truth $(u_i(x_j))_{j=1}^m$ using the data loss:
\begin{equation}
    \mathcal{L}_{\text{DATA}} = \mathbb{E} _{\gamma, f, g}\left[||u_{\mpar_L}-u_{\gamma, f, g}|| \right].
\label{eq:dataloss}
\end{equation}
% The expectation is
The expectation is computed on the distribution of the PDE parameters $(\gamma, f, g)$. The solution $u$ is entirely determined by these parameters as indicated by the notation $u_{\gamma, f, g}$. $||u_{\mpar_L}-u||$ denotes a distance between the target $(u(x_j))_{j=1}^m$ and the forecast $(u_{\mpar_L} (x_j))_{j=1}^m$ with $m$ the trajectory size \footnote{To simplify the notation, we used a fixed grid size $m$. However, this framework can be used with different grid sizes, as well as irregular grids. See Ablation in \cref{app:ablation}, \cref{tab:abirgrids}}. In practice, one samples a set of PDE instances $(\gamma_i, f_i, g_i)$ and for each instance a corresponding sample $u_i$.

%This data-based optimization makes the training of the network more efficient, since it does not suffers from the physical loss optimization issues. 

%a solver that builds accurate solution of the PDE. In \cref{alg:trainngd}, we train the output of our algorithm to match the samples $u_i$ in a fixed number of steps $L$ using the data loss:
% \begin{equation}
%     \mathcal{L}_{\text{DATA}} = \frac{1}{M} \sum_{i=1}^M ||u_{\theta_L}-u_i||
% \label{eq:dataloss}
% \end{equation}
% This data-based optimization makes the training of the network more efficient, since it does not suffers from the physical loss optimization issues. 

%  \subsection{Formal Analysis - Relation to Preconditioning}

% \color{red}
% XXXXXX j'ai mis ça, je ne sais pas ce que signifie "span the entire parameter space, j'imagine que pour un modèle lineaire ca veut dire que l'ens d'apprentissage permet d'engendrer tt l'espace lineaire (si 10 coefs dans le modèle, engendre un esp. de dim 10). Emmanuel si tu peux checker. Il faut egalement modifier les notations des gros thetaXXXXXX\\

\paragraph{Theoretical analysis and relation to preconditioning} 
Analyzing the behavior of the inference algorithm is challenging due to the non linear nature of the solver. We however could get some intuition using simplifying assumptions. We build on the ideas introduced in \cref{sec:motivation} for the simple case of the Poisson equation, for which an explicit analytical solution could be derived. We provide in \cref{app:proof}, a proof for a more general case and give below in \cref{thm:ngd_main_proof} our main result. This shows that the number of steps induced by $\mathcal{F}$ for the proposed algorithm is significantly less than the number of steps required by the baseline PINNs algorithm. This results is obtained under two main assumptions: (i) \( \mathcal{F} \) behaves like its linearization and (ii) the descent operator $\mathcal{F}$ used in our algorithm, allows us to reach the optimum of $\mathcal{L}_{\text{DATA}}$. \\

\begin{theorem}{\textbf{(Convergence rate in the linear case).}}
\label{thm:ngd_main_proof}
Given a linear ansatz \( u_\Theta(x) = \sum_{i=1}^N \theta_i \phi_i(x) \), assume the conditioner \( \mathcal{F} \) behaves like its linearization \(P = \text{Jacobian}( \mathcal{F}) \), meaning that \( \mathcal{F} \) can be replaced by \(P \) at any point. %\right|_{v=0} \(P = \tfrac{\partial \mathcal{F}}{\partial v} \).
% \begin{equation}
% A_{i,j} = \int_{\Omega} \left( \mathcal{D} \phi_i(x) \right) \left( \mathcal{D} \phi_j(x) \right) dx + \lambda \int_{\partial \Omega} \phi_i(x) \phi_j(x) dx.
% \end{equation}
 Let A be the matrix derived from the PDE loss as \cref{eq:A} for the Poisson equation or \cref{eq:A_def} in the more general case. Denote by \( \kappa(A) \) the condition number of the matrix \( A \). The number of steps \( N'(\varepsilon) \) required to achieve an error \( \| \Theta_l - \Theta^\ast \|_2 \leq \varepsilon \) satisfies: 
\begin{equation}
    N'(\varepsilon) = O\left( \kappa(PA) \ln \left( \tfrac{1}{\varepsilon} \right) \right),
\end{equation}
Moreover, if \( \mathcal{F} \) minimizes \( \mathcal{L}_{\text{DATA}} \) this necessarily implies \( \kappa(PA) = 1 \leq \kappa(A) \). Consequently, the number of steps is effectively reduced, i.e., \( N'(\varepsilon) \ll N(\varepsilon) \) with \(N(\varepsilon) \) the number of steps of the vanilla PINNs.
\end{theorem}
\begin{proof}
We sketch the main insights here and refer to \cref{app:proof} for the proof and a detailed analysis.
    \begin{itemize}
    \item Using a linearization of the neural solver, it can be shown that the solver performs as a pre-conditioner on the linear system.
    \item Assuming that solution \textbf{$u_L$} provided by the solver reaches the optimum \textbf{$u^*$}, and that the training set is such that the learned parameter $\mpar$ vectors span the whole parameter space of the model, then the convergence of the solver is guaranteed at an optimal rate.
    \item In practice, and as shown in the experiments (\cref{sec:expe}), the convergence rate is significantly improved w.r.t. the reference baseline gradient algorithm.%  In the experiments section, we provide experimental comparisons with several algorithms.  %the L-BFGS algorithm, which can be seen as a non-linear pre-conditioning technique \cite{RathoreICML2024}.
\end{itemize}
\end{proof}
% XXXXXXX introduire theorem + preciser notations A, Kappa

% We extend the result presented in \cref{sec:motivation} to the general case and show how the above method effectively improves the convergence of gradient descent algorithm when solving PDE with Physics-Informed losses.
% \begin{theorem}
%     Using the above notation, let consider solving a given PDE with boundary and/or initial conditions with solution \textbf{$u^*$}. We focus on solving this PDE using a Physics-Informed method and we denote $\Theta=\{\theta^i\}_{i=1}^N$ the parameters used to parameterize the approximation of the solution $u_{\Theta}$. Suppose we have access to samples of this PDE and train a neural network with \cref{alg:trngd}. \\
%     Suppose (1) $\exists \Theta^* \mid u_{\Theta^*}=u$ 
%     and (2)  the training samples span the parameter space completely \textit{i.e.}, the matrix $(\Theta_0-\Theta^*)^T(\Theta_0-\Theta^*)$ is full rank
%     Then, the above method with NN achieves enhanced convergence rate and reaches optimal solution. 
% \end{theorem}
% We provide the proof in \cref{app:proof} and show experimental results illustrating this statement in \cref{sec:expe}. 
% \color{black}
\vspace{-1.cm}
\section{Experiments}
\label{sec:expe}
% We present here the experimental set up used to obtain the results presented in \cref{sec:results}. 
%The experiments conducted in this section aim at showing the relevance of considering gradient-based iterative methods to solve PDE and the difficulty induced in training when using physical loss. Experiment are conducted on NVIDIA TITAN V ($12$ Go) for $1$d datasets and NVIDIA RTX A6000 GPU with 49Go for $1d$ + time or $2$d datasets. 
We present the datasets used in the experiments in \cref{subsec:datasets}, a comparison with selected baselines in \cref{subsec:baselines}, and a test-time comparison with different optimizers demonstrating the remarkable effectiveness of the proposed method in \cref{subsec:testimeopt}. Finally, we make a comparison of the training and inference time in \cref{subsec:comptime}. Experimental details and additional experiments can be found in the appendices: ablations are in \cref{app:ablation} and additional results and visualization are in \cref{app:losslandscape}, \cref{ssec:app_vismethod}, and \cref{app:visu}. 

% Mettre les ablations en annexe? 
% \vspace{-0.75cm}
\subsection{Datasets}
\label{subsec:datasets}

% \vspace{-0.5cm}
\begin{wraptable}{r}{0pt}%{0.48\textwidth}
%\begin{table}[htbp]
    \centering
    \begin{tabular}{c|c|c}
         Dataset & Parameters & Distribution \\
         \toprule
         \multirow{3}{*}{Helmholtz} & $\omega$ & $\mathcal{U}[0.5, 50]$\\
         & $u_0$ & $\mathcal{N}(0, 1)$\\
         & $v_0$ & $\mathcal{N}(0, 1)$\\
         \midrule
         \multirow{3}{*}{Poisson} & $A_i$ & $\mathcal{U}[-100, 100]$\\
         & $u_0$ & $\mathcal{N}(0, 1)$\\ 
         & $v_0$ & $\mathcal{N}(0, 1)$\\
        \midrule
         \multirow{2}{*}{NLRD} & $\nu$ & $\mathcal{U}[1, 5]$ \\
         & $\rho$ & $\mathcal{U}[-5, 5]$\\
         \midrule
         \multirow{2}{*}{Darcy} & \multirow{2}{*}{$a(x)$} & $\psi_{\#}\mathcal{N}(0, (-\Delta + 9I)^{-2})$ \\
         & & with $\psi = 12*\mathds{1}_{\mathbb{R}_+} + 3*\mathds{1}_{\mathbb{R}_+}$\\
         \midrule
         \multirow{6}{*}{Heat} & $\nu$ & $\mathcal{U}[2\times 10^{-3}, 2\times 10^{-2}]$ \\
         & $J_{max}$ & \{1, 2, 3, 4, 5\}\\
         & $A$ & $\mathcal{U}[0.5, -0.5]$\\
         & $K_x$, $K_y$ & $\{1, 2, 3\}$\\
         & $\phi$ & $\mathcal{U}[0, 2\pi]$ \\
         \bottomrule
    \end{tabular}
    \caption{Parameters changed between each trajectory in the considered datasets. }
    \label{tab:parameters_datasets}
\vspace{-0.5cm}
%\end{table}
\end{wraptable}
We consider several representative parametric equations for our evaluation. More details about the data generation are presented in \cref{app:appdataset}. Our objective is to learn a neural solver able to solve quickly and accurately a new instance of a PDE, given its parametric form, and the values of the parameters $\gamma$, forcing terms $f$ and initial/boundary conditions $g$, \textit{i.e.} $(\gamma, f, g) \mapsto u$. Solving is performed with a few iterations of the neural solver (\cref{alg:infngd}).
For that, one trains the neural solver on a sample of the PDE parameter instances, see \cref{tab:parameters_datasets} for the parameter distributions used for each parametric PDE. 
\textbf{Evaluation is performed on unseen sets of parameters within the same PDE family.}
\textbf{Helmholtz}: We generate a dataset following the $1d$ static Helmholtz equation $u''(x) + \omega^2u(x) = 0$ with boundary conditions $u(0) = u_0 \text{ and } u'(0) = v_0$. We generate $1,024$ trajectories with varying $\omega$, $u_0$, and $v_0$ with a spatial resolution of $256$. 
\textbf{Poisson}: We generate a dataset following the $1d$ static Poisson equation with forcing term: $- u''(x) = f(x)$ with $u(0) = u_0$ and $u'(0) = v_0$. The forcing term $f$ is a periodic function, $f(x) = \frac{\pi}{K}\sum_{i=1}^K a_ii^{2r}\sin(\pi x)$, with $K=16$ and $r=-0.5$. We generate $1,000$ trajectories with varying $u_0, v_0$, and $f$ (through changing $a_i$) with a spatial resolution of $64$. 
\textbf{Reaction-Diffusion}: In \cite{krishnapriyan_characterizing, toloubidokhti2024dats}, the authors propose a non-linear reaction-diffusion (\textit{NLRD}). This PDE has been shown to be a failure case for PINNs \citep{krishnapriyan_characterizing}. We generate $1,000$ trajectories by varying the parameters of the PDE: $\nu$ and $\rho$ (see \cref{tab:parameters_datasets}). Spatial resolution is $256$ and temporal resolution is $100$. The PDE is solved on $[0, 1]^2$.  
\textbf{Darcy Flow}: The $2d$ Darcy Flow dataset is taken from \citep{li2020fno} and is commonly used in the operator learning literature \citep{li2023pino, goswami2022PIDON}. For this dataset, the forcing term $f$ is kept constant $f=1$, and $a(x)$ is a piece-wise constant diffusion coefficient taken from \citep{li2020fno}. We kept $1,000$ trajectories (on the $5,000$ available) with a spatial resolution is $64\times 64$.
\textbf{Heat}: The $2d+t$ Heat equation is simulated as proposed in \citep{maepde24}. For this dataset, the parameter $\nu$ is sampled from $\mathcal{U}[2\times 10^{-3}, 2\times 10^{-2}]$ and initial conditions are a combination of sine functions with a varying number of terms, amplitude, and phase. 
\color{black}
A summary of the datasets and the varying parameters for each PDE are presented in \cref{tab:parameters_datasets} and more details on the dataset are provided in \cref{app:appdataset}.
Experiments have been conducted on NVIDIA TITAN V ($12$ Go) for $1d$ datasets to NVIDIA RTX A6000 GPU with 49Go for $1d$ + time or $2d$ datasets. For all datasets, $800$ PDEs are considered during training and $200$ for testing. All metrics reported are evaluated on test samples (\textit{i.e.}\textbf{ PDEs not seen during training}. Coefficients as well as initial and/or boundary conditions can vary from training).

\subsection{Comparison with baselines}
\label{subsec:baselines}
We performed comparisons with several baselines including fully data-driven supervised approaches trained from a data-loss only, unsupervised methods relying only on a PDE loss, and hybrid techniques trained from PDE + DATA losses. Network size and training details are described in \cref{app:impdetails}. In this experiment, we considered training the models using the training sets (physical losses or MSE when possible) unless stated otherwise.

% XXXX peut etre un mot sur la fairness des comparaisons en terme de nbre de data utilisées pour apprendre ?\\

%here the baselines used in the experiment presented in the following sections. First, we compare several type of models both supervised and/or unsupervised \cref{tab:test-loss}. Then, we also evalute the effectiveness of our learned optimization \textit{w.r.t.} other optimizers \cref{fig:test-time-opt}. Networks sizes and other training details are described in \cref{app:impdetails}. % iterative procedure to some test-time optimization . 

% \subsubsection{Comparison with baselines}
% In table \ref{tab:test-loss}, we chow a comparison of training procedure of several baselines. 

\textbf{Fully supervised}
We train a standard MLP to learn the mapping $(\gamma, f, g)\mapsto \mpar$, using as loss function $\mathcal{L}_{\text{DATA}} = \mathbb{E} _{\gamma, f, g,u}\left[||u_{\mpar_L}-u_{\gamma, f, g}|| \right]$
 with $u_{\mpar}(x) = \sum_{i=0}^N \theta_i\psi_i(x)$, the $\psi_i(.)$ being fixed B-Spline basis functions (see \cref{app:impdetails}). We denote this baseline as \textit{MLP+basis}. % In addition, we compare with a FNO predicting directly the solution from PDE parameters and/or IC/BC \textit{i.e.} in the setting used in \citep{li2020fno}. Moreover, we compare with a supervised DeepONet \citep{Lu_2021_DON}. This models allows to learn a basis, and compute the terms in this basis from the PDE data. 
\textbf{Unsupervised}
We compare our approach with unsupervised physics-informed models \citep{PINNs_Raissi19}. While the initial version of PINNs solves only one PDE instance at a time and requires retraining for each new instance, we developed here a parametric version of PINNs (\textit{PPINNs}) where the PDE parameters are fed to the network (similarly to \citep{23_PPINNs}). Finally, we used \citep{cho2024parameterizedphysicsinformedneuralnetworks}'s (\textit{P2INNs}) method as a physics-informed baseline specifically designed for parametric PDEs. In addition to PINNs-methods, we also compare our solver to the Physic-informed DeepONet (\textit{PO-DeepONet} for Physics-Only DeepONet) \citep{wang2021PIDON}, which is designed to learn an operator for \textit{function-to-function} mappings from physical losses and handles parametric PDEs. The mapping learned for the two unsupervised baselines is $(x, \gamma, f, g) \mapsto u_{\gamma, f, g}(x)$. In order to provide a fair comparison with our optimization method, we fine-tuned the unsupervised baselines for each specific PDE instance for a few steps ($10$ or $20$).  % \textit{i.e.} being able to manage functions. 
% In addition, we compare to Meta-auto decoder PINNs \citep{huang2022metaautodecoder} and hyper-PINNs \citep{belbute-peres2021hyperpinn} that is designed to handle several types of PDEs. 
\textbf{Comparison to preconditioning} We compare our approach with vanilla PINNs \citep{PINNs_Raissi19}, \textit{i.e.} by fitting one PINN per PDE in the test set and averaging the final errors. We optimize the PDE losses using L-BFGS \citep{Liu98LBFGS} and refer to this baseline as \textit{PINNs+L-BFGS}. As discussed in \citep{RathoreICML2024}, L-BFGS can be considered as a nonlinear preconditioning method for Physics-Informed methods and fastens convergence.  Finally, we use the training strategy proposed by \citep{RathoreICML2024} \textit{i.e.} trained PINNs using successive optimizer (Adam + L-BFGS). This baseline is denoted as \textit{PINNs-multi-opt}. For these baselines, one model is trained and evaluated for each PDE in the \textbf{test} set. We report the reader to \cref{app:impdetails} for more details on the training procedure. 
\textbf{Hybrid}
Finally, we compare our proposed method with neural operators, \textit{i.e.}, models trained to learn mappings $(x, \gamma, f, g) \mapsto u_{\gamma, f, g}(x)$ using a combination of physical and data loss: $\mathcal{L}_{\textnormal{DATA}} + \mathcal{L}_{\textnormal{PDE}}$. We use as baselines Physics-Informed Neural Operator (\textit{PINO}) \citep{li2023pino}  and Physics-Informed DeepONet (\textit{PI-DeepONet}) \citep{goswami2022PIDON}. As already indicated, for a fair comparison, the \textit{Unsupervised} and \textit{Hybrid} baselines are fine-tuned on each specific PDE instance for a few steps ($10$ on all datasets except for Heat for which $20$ steps are made). 
%Note that this hybrid setting is  the closest to ours since it assumes access to both the PDE equation and to data during training, and adapts to new PDE instances by optimizing a PDE loss.
%(through its parameters and initial/boundary conditions) during evaluation. 
%pre-training phase of hybrid operator learning: Physics-Informed Neural Operator (\textit{PI-DeepONet}) \citep{li2023pino}  and Physics-Informed DeepONet (\textit{PINO}) \citep{goswami2022PIDON}. We use a hybrid loss (i.e. a combination of the physical loss and a data loss) to train these operators. % These baselines allows us to quantify the complexity induced by the physical loss w.r.t purely data-losses in the training process of networks. 
\textbf{Ours}
 We represent the solution $u_{\mpar}$ with a linear combination of B-Spline functions for $\Psi$ \citep{PiegTill96_NURBS}.  This was motivated by the nature of B-Splines which allows to capture local phenomena. However, other bases could be used such as Fourier, Wavelet or Chebychev Polynomials. The neural solver $\mathcal{F}_{\varrho}$ is composed of Fourier Layers (FNO) \citep{li2020fno} that allow us to capture the range of frequencies present in the phenomenon. We refer the reader to \cref{app:impdetails} for more details about the construction of the B-Spline basis and the training hyper-parameters.

\begin{table}[htbp]
  \centering
  \begin{tabular}{lcccccc}
    % & & \multicolumn{4}{c}{Dataset} \\
    & & \multicolumn{2}{c}{1d} & \multicolumn{1}{c}{1d+time} & 2d  & 2d+time \\ 
    \cmidrule(lr){3-4} \cmidrule(lr){5-5} \cmidrule(lr){6-6} \cmidrule(lr){7-7} 
     & Baseline & Helmholtz & Poisson & NLRD & Darcy-Flow & Heat\\
    \toprule
    %& \textit{Maths Solver} & \\
     Supervised & \textit{MLP + basis} &  \underline{4.66e-2} & 1.50e-1 & \textbf{2.85e-4} & \underline{3.56e-2} & 6.00e-1\\
    % & Deeponet & 4.05e-1 & 1.14e-1 & 3.83e-1 & 6.40e-6 & 6.66e-2\\
    % & FNO & 4.7e-1 & 1.32e-4 & 6.07e-4 & 8.48e-6 & \textbf{6.87e-3}\\
    \midrule
    \multirow{5}{*}{Unsupervised} & \textit{PINNs+L-BFGS} & 9.86e-1 & 8.83e-1 & 6.13e-1 & 9.99e-1 & 9.56e-1\\
    & \textit{PINNS-multi-opt} & 8.47e-1 & 1.18e-1 & 7.57e-1 & 8.38e-1 & 6.10e-1\\
    & \textit{PPINNs} &  
    9.89e-1 & 4.30e-2 & 3.94e-1 & 8.47e-1 &  1.27e-1\\
    & \textit{P2INNs} & 9.90e-1 & 1.50e-1 & 5.69e-1 & 8,38e-1 & 1.78e-1 \color{black}\\
    & \textit{PO-DeepONet} & 9.83e-1 & 1.43e-1 & 4.10e-1 & 8.33e-1 & 4.43e-1\\
     %& \textit{MAD-PINNs} & \\
     %& Hyper-PINNs & \\
     \midrule
    \multirow{2}{*}{Hybrid} & \textit{PI-DeepONet} & 9.79e-1 & 1.20e-1 & 7.90e-2 & 2.76e-1 & 9.18e-1\\
     & \textit{PINO} & 9.99e-1 & \underline{2.80e-3} & 4.21e-4 & 1.01e-1 & \underline{9.09e-3}\\
    \midrule
%    & \textit{Hybrid DL solver} & \\
    {Neural Solver} & \textit{Ours} & \textbf{2.41e-2} & \textbf{5.56e-5} & \underline{2.91e-4} & \textbf{1.87e-2} & \textbf{2.31e-3}\\
    \bottomrule
  \end{tabular}
  \caption{Results of trained models - metrics in Relative MSE on the test set. Best performances are highlighted in \textbf{bold}, and second best are \underline{underlined}. }
  \label{tab:test-loss}
\end{table}
\textbf{Results:} \Cref{tab:test-loss} presents the comparison with the baselines. %Each method is trained to solve a parametric equation, using a sample of the equation parameters. 
We recall that the evaluation set is composed of several PDE instances sampled from unseen PDE parameters $(\gamma, f, g)$.
The proposed method is ranked first or second on all the evaluations. The most comparable baselines are the unsupervised methods, since at inference they leverage  only the PDE residual loss, as our method does. Therefore our method should be primarily compared to these baselines. Supervised and hybrid methods both incorporate data loss and make different assumptions while solving a different optimization problem.

\Cref{tab:test-loss} clearly illustrates that unsupervised Physics-informed baselines all suffer from ill-conditioning and do not capture the dynamics. Compared to these baselines, the proposed method improves at least by one order of magnitude in all cases.
PINNs baseline performs poorly on these datasets because of the ill-conditioning nature of the PDE, requiring numerous optimization steps to achieve accurate solutions (\cref{section-implementation-details}).  This is observed on PINNs models for parametric PDEs (PPINNs and P2INNs) as well as on PINNs fitted on one equation only (PINNs+L-BFGS and PINNs-multi-opt). We observe that our neural solver has better convergence properties than other Physics-Informed methods. As will be seen later it also converges much faster.

The supervised baseline performs well on all the PDEs except \textit{Poisson} and \textit{Heat}. The data loss used for training this model is the mean square error which is well-behaved and does not suffer from optimization problems as the PDE loss does. We note that our method reaches similar or better performances on every datasets, while relying only on physical information at inference (\cref{alg:infngd}) and solving a more complex optimization problem.
%XXX a quel point peut on comparer supervised avec la méthode proposée ? Il faudrait disqualifier le supervised ??
%during training as well as from learning the high frequencies it contains (apart for the \textit{Advection} problem, where Fourier layers seem to be particularly adapted \citep{takamoto2023pdebench})

The hybrid approaches, do not perform well despite taking benefits from the PDE+DATA loss and from adaptation steps at test time. Again, the proposed method is often one order of magnitude better than the hybrids except on NLRD, where it has comparable performances. This shows that the combination of physics and data losses is also hard to optimize, and suffers from ill-conditioning. 

\color{black}
 
 %The only exception is PINO on \textit{Advection}, and this is because the Fourier basis used in PINO is well adapted to this equation. XXXXX et ca marche pas mal sur NLRD aussi ????

% XXXX on  ne sait pas si on peut les comparer en terme de données ustilisées pour apprendre. Il faudrait préciser. Pour les PI-Deeponet dans le papier de Perdikaris, quels sont leurs scores sur du parametrique ? Il doivent bien reussir à apprendre qquechose non ? 
% As we see in \cref{tab:test-loss}, the only baseline to succeed to learn to solve the PDE is the \textit{MLP+basis} which is data-driven. On the NLRD dataset, this model also performs very well. However, this baseline is optimized from MSE, which has a much easier loss-landscape for optimization. Note that, our model achieves close performances on this dataset whereas model trained from physics-informed losses fails to learn the solution accuratly. 

% If L-BFGS optimization improves the results and fastens optimization, 
\color{black}

\subsection{Optimization for solving new equations}
\label{subsec:testimeopt}
 The main motivation for our learned PDE solver is to accelerate the convergence to a solution, w.r.t. predefined solvers, for a new equation. In order to assess this property, we compare the convergence speed at test time inference with classical solvers, PINNs, and pre-trained PINO as detailed below. Results are presented in \cref{fig:test-time-opt} for the  \textit{Poisson} equation with performance averaged on $20$ new instances of the \textit{Poisson} equation. This experiment is also performed on the other datasets in \cref{app:visu}. %XXXXX qu'est ce qu'on peut dire de plus sur ces comparaisons additionnelles ?

%In \cref{fig:test-time-opt}, we compare our learned optimization procedure to standard optimizers and test-time optimization with the physical loss of the pre-trained models PINO. We aim here at comparing the number of steps needed for convergence using a physical loss. % and MAD-PINNs. Experiments below are conducted on the \textit{Poisson} dataset (but we provide in \cref{app:visu} this experiment on the other datasets). %, which remains a challenging dataset despite being a $1d$-static problem, as it has been shown in \cref{tab:test-loss}.  We average the results on $20$ new instances of the equation. For this, we introduce comparison with other existing optimizers on a B-spline basis, as well as Neural Networks and Neural Operator. 
\textbf{Baseline optimizers}
As for the classical optimizer baselines, we used SGD, Adam \citep{KingBa15_Adam}, and L-BFGS \citep{krishnapriyan_characterizing}. These optimizers are used to learn the coefficient of the B-Spline basis expansion in the model $u_{\mpar}(x) = \sum_{i=0}^N \theta_i\psi_i(x)$. This provides a direct comparison to our iterative neural solver.
% For a fair comparison with our framework, we optimize the coefficient in the B-Spline basis used in our experiments. This setting is very similar to ours, since only optimizers are modified. 
\textbf{PINNs} - We also compare to the standard PINNs \citep{PINNs_Raissi19}, \textit{i.e.} by fitting one Neural Network (NN) per equation. Note that this requires full training from scratch for each new equation instance and this is considerably more computationally demanding than solving directly the parametric setting. This baseline is similar to the Adam optimizer mentioned above, except that the ansatz for this experience is a multilayer perceptron instead of a linear combination of a B-Spline basis. 
\textbf{Hybrid pre-training strategies} -
Finally we compare against the hybrid \textit{PINO} pre-trained on a set of several parametric PDE instances and then fine-tuned on a new PDE instance using only the PDE loss associated to this instance. 
\textbf{Ours} - We train our model as explained in \cref{alg:trngd} and show here the optimization process at \textbf{test time}. In order to perform its optimization, our model leverages the gradient of the physical loss and the PDE parameters (coefficients, initial/boundary conditions). In this experiment, we use $L=5$ steps for a better visualization (whereas, we used $L=2$ in \cref{tab:test-loss}). 
%models presented above, namely \textit{PINO}. We recall, that this model is first pre-trained on a set of several parametric PDE instances. Fine-tuning is computed only based on the physical loss $\mathcal{L}_{\textnormal{PDE}}$. 
%\textbf{Ours} Similarly as the previous section, we use our learned optimizer to solve new PDE instance. 

\begin{wrapfigure}{r}{0.5\textwidth}
%\begin{figure}[htbp]
   % \begin{center}
   % \vspace{-1cm}
        \includegraphics[width=\textwidth]{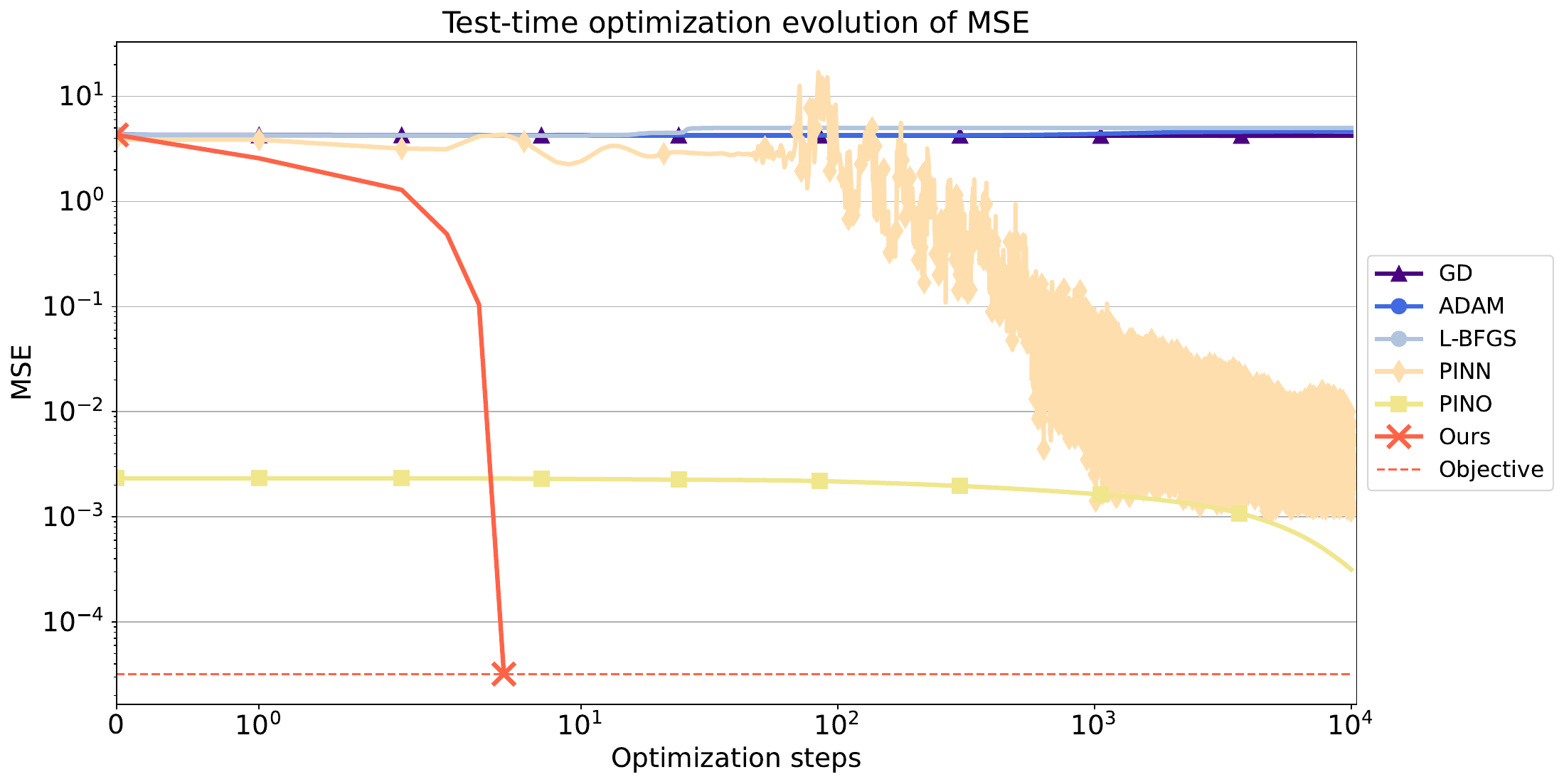}
    %\end{center}
    % \vspace{-50pt}
    \caption{Test-time optimization based on the physical residual loss $\mathcal{L}_{\textnormal{PDE}}$ for new PDE on \textit{Poisson}.}
    \label{fig:test-time-opt}
    % \vspace{-0.5cm}
%\end{figure}
\end{wrapfigure}
%\newline
% \vspace{-0.5cm}
\textbf{Result}: \cref{fig:test-time-opt} compares the number of optimization steps required for the different methods. Our neural solver converges very fast in only a few steps (5 here) to a good approximation of the solution, while all the other methods require thousands of iterations - we stopped here at $10,000$ steps. The classical optimizers (SGD, Adam, L-BFGS) do not converge for a new equation. The baseline PINNs trained here from scratch on each new equation show an erratic convergence behavior. Pretrained PINO behaves better than the other baselines but still did not converged after $10,000$ steps. This clearly demonstrates the potential of our learned solver to deal with physical losses w.r.t. alternative pre-defined solvers. 

\subsection{Computational time}
\label{subsec:comptime}

An important aspect of solving PDEs is the computational time required for each solution. Methods should find a trade-off between achieving high performance and maintaining reasonable computational costs. In \cref{app:ctime}, we provide the training (\cref{tab:traintime}) and inference (\cref{tab:testtime}) times for our method compared to various baselines. Our results show that while our method has comparable training and inference times to other approaches, it demonstrates substantially better precision (\cref{tab:test-loss}).

% \color{black}

%\input{RelatedWork}

% XXXXX j'ai pas regardé la suite XXXXX

\section{Conclusion}
We have presented a PDE solver learned from data that allows fast test-time optimization of physical losses. Our method succeeds to considerably accelerate the optimization process for the complex problem of minimizing physical losses and is several orders of magnitude faster than classical  hand-defined optimization methods such as Adam or L-BFGS.
%by having access to the solution through data. We showed that using gradient and iterative procedure could help improving performance and accelerating convergence of physics-informed method by conducting some experiments on failure cases for Physics-informed methods. 
\textbf{Limitations and Future Work}
% \label{sec:limitation}
While efficient, the proposed method can be further improved. First, training our iterative algorithms requires more memory than standard machine learning models due to the complexity of backpropagation through iterations, which becomes challenging in higher-dimensional bases. Second, we have focused on solution approximations expressed as a linear expansion in predefined bases. More expressive representations, such as neural networks, could be explored; however, our preliminary experiments indicate increased ill-conditioning due to the compositional nature of neural networks. More sophisticated training schemes could enhance the optimization process. Future work will investigate these directions to improve scalability and broaden the applicability of the proposed method.
\textbf{Reproducibility Statement}
Hyper-parameters, baselines configurations and training details are detailed in \cref{app:impdetails,tab:pinnslbfgshp,tab:archbase,alg:trngd,alg:infngd}. The creation of the datasets is explained in \cref{app:appdataset,tab:parameters_datasets,tab:parameters_datasets_app}. Finally, we provide a theoretical analysis of the model, under an ideal scenario in \cref{app:proof}. Code is available at \url{https://github.com/2ailesB/neural-parametric-solver}. 

\subsubsection*{Ethics Statement}
Solving PDE is of crucial interest in many applications of science and engineering. While we do not directly target such real-world applications in this paper, one should acknowledge that solvers can be used in various ranges of scenarios including weather, climate, medical, aerodynamics, industry, and military applications. 
%Secondly, several properties remains to be checked so that such methods behave as mathematical optimizers, such as the consistency wrt the number of steps. Future work will focus on solving these drawbacks of the method to improve its capabilities for broader applications. 
% parler scalabilité + petits dataset + no theory 
% => use modulation in the network based on the steps
% => ajouter ça à du méta learning de l'init ? 

% \subsubsection*{Author Contributions}
% If you'd like to, you may include  a section for author contributions as is done
% in many journals. This is optional and at the discretion of the authors.

\subsubsection*{Acknowledgments}
We acknowledge the financial support provided by DL4CLIM (ANR-19-CHIA-0018-01), DEEPNUM (ANR-21-CE23-0017-02), PHLUSIM (ANR-23-CE23-0025-02), and PEPR Sharp (ANR-23-PEIA-0008”, “ANR”, “FRANCE 2030”). This project was provided with computer and storage resources by GENCI at IDRIS thanks to the grant 2025-AD011014938R1 on the supercomputer Jean Zay's V100/A100/H100 partitions.

% Use unnumbered third level headings for the acknowledgments. All
% acknowledgments, including those to funding agencies, go at the end of the paper.
\clearpage

\bibliography{iclr2025_conference}
\bibliographystyle{iclr2025_conference}
\clearpage

\appendix
\section{Detailed related Work}
\label{app:rw}

\paragraph{PDE solvers:} Many tools for numerically solving PDEs have been developed for years. The standard methods for PDE include Finite Differences (FD), Finite Volume (FV), Finite Element Method (FEM), spectral and multigrid methods, and many others \citep{s2012numericalAnPDE, liu2009meshfree}. While these methods are widely used, they often suffer from a high computational cost for complex problems or high-precision simulations. To address these challenges, integrating deep learning (DL) into solvers has emerged as a promising approach. Current solutions include incorporating correction terms into mathematical solvers to reduce numerical errors \citep{um2021solverintheloop}. Some work such as \cite{hsieh2019learningNPDESconv, li23e-preconcgpde, rudikov2024fgcno, kopanivcakova2023enhancing} build a method to directly enhance the convergence of numerical solvers through preconditioner learning. As an example, \cite{rudikov2024fgcno, li23e-preconcgpde} use a neural operator to approximate conditioner for the flexible conjugate gradient method or \cite{hsieh2019learningNPDESconv} for the Jacobi method. Another example of preconditioner learning lies in \cite{li23e-preconcgpde}, where the author uses GNN to assess symmetry and positive definiteness. 

\paragraph{Unsupervised training:} Physics-Informed Neural Networks (PINNs) \citep{PINNs_Raissi19} have been a pioneering work in the development of DL method for physics. In these models, the solution is a neural network that is optimized using the residual loss of the PDE being solved. 
However, this method suffers from several drawbacks. First, as formulated in \cite{PINNs_Raissi19}, PINNs can solve one instance of an equation at a time. Any small change in the parameters of the PDE involves a full retraining of the network. Efforts such as \cite{22PINNSparammagneotstatic, 23_PPINNs, cho2024parameterizedphysicsinformedneuralnetworks} have attempted to address this limitation by introducing parametric versions of PINNs capable of handling parametric equations, while \cite{huang2022metaautodecoder, qin2022metapde} explores meta-learning approaches. Other approaches to improve PINNs generalization include using neural operators \citep{wang2021PIDON}, or hyper-network \citep{belbute-peres2021hyperpinn}. 
Moreover, PINNs have shown convergence difficulties: \cite{krishnapriyan_characterizing} show that PINNs' losses have complex optimization landscapes, complicating training despite adequate neural network expressiveness. Approaches like those detailed in \cite{WANG2022_WhenAndWhyPINNsFailed} adopt a Neural Tangent Kernel (NTK) perspective to identify reasons for failure and suggest using adaptive weights during training to enhance performance. Additionally, studies such as \cite{deryckManu2023PreconditioningPINNs} demonstrate that PINNs suffer from ill-conditioned losses, resulting in slow convergence of gradient descent algorithms. Recently, \citep{RathoreICML2024} shows how training strategies can improve the convergence of Physics-Informed Neural Networks and show that specific optimizers such as L-BFGS act as conditioners on the physical losses. 

\paragraph{Supervised training:} In contrast to the unsupervised training of Physics-informed Neural Networks, purely data-driven models have demonstrated remarkable capabilities for PDE simulation and forecasting. In most of the existing literature, the entire solver is replaced by a DL architecture and focuses on directly computing the solution from a given input data. A widely studied setting is operator learning which learns mappings between function spaces \citep{li2020fno, kovachki2023neuralop, Lu_2021_DON}. This method is very efficient, with the downside of relying on quite large quantities of data for training in order to ensure adequate generalization. Additionally, the neural network does not have access to the PDE in itself, only indirectly through the data.
To ensure physical constraints in purely data-driven training, hybrid models have been proposed. The latter relies on both the available physical knowledge and some data. Some examples include the Aphinity model \citep{Yin_2021_aphinity} (where the authors assume partial knowledge of the physics and learned the remaining dynamics from data), Physics-informed Deep Operator Networks (PIDON) \citep{wang2021PIDON, goswami2022PIDON}, Physics-informed Neural Operator (PINO) \citep{li2023pino} (DeepONet architecture \citep{Lu_2021_DON} or Neural Operator models \citep{kovachki2023neuralop, li2020fno} respectively with a combination of data and physical losses). 

\paragraph{Learning to solve:}
Improving the learning scheme and optimizers through data-driven training has been studied since \cite{li2016learning2o} and \cite{andrychowicz2016l2lbygdbygd}. These works propose to learn the optimizer of neural networks, which are classically optimized through gradient-based algorithms such as Adam. They focus on improving training strategies for neural networks, which do not suffer from the optimization issues and ill-conditioning properties of physics-informed losses. We refer the reader to the survey of \cite{chen2021l2o} for a complete overview. The closer work to ours is the very recent work of \cite{bihlo24PasCoolJMLR} in which the author assesses the capabilities of learned optimizers for physics-informed neural networks. The main difference with our work relies on the problem setting. \cite{bihlo24PasCoolJMLR} considers learning an optimizer on a single equation, and for different neural networks initialization, while we focus on efficiently solving several instances of parametric PDE with varying PDE parameters $\gamma, f, g$.

% \citep{bihlo24PasCoolJMLR} => The difference rely in \textcolor{red}{truc le + proche du notre}

% The network $\mathcal{F}_{\varrho}$ is trained to minimize the distance between the final ansatz $u_{K}$ of \cref{eq:NGinfstep} and the associated solution. Technical details are presented in \cref{app:impdetails}.

% % \paragraph{\textit{Relation to preconditioning}}: Ill-conditionned problems are very difficult to optimize through gradient descent algorithms: convergence is very slow. For these kind of problems, it is possible to use \textit{preconditioning} techniques to improves the converge speed of such methods. Most of them apply to linear systems such as the Krylov method. With this perspective, our solver can be seen as \textit{non-linear preconditioning} of the system through the function $\mathcal{F}$. 

\section{Theoretical Analysis of our Method and PINNs}
\label{app:proof}

\paragraph{Setting.} We consider the following linear PDE:

\begin{equation}
\label{eq:PDEproof}
\begin{aligned}
     \mathcal{D} u(x) &= f(x), \quad x \in \Omega,\\
     u(x) &= g(x), \quad x \in \partial \Omega,
\end{aligned}
\end{equation}

where \( \Omega \subset \mathbb{R}^d \) is an open bounded domain, \( \mathcal{D} \) is a linear differential operator, \( f(x) \) is a given function in \( \Omega \), and \( g(x) \) is a given function on the boundary \( \partial \Omega \).

\subsection{Theoretical Analysis of PINNs\label{sec:theoretical_analysis_our_method}}

Our aim is to find an approximate solution \( u_\Theta(x) \), parameterized by \( \Theta \in \mathbb{R}^N \), $\Theta = \{\theta_i\}_{i=1}^N$ that minimizes the loss function:

\begin{equation}
\label{eq:app_loss}
\mathcal{L}_{\textnormal{PDE}}(\Theta) = \mathcal{L}_{\textnormal{Res}}(\Theta) + \lambda \mathcal{L}_{\textnormal{BC}}(\Theta),
\end{equation}

where:

\[
\mathcal{L}_{\textnormal{Res}}(\Theta) = \frac{1}{2} \int_{\Omega} \left( \mathcal{D} u_\Theta(x) - f(x) \right)^2 dx, \quad \mathcal{L}_{\textnormal{BC}}(\Theta) = \frac{1}{2} \int_{\partial \Omega} \left( u_\Theta(x) - g(x) \right)^2 dx,
\]

and \( \lambda > 0 \) is a regularization parameter balancing the PDE residual and boundary conditions.

We perform gradient descent updates with step size \( \eta \). At step $k>0$, updates write as:

\[
\Theta_{k+1} = \Theta_k - \eta \nabla_\Theta \mathcal{L}_{\textnormal{PDE}}(\Theta_k).
\]

We establish the following theorem regarding the convergence rate of gradient descent.
\vspace{.4cm}

\begin{theorem}[\textbf{Convergence rate of PINNs}]
\label{thm:convergence_rate}
Given a linear ansatz \( u_\Theta(x) = \sum_{i=1}^N \theta_i \phi_i(x) \), the number of steps \( N(\varepsilon) \) required to achieve an error \( \| \Theta_k - \Theta^\ast \|_2 \leq \varepsilon \) satisfies:

\begin{equation}
\label{eq:nstepsproof}
N(\varepsilon) = O\left( \kappa(A) \ln \left( \tfrac{1}{\varepsilon} \right) \right),
\end{equation}

where \( \kappa(A) \) is the condition number of the matrix \( A \in \mathbb{R}^{n \times n} \) defined by:
\vspace{-.1cm}
\begin{equation}
\label{eq:A_def}
A_{i,j} = \int_{\Omega} \left( \mathcal{D} \phi_i(x) \right) \left( \mathcal{D} \phi_j(x) \right) dx + \lambda \int_{\partial \Omega} \phi_i(x) \phi_j(x) dx.
\end{equation}

\end{theorem}

\begin{proof}
Since \( u_\Theta(x) = \sum_{i=1}^N \theta_i \phi_i(x) \), we have:

\[
\frac{\partial u_\Theta(x)}{\partial \theta_i} = \phi_i(x), \quad \frac{\partial (\mathcal{D} u_\Theta(x))}{\partial \theta_i} = \mathcal{D} \phi_i(x).
\]

The gradient of the residual loss is:

\[
\nabla_\Theta \mathcal{L}_{\textnormal{Res}}(\Theta) = \int_{\Omega} \left( \mathcal{D} u_\Theta(x) - f(x) \right) \mathcal{D} \phi(x) \, dx,
\]

where \( \mathcal{D} \phi(x) \) is the vector with components \( \mathcal{D} \phi_i(x) \). Similarly, the gradient of the boundary loss is:

\[
\nabla_\Theta \mathcal{L}_{\textnormal{BC}}(\Theta) = \int_{\partial \Omega} \left( u_\Theta(x) - g(x) \right) \phi(x) \, dx,
\]

where \( \phi(x) \) is the vector of basis functions evaluated at \( x \). Therefore, the total gradient is:

\[
\nabla_\Theta \mathcal{L}_{\textnormal{PDE}}(\Theta) = A \Theta - b,
\]

where the (positive, semi-definite) matrix \( A \) and vector \( b \) are defined as:

\begin{equation}
\label{eq:A_and_b}
\begin{aligned}
A_{i,j} &= \int_{\Omega} \left( \mathcal{D} \phi_i(x) \right) \left( \mathcal{D} \phi_j(x) \right) dx + \lambda \int_{\partial \Omega} \phi_i(x) \phi_j(x) dx, \\
b_i &= \int_{\Omega} f(x) \mathcal{D} \phi_i(x) \, dx + \lambda \int_{\partial \Omega} g(x) \phi_i(x) \, dx.
\end{aligned}
\end{equation}

Thus, the gradient descent update becomes:

\[
\Theta_{k+1} = \Theta_k - \eta (A \Theta_k - b).
\]

Subtracting \( \Theta^\ast \) (the optimal parameter vector satisfying \( A \Theta^\ast = b \)) from both sides:

\[
\Theta_{k+1} - \Theta^\ast = \Theta_k - \Theta^\ast - \eta A (\Theta_k - \Theta^\ast).
\]

Simplifying:
\vspace{-.1cm}
\[
\Theta_{k+1} - \Theta^\ast = (\text{Id} - \eta A)(\Theta_k - \Theta^\ast).
\]

By recursively applying the update rule, we obtain:

\[
\Theta_k - \Theta^\ast = (\text{Id} - \eta A)^k (\Theta_0 - \Theta^\ast).
\]

Since \( A \) is symmetric positive definite, it has eigenvalues \( \lambda_1 \leq \lambda_2 \leq \dots \leq \lambda_n \) with \( \lambda_i > 0 \). To ensure convergence, we require \( 0 < \eta < \frac{2}{\lambda_{\max}(A)} \). Choosing \( \eta = \frac{c}{\lambda_{\max}(A)} \) with \( 0 < c < 2 \), we have:

\[
1 - \eta \lambda_i = 1 - c \frac{\lambda_i}{\lambda_{\max}(A)}.
\]

The spectral radius \( \rho \) of \( \text{Id} - \eta A \) is:

\[
\rho = \max \left\{ \left| 1 - c \frac{\lambda_{\min}(A)}{\lambda_{\max}(A)} \right|, \left| 1 - c \right| \right\} = \max \left\{ 1 - \frac{c}{\kappa(A)}, |1 - c| \right\},
\]

where \( \kappa(A) = \frac{\lambda_{\max}(A)}{\lambda_{\min}(A)} \) is the condition number of \( A \). By choosing \( 0 < c < 1 \), we ensure \( |1 - c| < 1 \), and since \( \kappa(A) \geq 1 \), we have \( 1 - \frac{c}{\kappa(A)} < 1 \). Thus, the convergence factor is:

\[
\rho = 1 - \frac{c}{\kappa(A)}.
\]

Therefore:
\vspace{-.07cm}
\[
\| \Theta_k - \Theta^\ast \|_2 \leq \left( 1 - \frac{c}{\kappa(A)} \right)^k \| \Theta_0 - \Theta^\ast \|_2.
\]

To achieve \( \| \Theta_k - \Theta^\ast \|_2 \leq \varepsilon \), the number of iterations \( N(\varepsilon) \) satisfies:

\[
N(\varepsilon) \geq \frac{ \ln \left( \varepsilon / \| \Theta_0 - \Theta^\ast \|_2 \right) }{ \ln \left( 1 - \frac{c}{\kappa(A)} \right) }.
\]

Using the inequality \( \ln(1 - x) \leq -x \) for \( 0 < x < 1 \), we get:

\[
N(\varepsilon) \leq \frac{ \kappa(A) }{ c } \ln \left( \frac{ \| \Theta_0 - \Theta^\ast \|_2 }{ \varepsilon } \right ).
\]

Thus:
\vspace{-.07cm}
\[
N(\varepsilon) = O\left( \kappa(A) \ln \left( \tfrac{1}{\varepsilon} \right) \right).
\]

\end{proof}

We have shown that for a linear ansatz \( u_\Theta(x) = \sum_{i=1}^N \theta_i \phi_i(x) \), the convergence rate of gradient descent depends linearly on the condition number \( \kappa(A) \) of the system matrix \( A \). A large condition number impedes convergence, requiring more iterations to achieve a desired accuracy \( \varepsilon \).

\hfill \(\qedsymbol\)

\subsection{Theoretical Analysis of our Method}

In practice, we often work with multiple data points. For each data point, there is an associated parameter vector \( \Theta \in \mathbb{R}^N \). We are interested in the iterative update where the gradient is transformed by a neural network \( \mathcal{F} \):

\begin{equation}
    \Theta_{l+1} = \Theta_l - \eta \mathcal{F} \left( \nabla_{\Theta} \mathcal{L}_{\textnormal{PDE}}(\Theta_l) \right),
    \label{eq:update_nn_psd}
\end{equation}

where \( \Theta_l \) represents the parameter vector at iteration \( l \). Recall that \( \mathcal{F} \) is trained to minimize the loss after \( L \) iteration steps for $M$ data points:

\begin{equation}
    \mathcal{L}_{\text{DATA}} = \frac{1}{m} \sum_{k=1}^M \left\| u_{\Theta_L^{(k)}} - u^\ast_k \right\|_2^2,
    \label{eq:dataloss_app}
\end{equation}

\vspace{.1cm}
\begin{theorem}{\textbf{(Convergence rate of our method).}}
Given a linear ansatz \( u_\Theta(x) = \sum_{i=1}^N \theta_i \phi_i(x) \), assume \( \mathcal{F} \) behaves like its linearization \(P = \left. \tfrac{\partial \mathcal{F}}{\partial v} \right|_{v=0}
 \). The number of steps \( N'(\varepsilon) \) required to achieve an error \( \| \Theta_l - \Theta^\ast \|_2 \leq \varepsilon \) satisfies:

\begin{equation}
    N'(\varepsilon) = O\left( \kappa(PA) \ln \left( \tfrac{1}{\varepsilon} \right) \right),
\end{equation}

Moreover, if \( \mathcal{F} \) minimizes \( \mathcal{L}_{\text{DATA}} \) this necessarily implies \( \kappa(PA) = 1 \leq \kappa(A) \). Consequently, the number of steps is effectively reduced, i.e., \( N'(\varepsilon) \ll N(\varepsilon) \).
\end{theorem}

\begin{proof}

Since \( \mathcal{F} \) behaves like its linearization \( P \), the gradient descent update becomes (refer to proof of Theorem~\ref{thm:convergence_rate} for steps):

\[
\Theta_{l+1} = \Theta_l - \eta P (A \Theta_l - b).
\]

Let \( \Theta^\ast\) be the optimal parameter vector minimizing \( \mathcal{L}_{\textnormal{PDE}} \). Then, the difference between the parameter vector at iteration \( l \) and the optimal parameter vector is:

\[
\Theta_{l+1} - \Theta^\ast = \Theta_l - \Theta^\ast - \eta P A (\Theta_l - \Theta^\ast) = (\text{Id} - \eta P A)(\Theta_l - \Theta^\ast).
\]

By recursively applying this update until the final step $L$, we obtain:

\[
\Theta_L - \Theta^\ast = (\text{Id} - \eta P A)^L (\Theta_0 - \Theta^\ast).
\]

Since we have multiple data points, each with its own parameter vector, we consider the concatenation when necessary. Let's introduce \( \Xi_l \) as the matrix whose columns are the parameter vectors:

\[
\Xi_L = [\Theta_L^{(1)}, \Theta_L^{(2)}, \dots, \Theta_L^{(m)}].
\]

Similarly, \( \Xi^\ast \) contains the optimal parameter vectors for each data point. The update for all data points can be written collectively:

\[
\Xi_L - \Xi^\ast = (\text{Id} - \eta P A)^L (\Xi_0 - \Xi^\ast).
\]

Since \( \mathcal{F} \) minimizes \( \mathcal{L}_{\text{DATA}} \), we have \( \Xi_L = \Xi^\ast \), implying:

\[
(\text{Id} - \eta P A)^L (\Xi_0 - \Xi^\ast) = 0.
\]

Given that the values of $\Xi_0$ are iid and sampled randomly  from a continuous distribution, because the set of singular matrices has measure zero, the square matrix \( (\Xi_0 - \Xi^\ast)(\Xi_0 - \Xi^\ast)^\top \) is full rank (i.e., invertible), with probability 1. Thus, the only way for the above equality to hold is if:

\[
(\text{Id} - \eta P A)^L = 0.
\]

This means \( \text{Id} - \eta P A \) is nilpotent of index \( L \). Consequently, all eigenvalues of \( \text{Id} - \eta P A \) are zero, implying that all eigenvalues of \( P A \) are equal to \( \dfrac{1}{\eta} \), leading to $\kappa(PA) = \lambda_{\max}(PA) / \lambda_{\min}(PA) = 1
$, which is the optimal condition number. Referring to the convergence analysis in Theorem~\ref{thm:convergence_rate}, we have:

\[
N'(\varepsilon) \leq \frac{\kappa(PA)}{c} \ln \left( \frac{\| \Xi_0 - \Xi^\ast \|_2}{\varepsilon} \right ).
\]

Which directly implies
\begin{equation}
    N'(\varepsilon) = O\left( \kappa(PA) \ln \left( \tfrac{1}{\varepsilon} \right) \right),
\end{equation}

With \( \kappa(PA) = 1 \), this leads us to the desired result:

\[
N'(\varepsilon)  = O\left(\ln \left( \tfrac{1}{\varepsilon} \right) \right) \ll O\left( \kappa(A) \ln \left( \tfrac{1}{\varepsilon} \right) \right) =: N(\epsilon),
\]

Thus, the number of iterations required is significantly reduced compared to the case without the neural network preconditioner.

\end{proof}

\vspace{1cm}
\paragraph{Discussion}

The convergence proofs for our method fundamentally rely on the assumption of linearity in the underlying problems. It is important to note that the theoretical analysis does not extend to non-linear cases. Consequently, for non-linear scenarios, the theory should be viewed primarily as a tool for building intuition or providing motivation, rather than a definitive proof. This is due to the lack of established methods for rigorously studying the non-linear regime, as no known results currently address such cases.
\clearpage
Under these conditions: 
\begin{itemize}
    \item This optimal condition number implies that \textbf{convergence is not only guaranteed but also optimal, requiring fewer iterations.}
    \item \textbf{Guaranteed Convergence}: The method reliably achieves convergence to the optimal solution due to the reduced condition number.
    \item \textbf{Optimal Convergence Speed}: With \(\kappa(PA) = 1\), the neural network provides an enhanced convergence rate, resulting in fewer required iterations compared to the original system without the neural network. 
\end{itemize}
\hfill \(\qedsymbol\)

\clearpage
\section{Dataset details}
\label{app:appdataset}
For all datasets, we kept $800$ samples for training and $200$ as testing examples (except otherwise stated in the experiments). 

\subsection{Helmholtz}
We generate a dataset following the $1d$ static Helmholtz equation \cref{eq:helmholtz}. 
For $x \in [0, 1[$, 
    \begin{equation}
        \begin{cases}
             u (x)'' + \omega^2u(x) &= 0, \\
            u(0) &= u_0, \\
            u'(0) &= v_0.
        \end{cases}
        \label{eq:helmholtz}
    \end{equation}
The solution can be analytically derived: $u(x) = \alpha\cos(\omega x + \beta)$, with $\beta = \arctan(\frac{-v_0}{\omega u_0})$, $\alpha = \frac{u_0}{\cos(\beta)}$ and directly computed from the PDE data.  We generate $1,024$ trajectories for training and $256$ for testing with $u_0, v_0 \sim \mathcal{N}(0, 1)$, and $\omega \sim \mathcal{U}(0.5, 50)$ and compute the solution on $[0, 1]$ with a spatial resolution of $256$. For training, we keep $800$ samples and use the complete dataset for the additional experiments presented in \cref{fig:abntrain_h,fig:abntrain_d}. Moreover, we sub-sample the spatial resolution by $4$ and keep $64$ points for training. 

\begin{figure}[htbp]
    % \vspace{-1.6cm}
    \begin{subfigure}{0.48\textwidth} % specify the width of subfigure
        \centering
        \includegraphics[width=\textwidth]{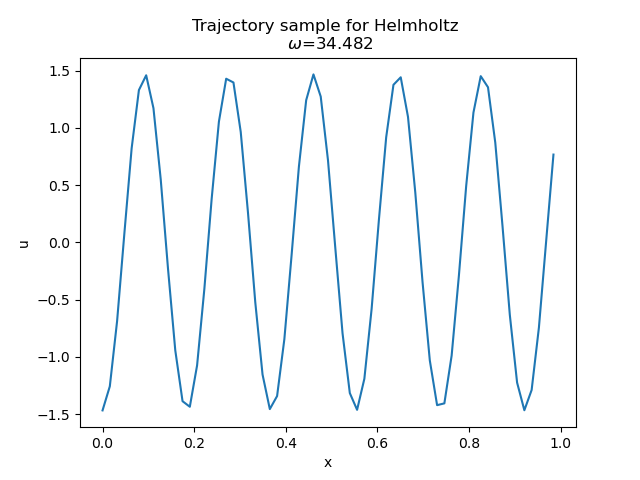}
        \caption{}
        \label{fig:vish1}
  \end{subfigure}
  \hfill
  \begin{subfigure}{0.48\textwidth} % specify the width of subfigure
        \centering
        \includegraphics[width=\textwidth]{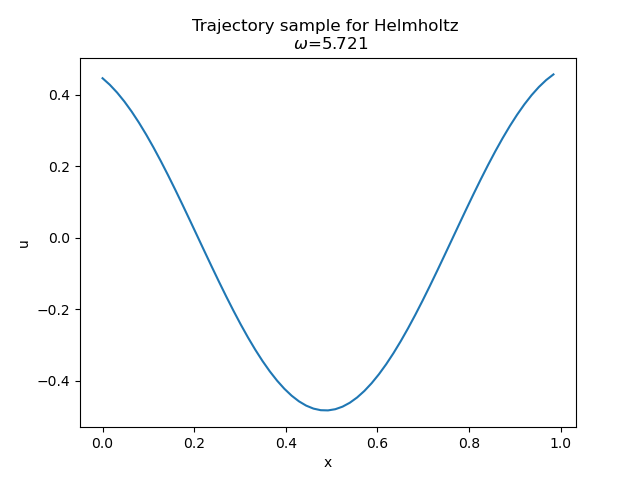}
        \caption{}
        \label{fig:vish2}
  \end{subfigure}
  \caption{Samples from the Helmholtz Dataset. }
\end{figure}

\subsection{Poisson}

We generate a dataset following the $1d$ static Poisson equation \cref{eq:poisson} with forcing term. For $x \in [0, 1[$, 
    \begin{equation}
        \begin{cases}
            - u''(x) &= f(x), \\
            u(0) &= u_0,\\
            u'(0) &= v_0.
        \end{cases}
        \label{eq:poisson}
    \end{equation}
We chose $f$ to be a non-linear forcing term: $f(x) = \frac{\pi}{K}\sum_{i=1}^K a_ii^{2r}\sin(\pi x)$, with $a_i\sim \mathcal{U}(-100, 100)$, we used $K=16$, $r=-0.5$, and solve the equation using a backward finite difference scheme. We generate $1,000$ trajectories with $u_0, v_0 \sim \mathcal{N}(0, 1)$ and compute the solution on $[0, 1]$ with a spatial resolution of $64$. 

\begin{figure}[htbp]
    % \vspace{-1.6cm}
    \begin{subfigure}{0.48\textwidth} % specify the width of subfigure
        \centering
        \includegraphics[width=\textwidth]{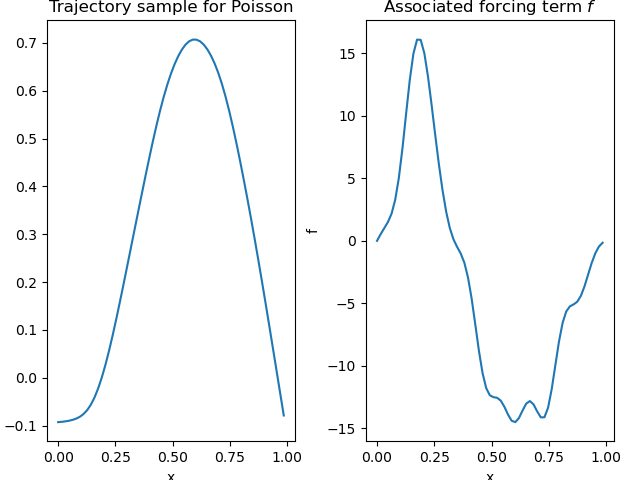}
        \caption{}
        \label{fig:visp1}
  \end{subfigure}
  \hfill
  \begin{subfigure}{0.48\textwidth} % specify the width of subfigure
        \centering
        \includegraphics[width=\textwidth]{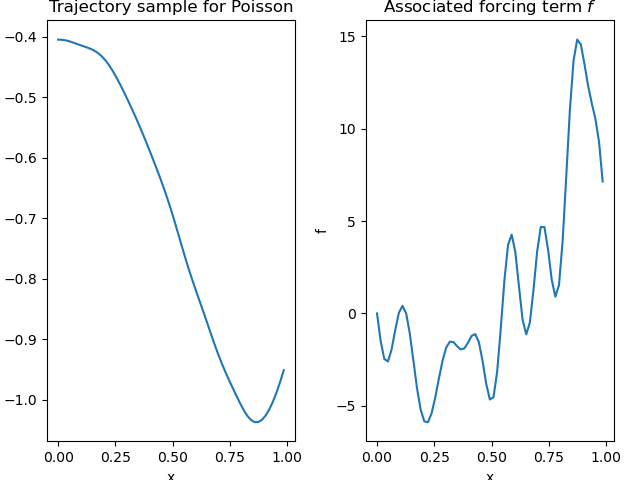}
        \caption{}
        \label{fig:visp2}
  \end{subfigure}
  \caption{Samples from the Poisson Dataset. }
\end{figure}

\paragraph{Reaction-Diffusion}
We use a non-linear reaction-diffusion used in \citep{krishnapriyan_characterizing, toloubidokhti2024dats}. This PDE has been shown to be a failure case for PINNs \citep{krishnapriyan_characterizing}. The PDE states as follows:
\begin{align}
    \frac{\partial u(t, x)}{\partial t} - \nu \frac{\partial^2u(t, x)}{\partial x^2} - \rho u(t, x)(1-u(t, x)) &= 0, \\
    u(0, x) = e^{-32(x-1/2)^2}.
\end{align}

We generate $800$ trajectories by varying $\nu$ in $[1, 5]$ and $\rho$ in $[-5, 5]$. Spatial resolution is $256$ and temporal resolution is $100$, which we sub-sample by $4$ for training, leading to a spatial resolution of $64\times25$. The PDE is solved on $[0, 1]^2$ as in \citep{toloubidokhti2024dats}.  

\begin{figure}[htbp]
    % \vspace{-1.6cm}
    \begin{subfigure}{0.48\textwidth} % specify the width of subfigure
        \centering
        \includegraphics[width=\textwidth]{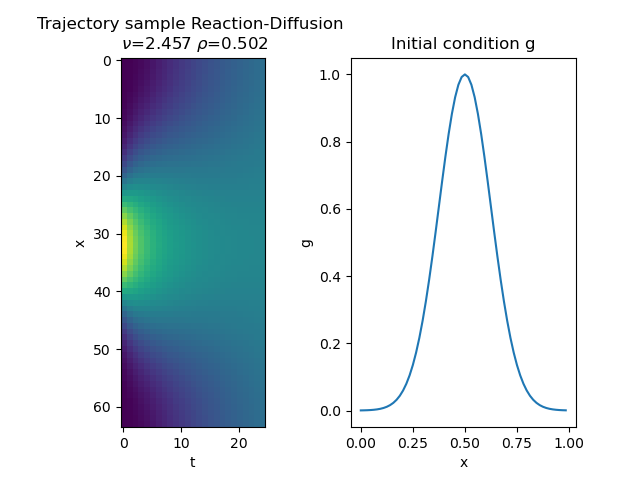}
        \caption{}
        \label{fig:visrd1}
  \end{subfigure}
  \hfill
  \begin{subfigure}{0.48\textwidth} % specify the width of subfigure
        \centering
        \includegraphics[width=\textwidth]{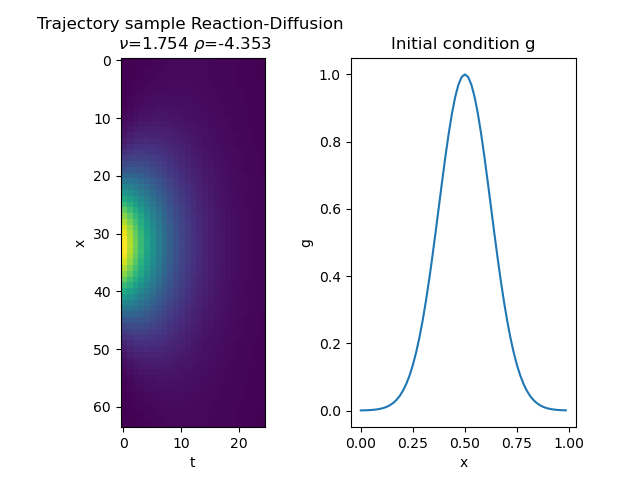}
        \caption{}
        \label{fig:visrd2}
  \end{subfigure}
  \caption{Samples from the Reaction-Diffusion Dataset. }
\end{figure}

\paragraph{Reaction-Diffusion with initial conditions:} To complexify the setting, we also change the initial condition of the problem (\textit{NLRDIC} in the following). The initial condition is expressed as follows:
\begin{equation}
    u(x, 0) = \sum_{i=1}^3 a_i e^{-\frac{\left( \frac{x-h/4}{h}\right)^2}{4}}.
\end{equation}
Where $a_i$ are randomly chosen in $[0, 1]$ and $h=1$ is the spatial resolution. 

\begin{figure}[htbp]
    % \vspace{-1.6cm}
    \begin{subfigure}{0.48\textwidth} % specify the width of subfigure
        \centering
        \includegraphics[width=\textwidth]{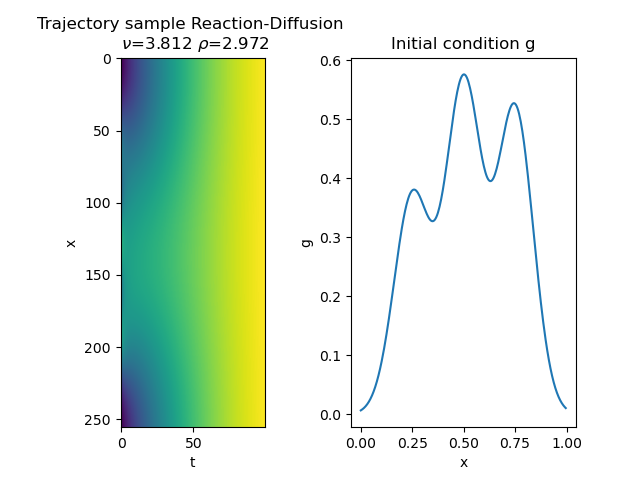}
        \caption{}
        \label{fig:visrdics1}
  \end{subfigure}
  \hfill
  \begin{subfigure}{0.48\textwidth} % specify the width of subfigure
        \centering
        \includegraphics[width=\textwidth]{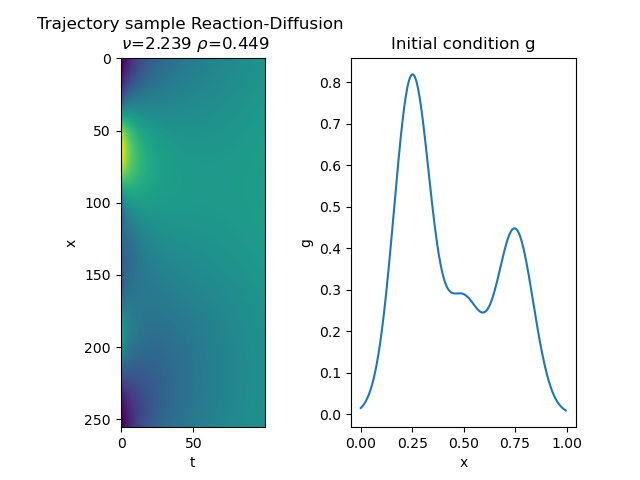}
        \caption{}
        \label{fig:visrdics2}
  \end{subfigure}
  \caption{Samples from the Reaction-Diffusion Dataset with initial conditions changed. }
\end{figure}

\vspace{4cm}
\subsection{Darcy Flow}

The $2d$ Darcy Flow dataset is taken from \citep{li2020fno} and is commonly used in the operator learning literature \citep{li2023pino, goswami2022PIDON}. 
\begin{align}
    -\nabla.(a(x)\nabla u(x)) &= f(x) \hspace{3mm} x \in (0, 1)^2,\\
    u(x) &= 0 \hspace{3mm} x \in \partial(0, 1)^2.
\end{align}
For this dataset, the forcing term $f$ is kept constant $f=1$, and $a(x)$ is a piece-wise constant diffusion coefficient taken from \citep{li2020fno}. We kept $1,000$ trajectories (on the $5,000$ available) with a spatial resolution of $64\times 64$. 

\begin{figure}[htbp]
    % \vspace{-1.6cm}
    \begin{subfigure}{0.48\textwidth} % specify the width of subfigure
        \centering
        \includegraphics[width=\textwidth]{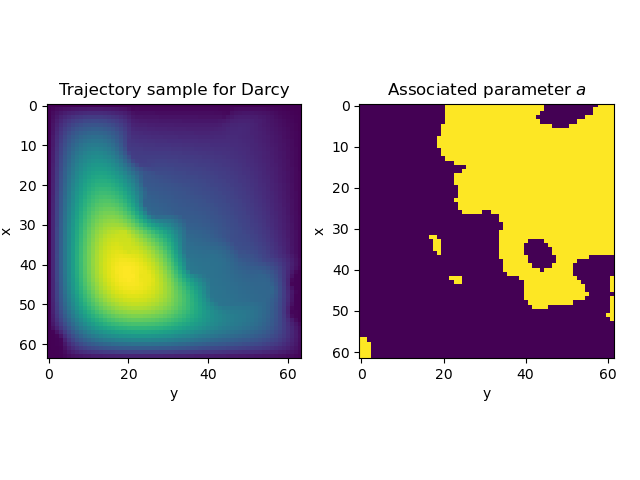}
        \caption{}
        \label{fig:visd1}
  \end{subfigure}
  \hfill
  \begin{subfigure}{0.48\textwidth} % specify the width of subfigure
        \centering
        \includegraphics[width=\textwidth]{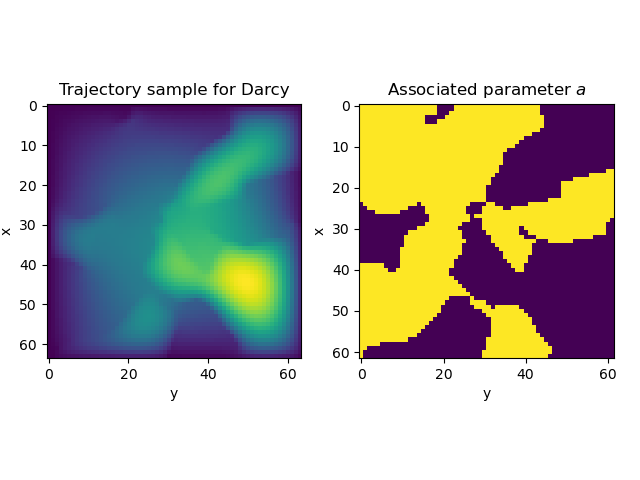}
        \caption{}
        \label{fig:visd2}
  \end{subfigure}
  \caption{Samples from the Darcy Dataset. }
\end{figure}

% \subsection{Helmholtz}
% The Helmholtz equation has been proposed by \citep{toloubidokhti2024dats}. The PDE states as follows: 
% \begin{align}
%     \nabla^2 u(x, y) + k^2u(x, )& = q(x, y) \\
%     u(x, y)& = h(x, y)
% \end{align}

\subsection{Heat}
As proof that our method can handle $2$d + time, we consider the dataset proposed by \citep{maepde24}. 
\begin{align}
    \frac{\partial u (x, y, t)}{\partial t} - \nu \nabla^2 u(x, y, t) &= 0,\\
    u(x, y, 0) &= \sum_{j=1}^J A_j\sin(\frac{2\pi l_{xj}x}{L} + \frac{2\pi l_{yj}y}{L} + \phi_i).
\end{align}
Where $L=2$, $\nu$ is randomly chosen between $[2\times10^{-3}, 2\times10^{-2}]$, $A_j$ in $[-0.5, 0.5]$, $l_{xj}, l_{xy}$ are integers in $\{1, 2, 3\}$ and $\phi$ is in $[0, 2\pi]$. As a difference with \citep{maepde24}, we randomly chose $J$ between $1$ and $J_{max}=5$ to have more diversity in the represented frequencies in the data. The PDEs are sampled with spatial resolution $64$ in x and y and temporal resolution $100$. However, during training the spatial resolution are subsampled by $4$ and the coordinates are re scaled between 0 and 1. As for other PDEs, we use $800$ trajectories for training and $200$ for testing. 

\begin{figure}[htbp]
    % \vspace{-1.6cm}
    \begin{subfigure}{0.48\textwidth} % specify the width of subfigure
        \centering
        \includegraphics[width=\textwidth, trim={3.5cm 4cm 4cm 4cm}, clip]{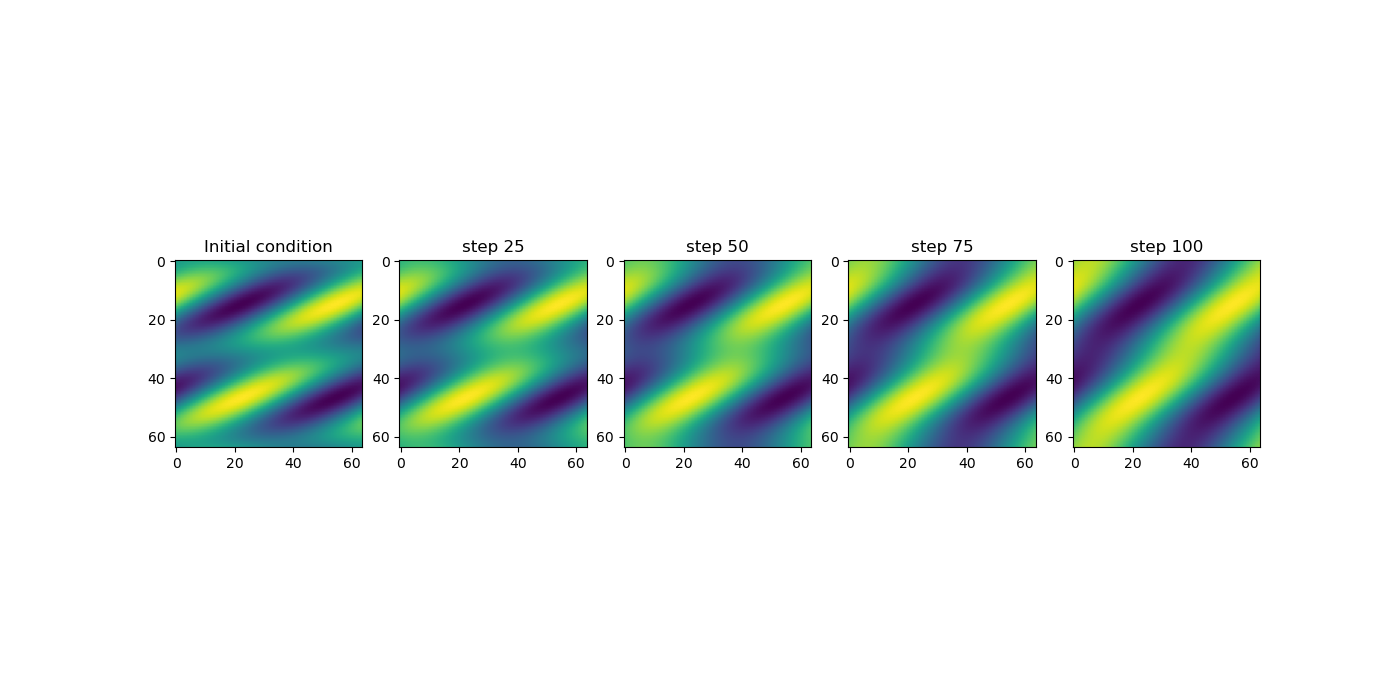}
        \caption{}
        \label{fig:visheat1}
  \end{subfigure}
  \hfill
  \begin{subfigure}{0.48\textwidth} % specify the width of subfigure
        \centering
        \includegraphics[width=\textwidth, trim={3.5cm 4cm 4cm 4cm}, clip]{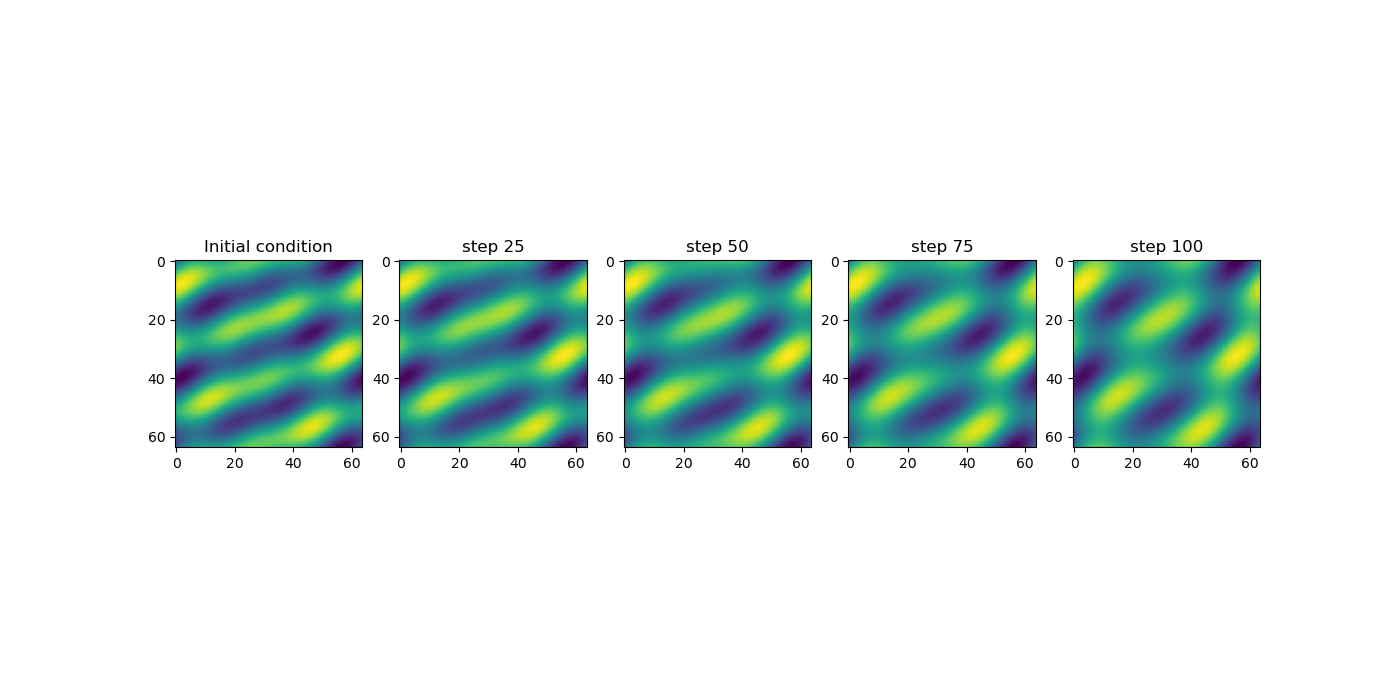}
        \caption{}
        \label{fig:visheat2}
  \end{subfigure}
  \caption{Samples from the Heat dataset. }
\end{figure}

\subsection{Additional dataset: Advection}
The dataset is taken from \citep{takamoto2023pdebench}. 
\begin{align}
    \frac{\partial u(t, x)}{\partial t} + \beta \frac{\partial u(t, x)}{\partial x} &= 0, \hspace{3mm} x\in (0, 1), t\in (0, 2], \\
    u(0, x) &= u_0(x), \hspace{3mm} x\in (0, 1).
\end{align}
Where $\beta$ is a constant advection speed, and the initial condition is $u_0(x) = \sum_{k_i=k_1...k_N} A_i\sin(k_ix+\phi_i)$, with $k_i=\frac{2\pi{n_i}}{L_x}$ and ${n_i}$ are randomly selected in $ [1,8]$. The author used $N=2$ for this PDE. Moreover, $A_i$ and $\phi_i$ are randomly selected in $[0, 1]$ and $(0, 2\pi)$ respectively. Finally, $L_x$ is the size of the domain \citep{takamoto2023pdebench}.

The PDEBench's Advection dataset is composed of several configurations of the parameter $\beta$ ($\{0.1, 0.2, 0.4, 0.7, 1, 2, 4, 7\}$), each of them is composed of $10,000$ trajectories with varying initial conditions. From these datasets, we sampled a total of $1,000$ trajectories for $\beta \in \{0.2, 0.4, 0.7, 1, 2, 4\}$ (which gives about $130$ trajectories for each $\beta$). This gives a dataset with different initial conditions and parameters. Moreover, during training, we sub-sampled the trajectories by $4$, leading to a grid of resolution $25$ for the t-coordinate and $256$ for the x-coordinate. 

\begin{figure}[htbp]
    % \vspace{-1.6cm}
    \begin{subfigure}{0.48\textwidth} % specify the width of subfigure
        \centering
        \includegraphics[width=\textwidth]{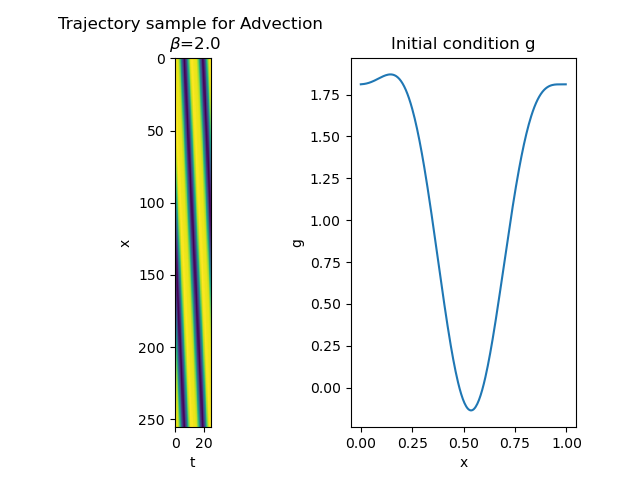}
        \caption{}
        \label{fig:visa1}
  \end{subfigure}
  \hfill
  \begin{subfigure}{0.48\textwidth} % specify the width of subfigure
        \centering
        \includegraphics[width=\textwidth]{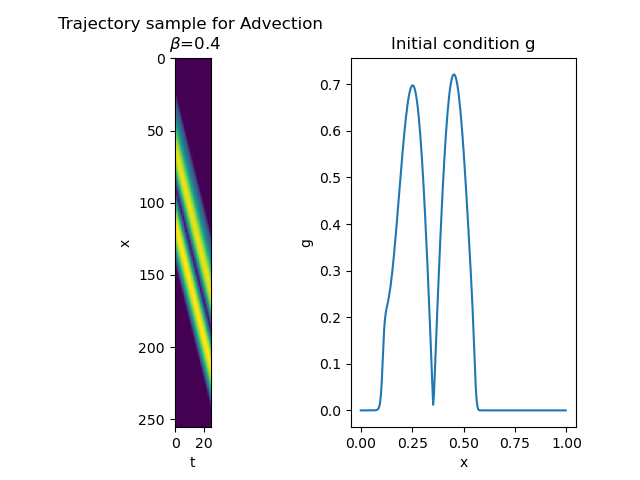}
        \caption{}
        \label{fig:visa2}
  \end{subfigure}
  \caption{Samples from the Advections Dataset. }
\end{figure}

\color{black}

\clearpage
\subsection{Summary of problem settings considered}
A summary of the datasets and parameters changing between $2$ trajectories is presented in \cref{tab:parameters_datasets_app}. 
\begin{table}[htbp]
    \centering
    \begin{tabular}{c|c|c}
         Dataset & Changing PDE data & Range / Generation\\
         \toprule
         \multirow{3}{*}{Helmholtz} & $\omega$ & $[0.5, 50]$\\
         & $u_0$ & $\mathcal{N}(0, 1)$\\
         & $v_0$ & $\mathcal{N}(0, 1)$\\
         \midrule
         \multirow{3}{*}{Poisson} & $A_i$ & $[-100, 100]$\\
         & $u_0$ & $\mathcal{N}(0, 1)$\\ 
         & $v_0$ & $\mathcal{N}(0, 1)$\\
         \midrule
         \multirow{2}{*}{Reaction-diffusion} & $\nu$ & $[1, 5]$ \\
         & $\rho$ & $[-5, 5]$\\
         \midrule
         \multirow{2}{*}{Darcy} & \multirow{2}{*}{$a(x)$} & $\psi_{\#}\mathcal{N}(0, (-\Delta + 9I)^{-2})$ \\
         & & with $\psi = 12*\mathds{1}_{\mathbb{R}_+} + 3*\mathds{1}_{\mathbb{R}_+}$\\
         \midrule
         %\color{red}
         % \multirow{1}{*}{Helmholtz-2d} & & \\
         % \midrule
         \multirow{5}{*}{Heat} & $\nu$ & $[2\times 10^{-3}, 2\times 10^{-2}]$ \\
         & $J_{max}$ & \{1, 2, 3, 4, 5\}\\
         & $A$ & $[0.5, -0.5]$\\
         & $K_x$, $K_y$ & $\{1, 2, 3\}$\\
         & $\phi$ & $[0, 2\pi]$ \\
         %\color{black}
         \bottomrule
        \multirow{4}{*}{Advection} & $\beta$ & $\{0.2, 0.4, 0.7, 1, 2, 4\}$\\ 
         & $A_i$ & $[0, 1]$\\
         & $\phi_i$ & $[0, 2\pi]$\\
         & $k_i$ & $\{2k\pi\}_{k=1}^8$\\
         \midrule
         \multirow{3}{*}{NLRDIC} &$\nu$ & $[1, 5]$ \\
         & $\rho$ & $[-5, 5]$\\
         & $a_i$ & $[0, 1]$ \\
         \bottomrule
    \end{tabular}
    \caption{Parameters changed between each trajectory in the considered datasets in the main part of the paper as well as additional datasets (Advections and NLRDIC). }
    \label{tab:parameters_datasets_app}
\end{table}

\clearpage
\section{Implementation details}

\label{app:impdetails}

We add here more details about the implementation and experiments presented in section \ref{sec:expe}. \label{section-implementation-details}

We implemented all experiments with PyTorch \citep{PyTorch2}. %The code is available at \textcolor{red}{https://anonymous.4open.science/r/NeurIPS2024-494F/}.
We estimate the computation time needed for development and the different experiments to be approximately $300$ days. 

\subsection{B-Spline basis}
We chose to use a B-Spline basis to construct the solution. We manually build the spline and compute its derivatives thanks to the formulation and algorithms proposed in \citep{PiegTill96_NURBS}. We used Splines of degree $d = 3$ and constructed the Splines with $2$ different configurations: \begin{itemize}
    \item Take $N + d + 1$ equispaced nodes of multiplicity $1$ from $\frac{d}{N}$ to $1 + \frac{d}{N}$. This gives a smooth local basis with no discontinuities (see \cref{fig:sh-basis}) represented by a shifted spline along the x-axis (denoted as \verb|shifted| in the following). 
    \item Use $N + 1 - d$ nodes of multiplicity $1$ and $2$ nodes of multiplicity $d$ (typically on the boundary nodes: $0$ and $1$). This means that nodes $0$ and $1$ are not differentiable (see \cref{fig:eq-basis}). We call this set-up \verb|equispaced|. 
\end{itemize}

    \begin{figure}[htbp]
    \centering
    \begin{subfigure}{.5\textwidth}
      \centering
      \includegraphics[width=\linewidth]{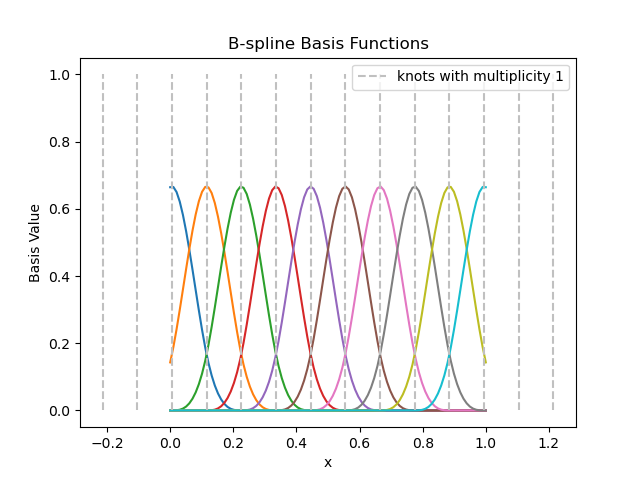}
      \caption{}
      \label{fig:sh-basis}
    \end{subfigure}%
    \begin{subfigure}{.5\textwidth}
      \centering
      \includegraphics[width=\linewidth]{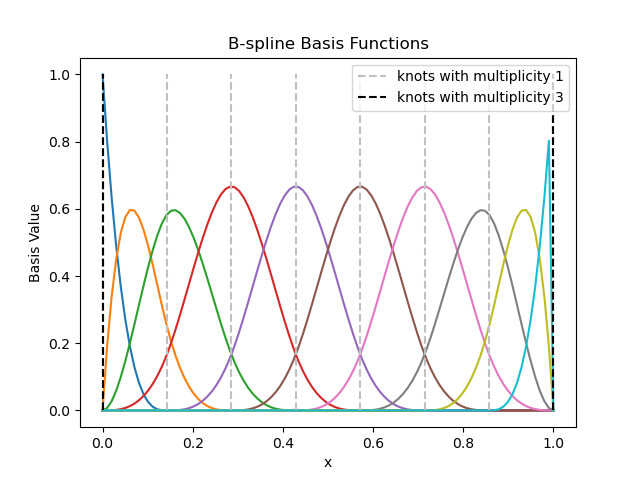}
      \caption{}
      \label{fig:eq-basis}
    \end{subfigure}
    \caption{B-spline basis with $N=10$ terms with shifted spline (Left) and higher multiplicity nodes (Right). Dashed lines represent nodes' position with color the darker, the higher the multiplicity.}
    \label{fig:bspline-basis}
\end{figure}

\paragraph{Higer-dimension basis}
For $1$d + time, $2$d dataset, and $2$d + time we build a $2$ (or $3$) dimensional B-spline basis, \textit{i.e.}, we treat the time coordinates as spatial ones. To build such bases, we compute the Cartesian product between the $2$ (or $3$) $1$d-bases, $1$ per dimension. This means that for a $2$d dataset, for which we chose to use bases of size $N1$ and $N2$ for each coordinate, the resulting basis will have $N1 + N2 + N1*N2$ terms. This method makes the training more costly and several techniques to improve its scalability could be used. For $3$d datasets, the number of terms in the basis is cubic. 

\subsection{Training details}
In our experiments, neural networks are trained using the Adam optimizer. For network optimization, we employ a smooth $l1$-loss for our solver while for other baselines, we use MSE loss and/or physical losses. All models are trained for at least $1,500$ epochs on datasets composed of some sampling of $\gamma$ and/or $f$ and/or $g$. If not stated otherwise, we train our proposed method for $750$ epochs and baselines for $1,500$ epochs. We use the Adam optimizer and an initial learning rate of $0.001$. We use an exponential learning rate scheduler that lowers the learning rate by $0.995$ every epoch. Experiments are conducted on NVIDIA TITAN V ($12$ Go) for $1$d datasets and NVIDIA RTX A6000 GPU with 49Go for $1d$ + time, $2$d datasets or $2$d + time. We recall in \cref{alg:trngd_app} the pseudo-code for training our proposed method. 

% \begin{minipage}[c]{0.49\linewidth}
\RestyleAlgo{ruled}
\begin{algorithm}[H]
    \caption{Training algorithm for learning to optimize physics-informed losses.}
    \label{alg:trngd_app}
    \KwData{$\mpar_0 \in \mathbb{R}^n$, PDE ($\gamma, f, g)$, sample values $u(x)$}
    \KwResult{$\mathcal{F}_{\varrho}$}
    \For{$e = 1... $ epochs}  {
    \For{(\textnormal{PDE}, x, u) in dataset}{
          
             Initialize $\mpar_0$\\
             Estimate $\mpar_L$ from $\mpar_0, (\gamma, f, g)$ using Algorithm \ref{alg:infngd} \\
             Reconstruct $u_{\mpar_L}(x)$ \\
             Update the parameters $\varrho$ of $\mathcal{F}$ with gradient descent from the data loss in Equation 10.

         }
     }
     \Return $\mathcal{F}_{\varrho}$
    \end{algorithm}

We make use of two nested components: the solver for providing the approximate solution to the PDE and the optimizer that conditions the training of the solver. Both are using gradient descent but with different inputs and objectives. The former optimizes the PDE loss (inner loop), while the latter optimizes the gradient steps of the solver through conditioning (outer loop). Optimization of the two components proceeds with an alternate optimization scheme. In particular, this implies that the map $F_{\varrho}$ is kept fixed during the inner optimization process. 
    
% \end{minipage}
% \begin{minipage}[c]{0.49\linewidth}
% \RestyleAlgo{ruled}
% \begin{algorithm}[H]
%     \caption{Inference algorithm for optimizing physics-informed losses.}
%     \label{alg:infngd}
%     \KwData{$\theta_0 \in \mathbb{R}^n$, PDE ($\gamma, f, g)$}
%     \KwResult{$\theta_L \in \mathbb{R}^n$}
%         \For{l = 0...L-1}{
%             $\theta_{l+1} = \theta_l - \eta \mathcal{F}_{\varrho}( \nabla\mathcal{L}_{\textnormal{PDE}}(\theta_l), \gamma, f, g)$
%         }
%         \Return $\theta_{L}$
%     \end{algorithm}
% \end{minipage}

\subsection{Models}
We present here the training details for our model and the baseline. The results can be found in \cref{tab:test-loss}. In all our experiments, we use a GeLU \citep{Hendrycks2016_GELU} activation function. The details of the model architecture on each dataset are presented in \cref{tab:archbase}. 

\begin{table}[htbp]
    \centering
    \begin{tabular}{c|c|c|c|c|c|c}
    \multirow{3}{*}{Model} & \multirow{3}{*}{Architecture} & \multicolumn{5}{c}{Dataset} \\
    \cmidrule(lr){3-7}
    & & \multicolumn{2}{c}{1d} & 1d + time & 2d & 2d + time \\ 
    \cmidrule(lr){3-4} \cmidrule(lr){5-5} \cmidrule(lr){6-6} \cmidrule{7-7}
     & & Helmholtz & Poisson & NLRD & Darcy & Heat \\ 
     \toprule
     \multirow{2}{*}{PINNs} & MLP depth & 3 & 3 & 3 & 3 & 3\\
     & MLP width & 256 & 256 & 256 & 256 & 256\\
     \midrule
     \multirow{2}{*}{PPINNs} & MLP depth & 8 & 3 & 5 & 3 & 3 \\
     & MLP width & 64 & 256 & 256 & 256 & 256\\
     \midrule
     \multirow{4}{*}{PO-DeepONet} & Branch Net depth & 3 & 2 & 5 & 5 & 5 \\
     & Branch Net width & 256 & 256 & 256 & 256 & 256 \\
     & Trunk Net depth & 3 & 2 & 5 & 5 & 5\\
     & Trunk Net width & 256 & 256 & 256 & 256 & 256\\
     \midrule
     \multirow{8}{*}{P2INNs} & Enc params depth & 4 & 4 & 4 & 4 & 4\\
     & Enc params width & 256 & 256 & 256 & 256 & 256 \\
     & Emb params & 128 & 128 & 128 & 128 & 128 \\
     & Enc coord depth & 3 & 3 & 3 & 3 & 3\\
     & Enc coord width & 256 & 256 & 256 & 256 & 256 \\
     & Emb coord & 128 & 128 & 128 & 128 & 128\\
     & Dec depth & 6 & 6 & 6 & 6 & 6\\
     & Dec width & 256 & 256 & 256 & 256 & 256\\
     & Activations & GeLU & GeLU & GeLU & GeLU & GeLU \\
     \midrule
     \multirow{2}{*}{MLP + basis} & MLP depth & 5 & 5 & 5 & 5 & 5\\
     & MLP width & 256 & 256 & 1,024 & 1,024 & 1,024\\ 
     \midrule
     \multirow{4}{*}{PI-DeepONet} & Branch Net depth & 3 & 2 & 5 & 5 & 5\\
     & Branch Net width & 256 & 256 & 256 & 256 & 5\\
     & Trunk Net depth & 3 & 2 & 5 & 5 & 5\\
     & Trunk Net width & 256 & 256 & 256 & 256 & 256\\
     \midrule
     \multirow{6}{*}{PINO} & FNO depth & 3 & 3 & 4 & 3 & 3 \\
     & FNO width & 64 & 64 & 64 & 64 & 64\\
     & FNO modes 1 & 16 & 16 & 10 & 20 & 7\\
     & FNO modes 2 & - & - & 5 & 20 & 7 \\
     & FNO modes 3 & - & - & - & - & 5 \\
     & FNO fc dim & 64 & 64 & 64 & 64 & 64\\
     \midrule
     \multirow{6}{*}{Ours} & FNO depth & 3 & 3 & 3 & 3 & 3\\
     & FNO width & 64 & 64 & 64 & 64 & 64\\
     & FNO modes 1 & 16 & 16 & 10 & 20 & 7\\
     & FNO modes 2 & - & - & 5 & 20 & 7\\
     & FNO modes 3 & - & - & - & - & 5\\
     & FNO fc dim & 64 & 64 & 64 & 64 & 64\\     
     \bottomrule
    \end{tabular}
    \caption{Architecture details of our model and baselines }
    \label{tab:archbase}
\end{table}

\paragraph{Hyperparameters for Helmholtz: } For the baselines, we empirically searched hyperparameters to allow the network to handle the high frequencies involved in the \textit{Helmholtz} dataset. Unfortunately, other network architectures did not improve the results. We proceed similarly for other datasets. 

\paragraph{Training of PPINNs: } For the Helmholtz dataset, we trained our model for 2000 epochs with a plateau learning rate scheduler with patience of 400 epochs. For the Poisson dataset, we trained our parametric PINN model for 5000 epochs with a cosine annealing scheduler with a maximum number of iterations of 1000. For the Reaction-diffusion dataset, we consider an initial learning rate of $0.0001$ instead of $0.001$.
Finally, \textit{P2INNs} is trained for $5,000$ epochs using the Adam optimizer and an exponential scheduler with patience $50$. 

\paragraph{PINNs baselines (PINNs + L-BFGS and PINNs-multi-opt): } The conditioning of the problem highly depends on the parameters of the PDE and initial/boundary conditions. This can lead to unstable training when optimizing PINNs, requiring specific parameter configurations for each PDE. To avoid extensive research of the best training strategy, we found a configuration that allowed a good fitting of \textbf{most of the testing dataset}. This means that in the values reported in \cref{tab:test-loss}, we removed trainings for which the losses exploded (only a few hard PDEs were removed, typically between none to $20$). This prevents us from extensive hyper-parameter tuning on each PDE. Please note that this lowers the reported relative MSE, thus advantaging the baseline. %Finally, for computational reasons, we fitted the PINNs for $1,500$ epochs depending on the dataset. This can explain the observed performances of this baseline since PINNs are traditionally trained for more than $10,000$ epochs. 
%Moreover, Since L-BFGS is attracted by saddle points \citep{RathoreICML2024}, we consider a higher number of steps for Adam rather than L-BFGS to better avoid the attractions produced by saddle points and lastly we use L-BFGS.
We detail in \cref{tab:pinnslbfgshp,tab:pinnsmultihp} the hyper-parameters for each dataset.
% \color{red}
% As a more challenging PINNs baseline, we also fit one PINNs per PDE using the training strategy proposed by \citep{RathoreICML2024}. Since L-BFGS is attracted by saddle points \citep{RathoreICML2024}, we consider a higher number of steps for Adam rather than L-BFGS to better avoid the attractions produced by saddle points and lastly we use L-BFGS. We detail in \cref{tab:pinnsmultihp} the hyper parameters of each training. 

%Similarly to \textit{PINNs+L-BFGS}, we could not optimize the PINNs on the entire dataset for a long time for computational reasons. Thus the NNCG optimizer did not brought any improvements because it requires to be used after a long L-BFGS optimization (it requires a sufficiently good initialization as explained in \citep{RathoreICML2024}). Finally, since we perform a small number of steps and L-BFGS is attracted by saddle points, we consider a higher number of steps for Adam rather than L-BFGS to better avoid the attractions produced by saddle points and lastly we use L-BFGS. We detail in \cref{tab:pinnsmultihp} the hyper parameters of each training. 

\color{black}

\begin{table}[htbp]
    \centering
    \begin{tabular}{c|c|c|c|c|c}
    \multirow{3}{*}{Hyper-parameter} & \multicolumn{5}{c}{Dataset} \\
    \cmidrule(lr){2-6}
    & \multicolumn{2}{c}{1d} & 1d + time & 2d & 2d + time \\ 
    \cmidrule(lr){2-3} \cmidrule{4-4} \cmidrule(lr){5-5} \cmidrule(lr){6-6} 
     & Helmholtz & Poisson & NLRD & Darcy & Heat \\ 
     \toprule
     epochs & $1,000$ & $1,000$ & $1,500$ & $1,500$ & $1,000$\\
     learning rate & 1e-4 & 1e-5 & 1e-4 & 1e-3 & 1e-3 \\
     \bottomrule
    \end{tabular}
    \caption{Hyperparameters for \textit{PINNs+L-BFGS} baseline. }
    \label{tab:pinnslbfgshp}
\end{table}

\begin{table}[htbp]
    \centering
    \begin{tabular}{c|c|c|c|c|c}
    \multirow{3}{*}{Hyper-parameter} & \multicolumn{5}{c}{Dataset} \\
    \cmidrule(lr){2-6}
    & \multicolumn{2}{c}{1d} & 1d + time & 2d & 2d + time \\ 
    \cmidrule(lr){2-3} \cmidrule{4-4} \cmidrule(lr){5-5} \cmidrule(lr){6-6} 
     & Helmholtz & Poisson & NLRD & Darcy & Heat \\ 
     \toprule
     epochs Adams & $800$ & $800$ & $800$ & $1,200$ & $1,200$\\
     epochs L-BFGS & $200$ & $200$ & $200$ & $300$ & $300$\\
     epochs total & $1,000$ & $1,000$ & $1,000$ & $1,500$ & $1,500$\\
     learning rate Adam & 1e-4 & 1e-5 & 1e-4 & 1e-3 & 1e-3 \\
     learning rate L-BFGS & 1 & 1 & 1e-3 & 1 & 1 \\
     \bottomrule
    \end{tabular}
    \caption{Hyperparameters for \textit{PINNs-multi-opt} baseline. }
    \label{tab:pinnsmultihp}
\end{table}

\paragraph{Basis configuration: } For all datasets, we use Splines of degree $3$, built with \verb|shifted| nodes. We change the number of terms in the basis depending on the problem considered. For the $1d$-problem, we use $32$ terms. For the $2d$-problem, $40$ elements are in the basis of each dimension, except for \textit{Reaction-Diffusion} where the variable t has $20$ terms. 
Moreover, during training, we use the projection of the initial conditions and/or parameters and/or forcing terms function in the basis as input to the networks. Finally, for experiments using the \textit{Heat} dataset (\textit{i.e.} 3d basis), we used $15$ terms for the $x$ and $y$ spatial coordinates and $10$ terms for the $t$-coordinate.

% \begin{table}[htbp]
%     \centering
%     \begin{tabular}{c|c|c|c|c|c}
%     \multirow{3}{*}{Parameter} & \multicolumn{5}{c}{Dataset} \\
%     \cmidrule(lr){2-6}
%     & \multicolumn{2}{c}{1d} & \multicolumn{2}{c}{1d + time} & 2d  \\ % & 2d + time \\ 
%     \cmidrule(lr){2-3} \cmidrule(lr){4-5} \cmidrule(lr){6-6}
%      & Helmholtz & Poisson & Advections & Reaction-diffusion & Darcy \\ 
%      \toprule
%      knots (\cref{app:impdetails}) & shifted & shifted & shifted & shifted & shifted\\
%      N & 32 & 32 & 40, 40 & 40, 20 & 40, 40\\ 
%      \end{tabular}
%      \caption{Basis details for each dataset. }
%     \label{tab:basissize}
% \end{table}
\clearpage

\newpage
\section{Additional Experiments}
\label{app:ablation}

\subsection{Ablation}
We experimentally show some properties of our iterative method.
These evaluations are made on a test set composed of several instances of PDE with varying configurations $(\gamma, f, g)$ that \textbf{are unseen during training}. These experiments are performed on the \textit{Helmholtz} equation which appeared as one of the most complex to optimize in our evaluation. We used $L=5$ in our method for a better visualization, unless stated otherwise.

\paragraph{Error w.r.t the number of steps L}
On \cref{fig:abL}, we show that having more optimizer steps allows for a better generalization. However, we observed in our experiments that the generalization error stabilizes or even increases after the proposed $5$ steps. This limitation of the solver should be investigated in future work in order to allow the model to make more iterations. Note that for this experiment, we lowered batch size (and adapted the learning rate accordingly) when the number of steps increased (\cref{fig:abL}). 

\begin{figure}[htbp]
    \centering
    \includegraphics[width=0.5\textwidth]{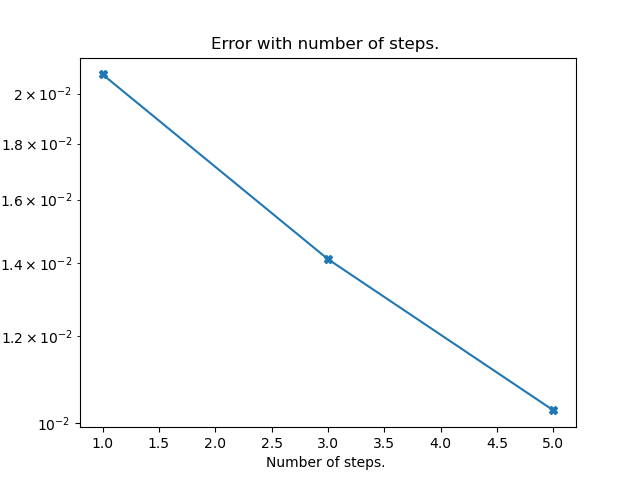}
    \caption{Error on the test set (\textit{Helmholtz} equation) \textit{w.r.t} the number of iterations }
    \label{fig:abL}
\end{figure}
% \vspace{-0.5cm}

\paragraph{Error w.r.t the number of training samples}
On \cref{fig:abntrain_h,fig:abntrain_d}, we show that compared to other physics-informed baselines, our solver requires less data to learn to solve PDE. Note that in this simple example, the \textit{MLP+basis} baseline also performed well. However, as shown in \cref{tab:test-loss}, this is not the case for all the datasets.  
The contribution of the iterative procedure clearly appears since with $2\times$ less data, our model performs better than this baseline. 

\begin{figure}[htbp]
    % \vspace{-1.6cm}
    \begin{subfigure}{0.48\textwidth} % specify the width of subfigure
        \centering
        \includegraphics[width=\textwidth]{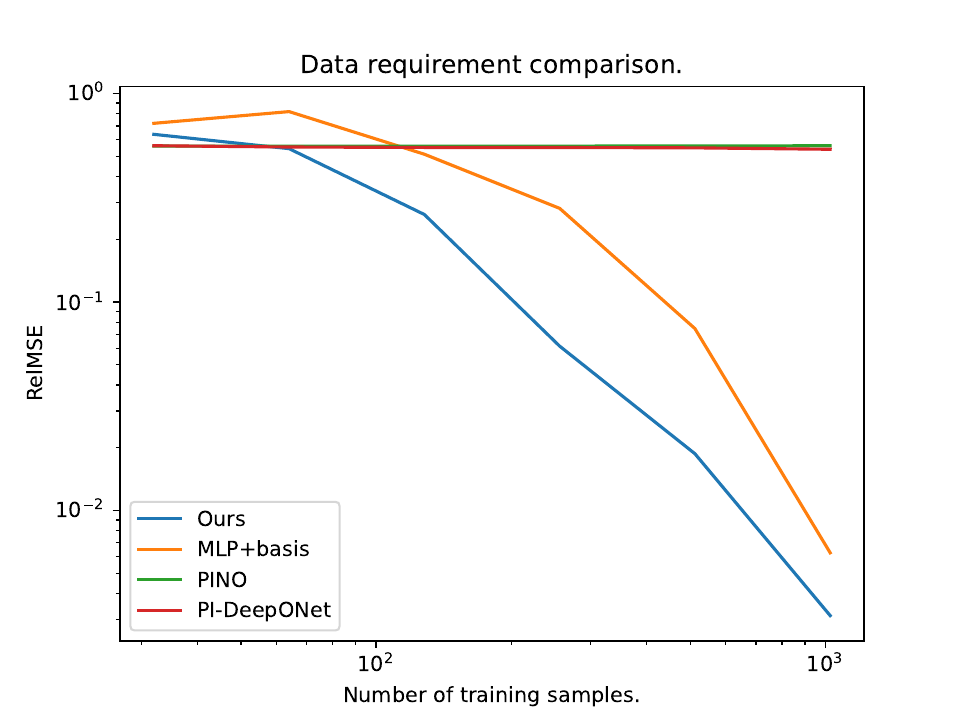}
        \caption{}
        \label{fig:abntrain_h}
  \end{subfigure}
  \hfill
  \begin{subfigure}{0.48\textwidth} % specify the width of subfigure
        \centering
        \includegraphics[width=\textwidth]{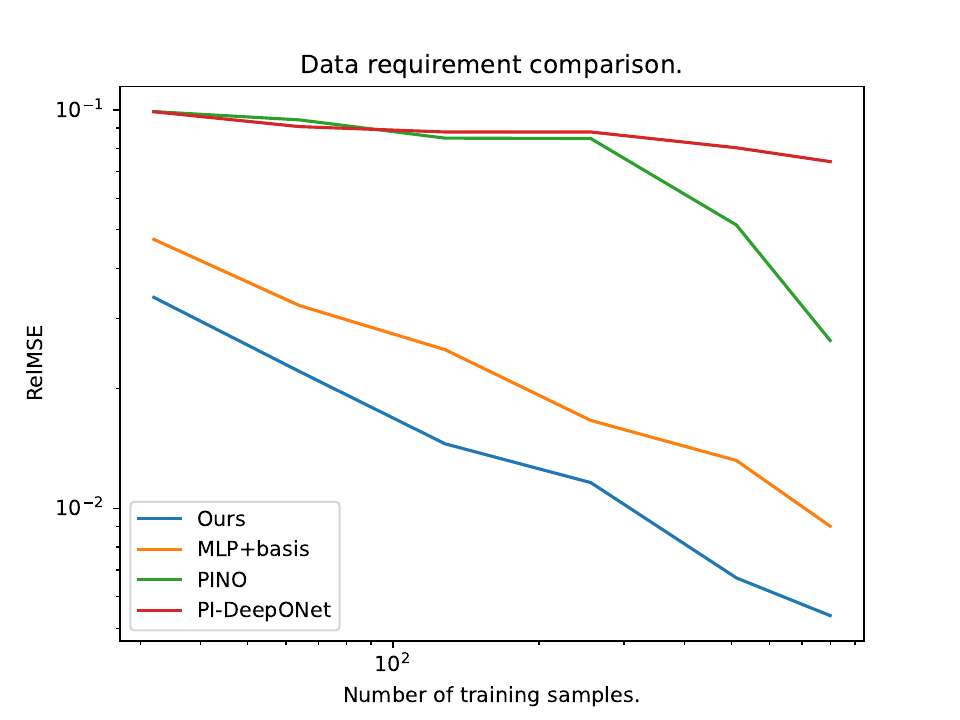}
        \caption{}
        \label{fig:abntrain_d}
  \end{subfigure}
  \caption{Relative MSE on the test set \textit{w.r.t} the number of training samples on \textit{Helmholtz} (left) and \textit{Darcy} (right) datasets.  }
\end{figure}

\paragraph{Optimization with $\mathcal{L}_{\textnormal{PDE}}$} \Cref{tab:abloss} shows that optimizing our network with physical loss greatly complicates training. Indeed, adding an ill-conditioned loss to the standard MSE makes training ill-conditioned. 
\begin{table}[htbp]
    \centering
    \begin{tabular}{lc}
    % & & \multicolumn{4}{c}{Dataset} \\
    Input & Relative MSE \\
    \midrule
    % $\mathcal{L}_{\textnormal{PDE}}$ \\
    $\mathcal{L}_{\textnormal{PDE}}$ + $\mathcal{L}_{\textnormal{DATA}}$ & 1.11\\
    $\mathcal{L}_{\textnormal{DATA}}$  & \textbf{2.19e-2}\\
    \bottomrule
  \end{tabular}
  \vspace{5mm}
  \caption{How the physical loss complexifies training and lowers performances. Metrics in Relative MSE on the test set of \textit{Helmholtz} equation. }
  \label{tab:abloss}
\end{table}
%\begin{minipage}{0.48\textwidth}
% \begin{wraptable}{r}{0pt}
    
% \end{wraptable}
%\end{minipage}
  % \begin{table}[htbp

\paragraph{Iterative update \& SGD-based update}
In \cref{tab:abitgd}, we compare two optimizer configurations. "Direct" means that the network directly predicts the parameters for the next step (\textit{i.e.} $\Theta_{k+1} = \mathcal{F}_{\varrho}(\nabla\mathcal{L}_{\textnormal{PDE}}(\Theta_k), \gamma, f, g)$), while "GD", corresponds to the update rule described before and using SGD as the base algorithm $\Theta_{k+1} = \Theta_k - \eta \mathcal{F}_{\varrho}(\nabla\mathcal{L}_{\textnormal{PDE}}(\Theta_k), \; \gamma, f, g)$. The latter clearly outperforms the direct approach and shows that learning increments is more efficient than learning a direct mapping between two updates. As shown in \cref{tab:abitgd} increasing the number of steps improves the performance (shown here for 1 and 5 update steps). However, the performance does not improve anymore after a few steps (not shown here).
%We also compare the performance for different update step numbers. Increasing the number of update steps improves the performance. Only a few steps are required: the performance stagnate after  not improve After a few steps the par We see that using the SGD guide the network to predict a more efficient descent direction. Moreover, we show that more steps improve performances. However, performances remain stable after a few steps (typically $5$).  
% compare the number of steps of these both strategies and show that iterating helps for performances. \\
\begin{table}[htbp]
    \centering
    \begin{tabular}{lcc}
    % & & \multicolumn{4}{c}{Dataset} \\
    & \multicolumn{2}{c}{Relative MSE} \\
    \cmidrule(lr){2-3}
    N-steps & Direct & GD \\
    \midrule
    %& \cmidrule{2-2} & \cmidrule{3-3}
    $1$-step & 1.08e-1 & 8.5e-2\\
    $5$-step & 9.07e-2 & \textbf{2.19e-2}\\
    \bottomrule
  \end{tabular}
    \caption{Comparison of different optimizer configurations for solving the \textit{Helmholtz} equation. Metrics in Relative MSE on the test set.}
    \label{tab:abitgd}
\end{table}

\paragraph{Optimization with different inner learning rates}
In \cref{tab:abinnerlr}, we study the performance of our proposed method with different inner learning rates $\eta$. As expected, a higher learning rate leads to better performances since the optimization is taking bigger steps. 

\begin{table}[htbp]
    \centering
    \begin{tabular}{c|c}
        learning rate & Relative MSE \\
        \midrule
         0.01 & 7.32e-2\\ % 16
         0.1 & 4.93e-2\\ % 17
         1 & \textbf{2.19e-2} \\ %15
    \end{tabular}
    \caption{Ablation on the inner learning rate. Metrics in Relative MSE on the test set of \textit{Helmholtz} equation.}
    \label{tab:abinnerlr}
\end{table}

\paragraph{Quantifying the importance of input feature for the learned solver} As indicated in \cref{eq:NGinfstep} the inputs of our learned solver are $(\gamma, f, g, \nabla\mathcal{L}_{\textnormal{PDE}}(\theta_L))$. 
We performed experiments by removing either $\gamma$ or $\nabla\mathcal{L}_{\textnormal{PDE}}(\theta_l)$ from the input (For the \textit{Helmholtz} equation, there is no forcing term $f$). The BC $g$ are kept since they are part of the PDE specification and are required to ensure the uniqueness of the solution.
This experiment (\cref{tab:abinputs}) illustrates that conditioning on the PDE parameters 
$\gamma$ is indeed required to solve the parametric setting. Without $\gamma$, the solver has no hint on which instance should be used. The addition of the gradient information, $\nabla$, which is at the core of our method, is also crucial for improving the convergence and validates our setting.

\begin{table}[htbp]
    \centering
    \caption{Effect of using the gradient as input \textit{w.r.t} the PDE parameters. Metrics in Relative MSE on the test set of \textit{Helmholtz} equation.}
    \begin{tabular}{lc}
    % & & \multicolumn{4}{c}{Dataset} \\
    Input & Relative MSE \\
    \midrule
    % $\nabla$ & 1.54\\
    $\gamma$ + g & 3.75e-1\\
    $\nabla$ + g & 1.07e-1\\
    $\nabla$ + $\gamma$ + g & \textbf{2.19e-2}\\
    \bottomrule
  \end{tabular}
  \label{tab:abgrad}
    \label{tab:abinputs}
\end{table}

% \subsection{Ablations}\label{subsec:abaltion}

% \paragraph{Non-linear Reconstruction}

\vspace{2cm}
We showed by a simple experiment that our model can handle nonlinear cases (see \cref{tab:abnl}). We propose to model the solution $u$ using a \textit{non-linear combination of the basis terms $\phi_i$}. The relation between the $\phi_i$ is modeled using a simple NN with one hidden layer and a \verb|tanh| activation function. This experiment is performed on the \textit{Poisson} PDE.

\begin{table}[htbp]
    \centering
    \caption{Nonlinear combination of the basis. 
    Relative MSE on the test set for our proposed method and comparison with other non linear models and optimizers.}
    \begin{tabular}{c|c}
        baselines & Relative MSE \\
        \midrule
         PINNs+L-BFGS & 8.83e-1\\
         PPINNs & 4.30e-2 \\
         Ours & \textbf{3.44e-3}\\ 
    \end{tabular}
    \label{tab:abnl}
\end{table}

\textbf{Network architecture}
We show in \cref{tab:ablayertype}, an ablation on the layer type used for $\mathcal{F}_{\varrho}$ in our experiments: MLP, Residual Network (ResNet), FNO and a modified version of a MLP (ModMLP) taken from \citep{wang2021understanding}, inspired from attention mechanism. 
We conducted this experiment on the Helmholtz dataset.

\begin{table}[htbp]
    \centering
    \caption{Ablation on different layer types. Metric on the test set of the \textit{Helmholtz} equation.}
    \begin{tabular}{c|c}
        Layer type & Relative MSE \\
        \midrule
         MLP & 8.25e-1\\ 
         % ModMLP & \underline{5.55e-2} retourver ce putain de paper\\
         ResNet & 6.87e-1\\
         ModMLP & 5.55e-2\\
         FNO &  \textbf{2.41e-2}\\
    \end{tabular}
    \label{tab:ablayertype}
\end{table}

%$z \odot U + (1-z) \odot V$

\textbf{Irregular grids}
We show in \cref{tab:abirgrids}, an ablation on different types of grids: regular as in \cref{tab:test-loss}, and irregular. %Both experiments have 64 points in the grid for training but are evaluated on the full $256$ available points. 
The latter were created by sampling uniformly 75\% of the points in the original grid.
The sampled grids \textbf{are different between each trajectory both during the training and testing phases}. We conducted this experiment on the Helmholtz dataset. Finally, we show some reconstruction examples in \cref{fig:reggridsol,fig:irrgridsol}. 

\begin{table}[H]
    \centering
    \caption{Comparison of the performances when training our solver using regular or irregular and different grids for each PDE. Metrics on the test set. }
    \begin{tabular}{c|c}
        Grid & Relative MSE \\
        \midrule
         regular & 2.41e-2\\
         irregular & 3.38e-1 \\
    \end{tabular}
    \label{tab:abirgrids}
\end{table}

% \begin{figure}[htbp]
% \begin{subfigure}{0.48\textwidth}
%     \centering
%     \includegraphics[width=\textwidth]{img/.pdf}
%     \caption{}
%     \label{fig:irrgrid1}
% \end{subfigure}
% \hfill
% \begin{subfigure}{0.48\textwidth}
%     \centering
%     \includegraphics[width=\textwidth]{img/.pdf}
%     \caption{}
%     \label{fig:irrgrid2}
% \end{subfigure}
% \caption{Example of irregular grids used to train our model. }
% \end{figure}

\begin{figure}[htbp]
\begin{subfigure}{0.48\textwidth}
    \centering
    \includegraphics[width=\textwidth]{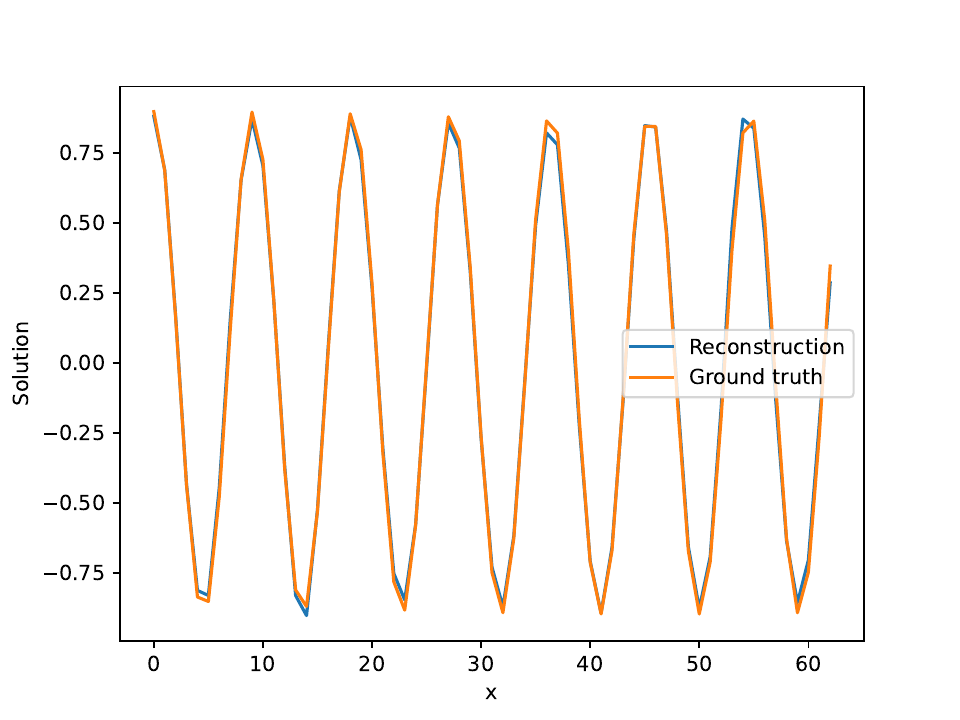}
    \caption{}
    \label{fig:reggridsol}
\end{subfigure}
\hfill
\begin{subfigure}{0.48\textwidth}
    \centering
    \includegraphics[width=\textwidth]{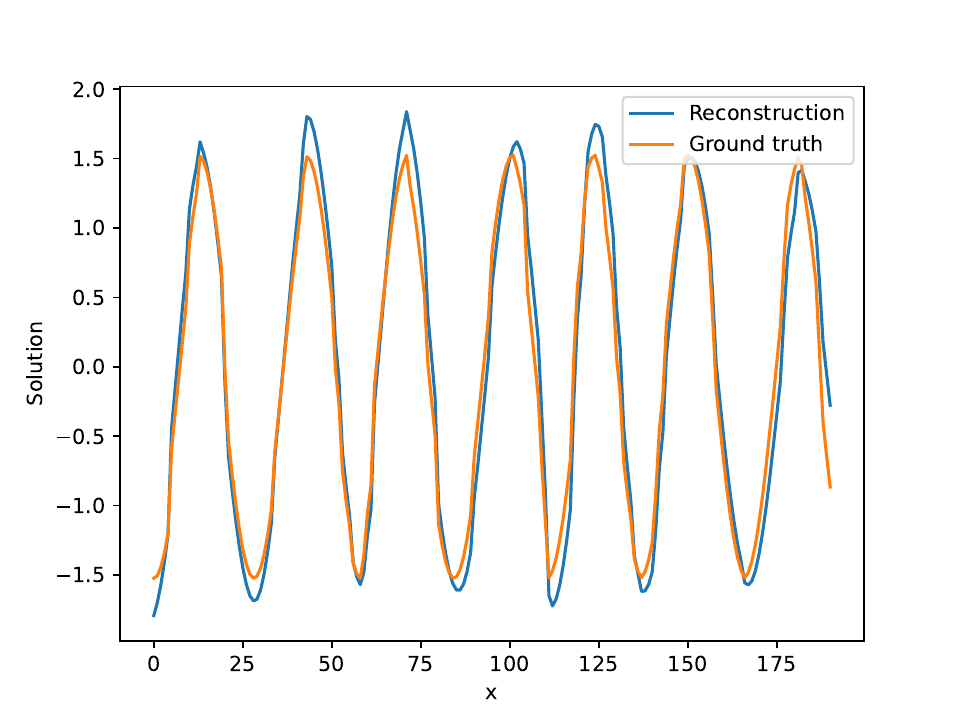}
    \caption{}
    \label{fig:irrgridsol}
\end{subfigure}
\caption{Comparison of the reconstruction of the solution when models are trained on regular (left) and irregular grid (right). }
\end{figure}

This experiment shows that our method is capable of handling irregular grids. We observe a decrease in the performances (\cref{tab:abirgrids}) and in the reconstruction quality \cref{fig:reggridsol,fig:irrgridsol}. However, we also note that our method is still capable of reconstructing the dynamic of the PDE, where most of the considered baselines failed to capture the oscillation in the Helmholtz PDE solution. 

\vspace{2cm}

\textbf{Error as a function of the PDE parameters values}
We illustrate in \cref{fig:mseomega} the behavior of the reconstruction of the MSE varying the PDE parameter values. We conducted this experiment on the Helmholtz PDE and varied $\omega$ from $-5$ to $55$ \textit{i.e.} extrapolation of $10\%$ beyond the parameter distribution (and kept fixed boundary conditions). 

\begin{figure}[htbp]
    \centering
    \includegraphics[width=0.5\linewidth]{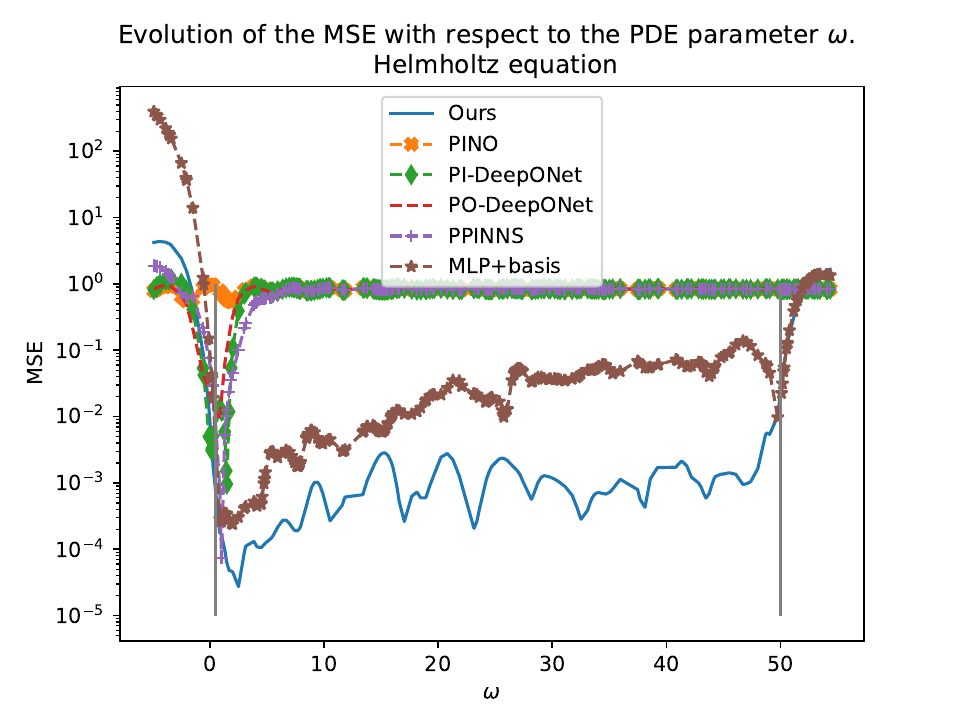}
    \caption{MSE comparison with PDE parameters $\omega$ of the Helmholtz dataset and extrapolation outside of the training distribution. }
    \label{fig:mseomega}
\end{figure}

We observe in \cref{fig:mseomega} that, as expected, solving Helmholtz PDE outside the training distribution fails for our model as well as for all baselines. The proposed method performs well inside the training distribution (as already observed in \cref{tab:test-loss}) and behaves similarly to fully supervised methods on Out-Of-Distribution (OOD) examples. Finally, other baselines, involving physical only or hybrid training, only predict a mean solution.

\subsection{Statistical analysis}

We provide a statistical analysis on some datasets and baselines. These experiments are conducted using $L=5$ (instead of $L=2$ in \cref{tab:test-loss}). For computational cost reasons, we did not make this experiment on all datasets but kept $2$ datasets in $1d$ (\cref{tab:test-loss-stat1d}), 1 dataset in $1$d+time, and 1 dataset in $2$d (\cref{tab:test-loss-stat-2d}) so that several configurations and data sizes are represented. 

\begin{table}[htbp]
  \centering
  \begin{tabular}{lccc}
    % & & \multicolumn{4}{c}{Dataset} \\
    & & \multicolumn{2}{c}{1d}\\ % & \multicolumn{2}{c}{1d + time}\\
    \cmidrule(lr){3-4} \\
     & Baseline & Helmholtz & Poisson \\
    \toprule
    %& \textit{Maths Solver} & \\
     Supervised & \textit{MLP + basis} &  \underline{5.26e-2 $\pm$ 7.56e-3} & 1.58e-1 $\pm$ 7.98e-3 \\
    % & Deeponet & 4.05e-1 & 1.14e-1 & 3.83e-1 & 6.40e-6 & 6.66e-2\\
    % & FNO & 4.7e-1 & 1.32e-4 & 6.07e-4 & 8.48e-6 & \textbf{6.87e-3}\\
    \midrule
    \multirow{2}{*}{Unsupervised} & \textit{PPINNs} &  
    8.33e-1$\pm$5.61e-3 & 3.59e-2$\pm$2.11e-2 \\
    & \textit{PO-DeepONet} & 9.84e-1$\pm$6.93e-4 & 1.79e-1$\pm$3.09e-2 \\
%      %& \textit{MAD-PINNs} & \\
%      %& Hyper-PINNs & \\
     \midrule
    & \textit{PI-DeepONet} & 9.81e-1 $\pm$ 2.25e-3 & 1.25e-1 $\pm$1.04e-2 \\
    Hybrid & \textit{PINO} & 9.95e-1 $\pm$ 3.30e-3 & \underline{3.27e-3 $\pm$ 1.38e-3} \\
%     \cmidrule(lr){2-7}
% %    & \textit{Hybrid DL solver} & \\
    & \textit{Ours} & \textbf{2.17e-2 $\pm$ 1.12e-3} & \textbf{4.07e-5 $\pm$ 2.65e-5} \\
    %\bottomrule
  \end{tabular}
  \caption{Results of trained models with error bars (std errors) on 1d datasets - metrics in Relative MSE on test set. Best performances are highlighted in \textbf{bold}, and second best are \underline{underlined}. }
  \label{tab:test-loss-stat1d}
\end{table}

\begin{table}[htbp]
  \centering
  \begin{tabular}{lccc}
    % & & \multicolumn{4}{c}{Dataset} \\
    & & 1d + time & 2d \\ 
    \cmidrule(lr){3-3} \cmidrule(lr){4-4} 
     & Baseline & 1dnlrd & Darcy-Flow\\
    \toprule
    %& \textit{Maths Solver} & \\
     Supervised & \textit{MLP + basis} & \underline{2.83e-5 $\pm$ 6.83e-7}  & \underline{3.78e-2$\pm$2.09e-3} \\
    % & Deeponet & 4.05e-1 & 1.14e-1 & 3.83e-1 & 6.40e-6 & 6.66e-2\\
    % & FNO & 4.7e-1 & 1.32e-4 & 6.07e-4 & 8.48e-6 & \textbf{6.87e-3}\\
    \midrule
    \multirow{2}{*}{Unsupervised} & \textit{PPINNs} & 4.64e-1 $\pm$ 1.92e-2 &9.99e-1$\pm$2.63e-2 \\
    & \textit{PO-DeepONet} & 4.18e-1 $\pm$ 1.04e-2 & 8.32e-1$\pm$2.51e-4 \\
     %& \textit{MAD-PINNs} & \\
     %& Hyper-PINNs & \\
     \midrule
    & \textit{PI-DeepONet} & 7.88e-2 $\pm$ 1.96e-4 & 2.72e-1 $\pm$ 4.44e-3 \\
    Hybrid & \textit{PINO} & 8.00e-5 $\pm$ 1.00e-5 & 1.17e-1 $\pm$ 1.42e-2\\
    % \cmidrule(lr){2-3}
%    & \textit{Hybrid DL solver} & \\
    & \textit{Ours} & \textbf{2.61e-5 $\pm$ 2.53e-6} & \textbf{1.62e-2 $\pm$ 3.06e-4}  \\
    \bottomrule
  \end{tabular}
  \caption{Results of trained models with error bars (std errors) on 1d + time, and 2d datasets - metrics in Relative MSE on test set. Best performances are highlighted in \textbf{bold}, and second best are \underline{underlined}. }
  \label{tab:test-loss-stat-2d}
\end{table}

This analysis shows the robustness of our proposed method \textit{w.r.t.} initial seed. 
\vspace{5cm}

\subsection{Computational cost}
\label{app:ctime}

Finally, we detail here the training and inference times of our method as well as baselines. 

As we can see in \cref{tab:traintime}, our model takes longer to train due to the iterative process occurring at each epoch. 
However, note that this is training time; inference time is similar to other methods (see \cref{tab:testtime}). We detail a justification for this additional training time compared to PINNs variants below:
%and because we are handling a parametric problem, whereas vanilla PINNs require specific training for each instance (parameter values) of an equation. 
\begin{itemize}
    \item \underline{Comparison with vanilla PINNs}: 
Consider the following experiment. Suppose we train a classical PINN on a single instance of the Darcy PDE. Based on the training times shown in \cref{tab:traintime}, using this method, we performed $1,500$ steps, which took $420$ minutes for training (please note that the performances were less accurate than our model’s performance). If we wanted to train a PINNs on each PDE of our entire test dataset for $15 000$ epochs (sometimes even more steps are required), this would take $4200$ minutes or stated otherwise approximatly $3$ days. 
%For our entire test dataset composed of $200$ PDE instances, this would require training a PINN for each PDE, leading to a total computational time of 5.5 days (40 minutes × 200 instances). 
Suppose now, that one wants to solve an additional equation. This will require an average of $0.226$ seconds (see \cref{tab:testtime}) with our method, while PINNs would require an entirely new training session of approximately $20$ minutes for $15,000$ steps. This makes our method $5,000$ times faster than traditional PINNs for solving any new equation.

\item \underline{Comparison with PINNs parametric variants}: Now let us consider two parametric variants of PINNs designed to handle multiple PDEs (PPINNs for parametric PINNs, P2INNs for the model proposed by \citep{cho2024parameterizedphysicsinformedneuralnetworks}). In \cref{tab:test-loss}, models was trained for $5,000$ epochs only. Let us consider training it further as suggested for vanilla PINNs, until $15,000$ epochs. First, this training would require approximately $19$h$30$m. Then, the optimization problem is still ill-conditioned, training further would probably not significantly improve the performance. We can extend this reasoning to the \textit{P2INNs} baseline, for which we observed similar performance and behaviors. 
%For vanilla PINNS we extend optimization to $10 000$ epochs without clear increase of performance compared to the $5 000$ steps already reported: $8.47 e-1$ RMSE at $5000$ epochs compared to $8.35 e-1$ at $10 000$ epochs. Since the 
%As a second alternative, let us take the PPINNs model trained with $5000$ epochs as in table 2 and let us fine tune it for $1000$ steps instead of $10$ in the experiments reported in the paper. The performance is $8.37e-1$ instead of the $8.58 e-1$ RMSE at $5000$ epochs. This indicates that further fine tuning is not a solution for parametric PINNs versions.
\end{itemize}

\begin{table}[htbp]
    \centering
    \begin{tabular}{c|c|c|c|c|c}
         Dataset & Helmholtz & Poisson & 1dnlrd & Darcy & Heat\\
         \midrule
         MLP + basis & 30m & 20m &  1h10m & 2h & 4h45\\
         PPINNs & 15m & 20m & 4h15m & 6h30m & 1d2h\\
         P2INNs & 2h & 3h & 11h & 1d7h & 1d8h \\
         PODON & 10m & 10m & 3h30m & 1d9h & 22h \\
         PIDON & 10m & 10m & 3h30m & 1d10h & 22h\\
         PINO & 15m & 10m & 1h10m & 45m & 2h40\\
         Ours & 30m & 30m & 4h30 & 10h15 & 1d 13h\\
    \end{tabular}
    \caption{Training time of the experiments shown in Table 2. of the paper on a single NVIDIA TITAN RTX (25 Go) GPU. d stands for days, h for hours, m for minutes.}
    \label{tab:traintime}
\end{table}

\begin{table}[htbp]
    \centering
    \begin{tabular}{c|c|c|c|c|c}
         Dataset & Helmholtz & Poisson & 1dnlrd & Darcy & Heat \\
         \midrule
         MLP + basis & 1.12e-2 & 1.18e-2 & 1.25e-2 & 1.19e-2 & 1.66e-2\\
         PINNs+L-BFGS & 274 & 136 & 369 & 126 & 234 \\
         PINNs-multi-opt & 15.5 & 25.5 & 16.5 & 105 & 90\\
         PPINNs & 3.09e-1 & 2.03e-1 & 2.91e-1 & 3.22e-1 & 5.45e-1\\
         P2INNs & 2.84e-1 & 3.09e-1 & 6.76e-1 & 1.29 & 1.23\\
         PODON & 3.27e-1 & 2.71e-1 & 4.38e-1 & 6.32e-1 & 8.85e-1 \\
         PIDON & 3.32e-1 & 2.96e-1 & 4.43e-1 & 6.35e-1 & 8.80e-1\\
         PINO & 3.14e-1 & 1.24e-1 & 5.19e-1 & 2.21e-1 & 8.08e-1\\
         Ours & 2.58e-1 & 2.16e-1 & 2.84e-1 & 2.26e-1 & 2.90e-1\\
    \end{tabular}
    \caption{Inference time, averaged (in seconds) on the test set. All experiments are conducted on a single NVIDIA RTX A6000 (48Go). We report the mean time computed on the test set to evaluate the baselines as performed in \cref{tab:test-loss} in the paper (\textit{i.e.} with $10$ test-time optimization steps when applicable and $20$ steps on the Heat dataset). We consider as inference the solving of a PDE given its parameters and/or initial/boundary conditions. }
    \label{tab:testtime}
\end{table}

\vspace{3cm}

\subsection{Additional datasets}

We provide additional experiments on $2$ new datasets: Non-linear Reaction-Diffusion in a more complex setting and Advections. We refer the reader to \cref{app:appdataset} for the details about the PDE setting. These datasets were not included in the main part of the paper due to a lack of space in the results table. 

% \textbf{Advection}: We use PDEBench's data for the $1d$ Advection PDE \citep{takamoto2023pdebench}. The PDEBench's Advection dataset is composed of several configurations of the advection parameter $\beta$ ($\{0.1, 0.2, 0.4, 0.7, 1, 2, 4, 7\}$). Each configuration is composed of $10,000$ trajectories with varying initial conditions. From these datasets, we sampled a total of $1,000$ trajectories for $\beta \in \{0.2, 0.4, 0.7, 1, 2, 4\}$ (this is about $130$ trajectories for each $\beta$). This gives a dataset with different initial conditions $g$ by varying $A_i$ and $k_i$, (see \cref{tab:parameters_datasets_app} for the range used and \cref{sec:appdataset} for more details) and parameters $\gamma$. % Moreover, we sub-sampled the trajectories by $4$, leading to a grid of resolution $25$ for the t-coordinate and $256$ for the x-coordinate.

\begin{table}[htbp]
  \centering
  \begin{tabular}{lcccc}
    % & & \multicolumn{4}{c}{Dataset} \\
     & Baseline & NLRDIC & Advections \\ %& Helmholtz 2d\\
    \toprule
    %& \textit{Maths Solver} & \\
     Supervised & \textit{MLP + basis} & \textbf{9.88e-4} & 6.90e-2  \\
    % & Deeponet & 4.05e-1 & 1.14e-1 & 3.83e-1 & 6.40e-6 & 6.66e-2\\
    % & FNO & 4.7e-1 & 1.32e-4 & 6.07e-4 & 8.48e-6 & \textbf{6.87e-3}\\
    \midrule
    \multirow{2}{*}{Unsupervised} & \textit{PPINNs} & 3.71e-1 & 4.50e-1 \\
    & \textit{PO-DeepONet} & 4.36e-1 & 5.65e-1 \\
     %& \textit{MAD-PINNs} & \\
     %& Hyper-PINNs & \\
     \midrule
    & \textit{PI-DeepONet} & 5.39e-2 & 4.26e-1 \\
    Hybrid & \textit{PINO} & 3.79e-3 & \textbf{6.51e-4}\\
    % \cmidrule(lr){2-3}
%    & \textit{Hybrid DL solver} & \\
    & \textit{Ours} & \underline{1.41e-3} & \underline{5.39e-3} \\
    \bottomrule
  \end{tabular}
  \caption{Results of trained models on additional datasets - metrics in Relative MSE on the test set. Best performances are highlighted in \textbf{bold}, and second best are \underline{underlined}.}
  \label{tab:test-loss-addds}
\end{table}

These additional datasets show cases where baselines are performing very well on the considered PDE. For \textit{Advections}, PINO reached very good performances. We believe that the Fourier layers used in the model fit well to the phenomenon. Indeed, the solution is represented using a combination of sine moving with time. This simple dynamics is easily captured using Fourier transformations. Our B-Splines basis can be sub-optimal for this dataset. For the complexified version of \textit{NLRD}, \textit{NLRDIC}, it is the supervised baseline that performs best. This dataset does not present high frequencies, and the MLP looks sufficiently expressive to find the coefficient in the basis. However, this baseline had difficulties in reconstructing the higher frequencies in \textit{Advections}. 
Even if our model does not perform best on these datasets, it is ranked second. We believe that these additional results show the robustness of our method across different physical phenomena. 

\subsection{Training behavior}

We show in this section the evolution of the MSE as the training progresses (see \cref{fig:MSEtraining}). 
\begin{figure}[htbp]
    \centering
    \includegraphics[width=0.8\linewidth]{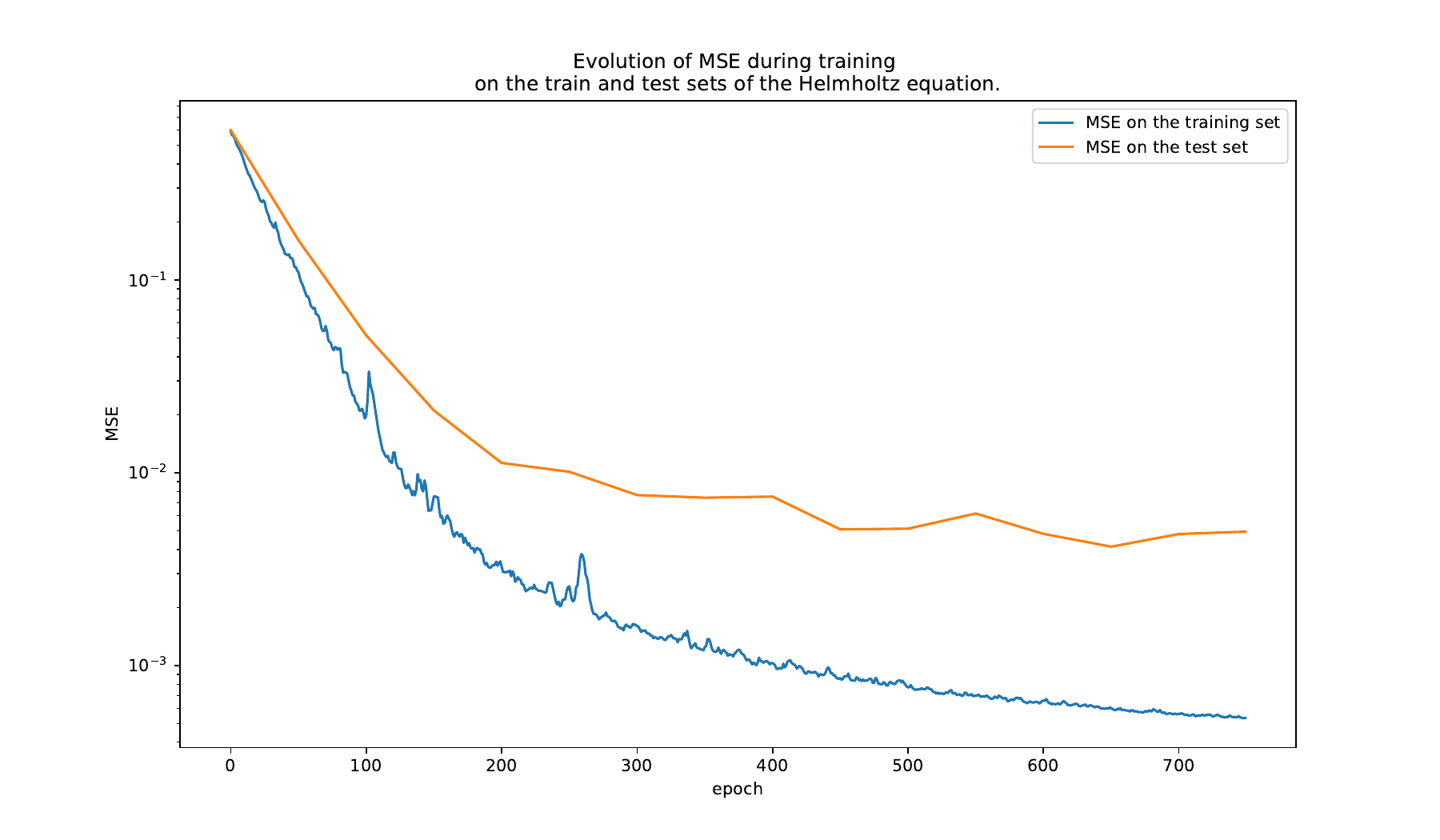}
    \caption{MSE during training on the training and testing sets. Example shown on the Helmholtz equation for results as presented in \cref{tab:test-loss}.}
    \label{fig:MSEtraining}
\end{figure}

\begin{figure}[htbp]
\begin{subfigure}{0.48\textwidth}
    \centering
    \includegraphics[width=\textwidth]{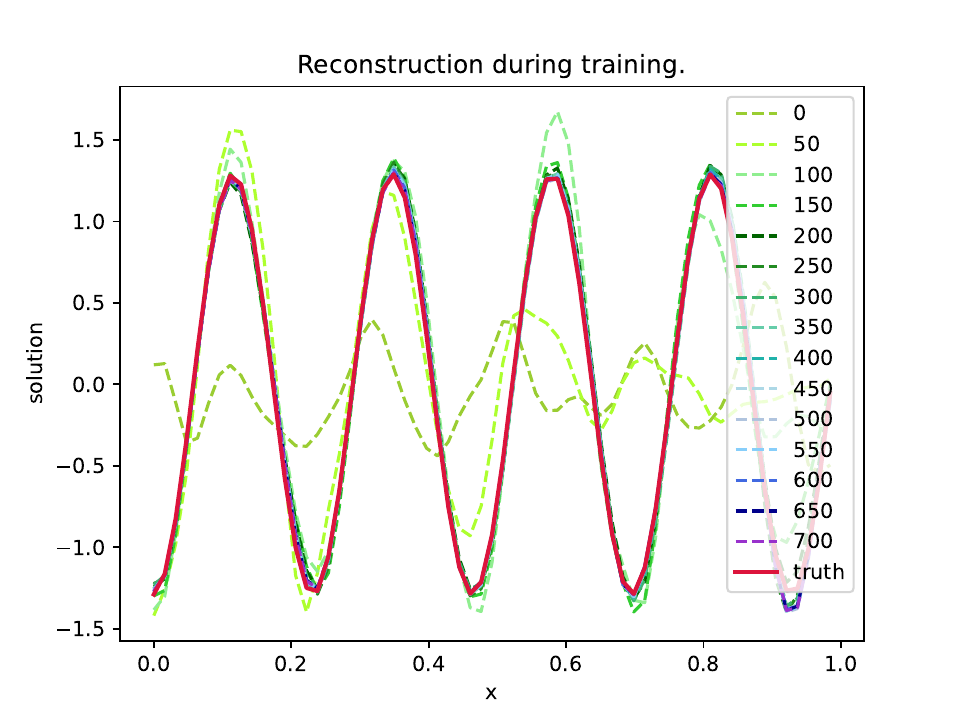}
    \caption{}
    \label{fig:solwrtepochs1}
\end{subfigure}
\hfill
\begin{subfigure}{0.48\textwidth}
    \centering
    \includegraphics[width=\textwidth]{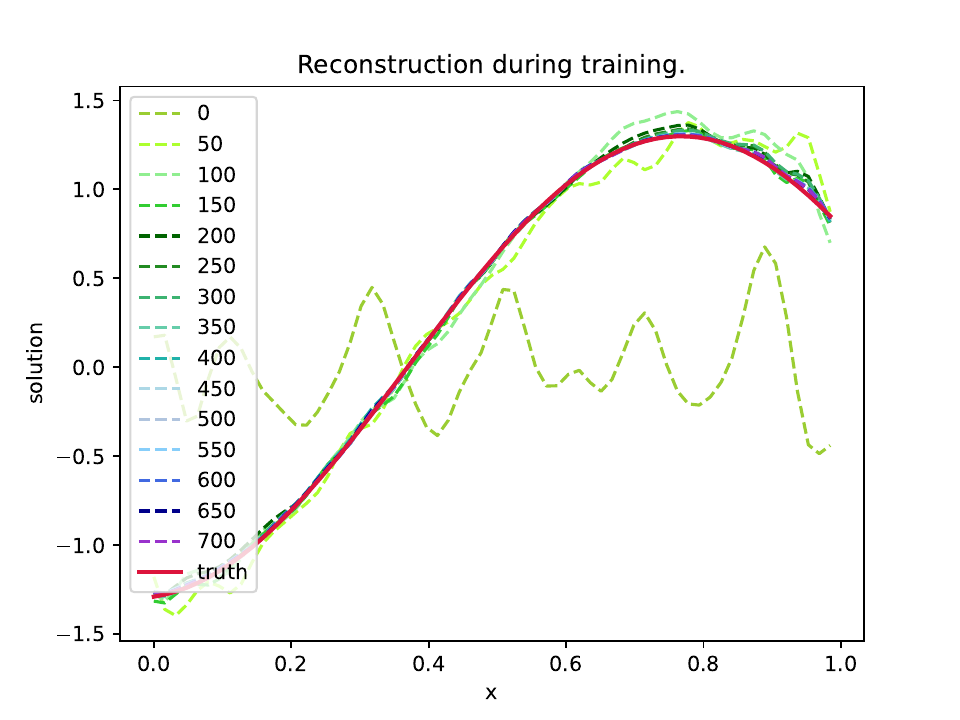}
    \caption{}
    \label{fig:solwrtepochs2}
\end{subfigure}
\caption{Reconstruction of the solution PDEs during training. Example drawn from the test set. }
\end{figure}

In \cref{fig:MSEtraining}, we show the MSE evolution on both the train and test sets (evaluation every $50$ epochs) and on \cref{fig:solwrtepochs1,fig:solwrtepochs2} some reconstruction (from the test set), with respect to the training epochs.
%We first recall that the \textbf{training set} correseponds to a set of PDE used to optimize the neural network (\textit{i.e.} $\mathcal{F}_(\varrho)$, and that the test set is composed of \textbf{new PDEs} (\textit{i.e.} \textbf{different} parameters and/or initial/boundary conditions). 
The training set ($800$ PDE trajectories) corresponds to PDEs used to update the network parameters, while the test set ($200$ trajectories) are unseen PDEs for the network. \textbf{This means that the model has not been trained or optimized on these PDEs.} The test set is composed of PDEs with varying PDE parameter values ($\omega$) and boundary conditions ($u_0, v_0$). This illustrates that the generalization performance on \textbf{new} PDEs \textbf{within} the training distribution is rapidly achieved by the network.

\clearpage
\subsection{Loss landscapes}
\label{app:losslandscape}

In this section, we propose a visual representation of the optimization paths in the loss landscape.  
\Cref{fig:losslandscapesPHY,fig:losslandscapesMSE} illustrates the behavior of the vanilla PINNs algorithm and of our "learning to optimize" method. The plots represent respectively a 2D visualization of the physical loss landscape $\mathcal{L}_{PDE}$ (\cref{fig:losslandscapesPHY}) and of the data loss landscape $\mathcal{L}_{\text{DATA}}$ (\cref{fig:losslandscapesMSE}), around the approximate solution (in our basis) of a given Helmholtz equation. The function basis used in this experiment has a size of $32$. We used the technique in \citep{NEURIPS2018_losslandscape} to visualize the solution coordinates in the basis $\Theta$ in the loss landscape (more details in \cref{sssec:projbasis}).
Superimposed on the loss background, we plot three trajectories obtained by starting from an initial random vector $\Theta_0$.

\Cref{fig:losslandscapesPHY} visualizes the sharp and ill conditioned landscape of the physical loss $\mathcal{L}_{PDE}$, while \Cref{fig:losslandscapesMSE} shows the better conditioned landscape of the MSE data loss $\mathcal{L}_{\text{DATA}}$. The figures provide intuition on how the proposed algorithm operates and improves the convergence. The optimizer $\mathcal{F}$ modifies the direction and magnitude of the stochastic gradient descent (SGD) updates in the physical loss landscape (\Cref{fig:losslandscapesPHY}, left plot). It achieves this by utilizing solution values at various sample points to adjust the gradient, steering it toward the corresponding minimum in the mean squared error (MSE) landscape (\Cref{fig:losslandscapesMSE}, central plot). While the descent remains within the residual physics error space, the learned optimizer provides an improved gradient direction, enhancing convergence efficiency (\Cref{fig:losslandscapesPHY}, right plot).
We describe below the different figures.

\begin{itemize}
    \item  Left Columns in \Cref{fig:losslandscapesPHY,fig:losslandscapesMSE}: these figures show the gradient path ($100$ steps) obtained by directly optimizing the physical loss $\mathcal{L}_{PDE}$, similar to the vanilla PINNs algorithm. This trajectory highlights the ill-conditioning of the optimization problem associated to $\mathcal{L}_{PDE}$.

    \item Center Column in \Cref{fig:losslandscapesPHY,fig:losslandscapesMSE}: these figures plot the trajectory of a gradient-based optimization algorithm trained with a mean squared error $\mathcal{L}_{\text{DATA}}$ data loss ($100$ steps), under the assumption that the solution values are known at collocation points. While these quantities are not available in our case, this visualization is included to illustrate the differences in convergence behavior between physics-based and MSE-based loss functions when both are accessible.
    \item Right Column in \Cref{fig:losslandscapesPHY,fig:losslandscapesMSE}: these plots illustrate the behavior of our algorithm (the solver is trained with 2 gradient steps for this example). It demonstrates the effect of our learned optimizer and the significant improvements achieved compared to a standard gradient descent algorithm on the physical loss. 
\end{itemize}. 
We describe in \cref{ssec:app_vismethod} the construction of the figures \cref{fig:losslandscapesPHY,fig:losslandscapesMSE}. 

\begin{figure}[htbp]
    \centering
    \includegraphics[width=\linewidth]{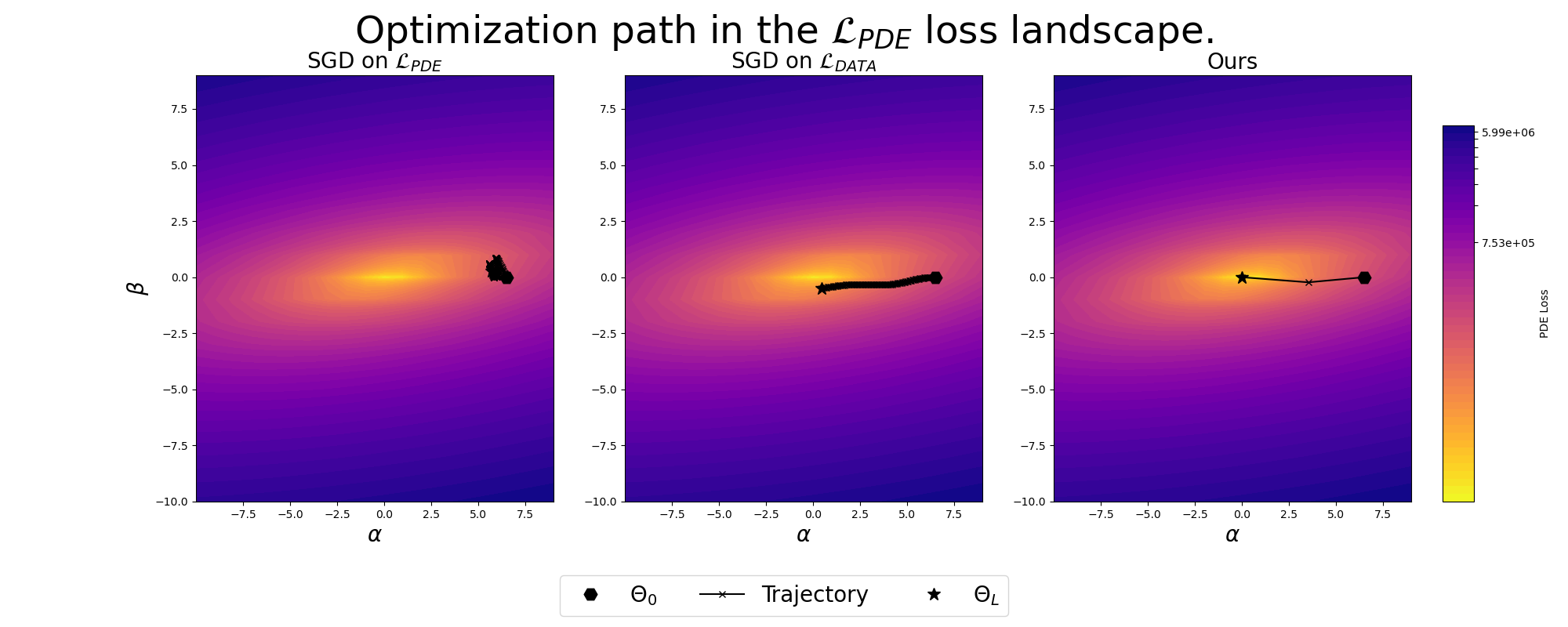}
    \caption{Loss landscapes and optimization trajectories on an instance of the Helmholtz PDE. The figure represents the PDE loss $\mathcal{L}_{\text{PDE}}$ landscape and superimposed are examples of optimization paths computed using SGD on the physical loss (left), the data loss (center), and our method (right).}
    \label{fig:losslandscapesPHY}
\end{figure}
\begin{figure}[htbp]
    \centering
    \includegraphics[width=\linewidth]{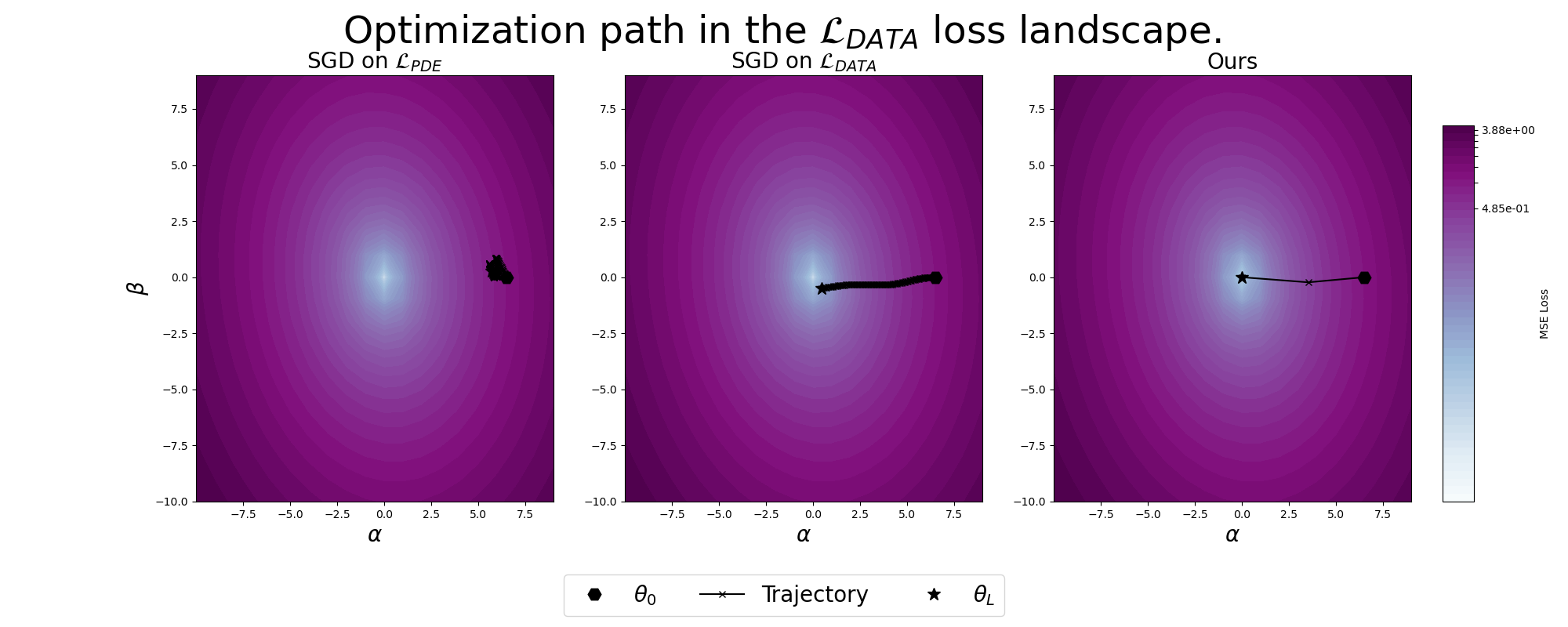}
    \caption{Loss landscapes and optimization trajectories on an instance of the Helmholtz PDE. The figures represents the DATA loss $\mathcal{L}_{\text{DATA}}$ and superimposed are examples of optimization paths computed using SGD on the physical loss (left), the data loss (center), and our method (right). }
    \label{fig:losslandscapesMSE}
\end{figure}

% XXXXXX et on vire ce qui suit qui est déjà + haut, ok ?
% yessss c'est fait ! 

% \Cref{fig:losslandscapesPHY,fig:losslandscapesMSE} clearly illustrates the ill-conditioning of the PDE loss for optimizing the $\Theta$ parameters, by exhibiting a very sharp and steep loss landscape. 
% \begin{itemize}
%     \item Left Column: The figure shows the gradient path ($100$ steps) obtained by directly optimizing the physical loss, similar to the PINNs algorithm. This highlights the ill-conditioning of the optimization problem.
%     \item Center Column: This figure depicts a gradient-based optimization algorithm trained with a mean squared error (MSE) data loss ($100$ steps), under the assumption that the solution values are known at collocation points. While these quantities are not available in our case, this visualization is included to illustrate the differences in convergence behavior between physics-based and MSE-based loss functions when both are accessible.  
%     \item Right Column: This plot demonstrates the behavior of our algorithm (2 gradient steps). They show the effect of our learned optimizer and the significant improvements achieved compared to a standard gradient descent algorithm on the physical loss. 
% \end{itemize}

\subsection{Projection of the basis for visualization}
\label{ssec:app_vismethod}

Following the technique described in \citep{NEURIPS2018_losslandscape}, we plot a 2D slice of the loss function around the minimum solution. We explain below the method used for creating \cref{fig:losslandscapesPHY,fig:losslandscapesMSE}.

\subsubsection{Create a basis for projection}
\label{sssec:projbasis}
We want to visualize the loss landscapes around a parameter solution $\Theta^{\star}$ of a given PDE. Due to the high dimensionality of the parameter space, we use the method proposed in \citep{NEURIPS2018_losslandscape}, that involves creating a 2d slice of the loss landscape for visualization. This methods projects the multidimensional landscape onto a 2D basis $(u,v)$ so that any solution vector $\Theta$ could be written as $\Theta=\alpha u + \beta v$, with, $\alpha, \beta \in \mathbb{R}$ and $u, v \in \mathbb{R}^{N}$, $2$ orthonormal vectors. Let $\Omega_{u, v}:= \text{Vect}\{u, v\}$ denote this projection space. For the visualization, the basis must include the solution vectors found by our algorithm. For that: 
\begin{enumerate}
    \item We start by running our optimization algorithm from a random vector parameters $\Theta_0$ to find $\Theta_L \in \mathbb{R}^N$, our approximate PDE solution. We set $w_1 = \Theta_L$ and 
    %To find this solution parameter, we chose to use our learned optimizer (but one could use any other optimization method).
     set a second vector $w_2=\Theta_0$, chosen here as the initial vector of our optimization process. 
    \item We generate a third vector  $w_3 \in \mathbb{R}^N$ randomly. 
    \item We then set: $u=w_2-w_1$ and $v=(w_3-w_1)-\frac{<w_3-w_1, w_2-w_1>}{||w_2-w_1||^2}(w_2-w_1)$
    \item The resulting normalized vectors $\hat{u}=\frac{u}{||u||}$ and $\hat{v}=\frac{v}{||v||}$ form an orthogonal basis containing $w_1, w_2$ and $w_3$. 
    \item Using the solution parameters $\Theta_L$ as the origin for our newly created basis, we can compute a $2$D slice of the loss landscape where each solution $\Theta$ could be represented in the $(u,v)$ basis with its projection coordinates   $(\alpha, \beta)$ as $\Theta_{projection}=w_1 + \alpha \hat{u} + \beta \hat{v}$. 
\end{enumerate}

\subsubsection{Plot trajectories}
Finally, we run a baseline optimization procedures and get  trajectories $\{\Theta_i\}_{i=0}^L$, with L being the number of optimization steps, e.g. a gradient descent, on the physical loss, the MSE loss or our proposed method. We project the $\Theta_i$'s in the  created basis (see \cref{sssec:projbasis}) and plot this projection of the optimization trajectory on the loss landscape $\Omega_{u, v}$ as visualized on \cref{fig:losslandscapesPHY,fig:losslandscapesMSE}.

\subsubsection{Illustration of the ill-conditioning of the PDE loss}

As an illustration of the ill-conditioning of the PDE loss, we replicate \cref{fig:losslandscapesPHY} using a basis that further emphasizes this aspect. To build this basis, we use the procedure described in \cref{sssec:projbasis} by using $2$ eigenvectors of the hessian of the PDE loss. 
First, we compute $\text{Hess}(\mathcal{L}_{\text{PDE}})$ and its eigenvectors decomposition. We select the vectors associated to the highest and lowest eigenvalues, and respectively set them to $w_2$ and $w_3$. The resulting landscape visualization is shown in \cref{fig:losslandscapesPHYeig}. 
\begin{figure}[htbp]
    \centering
    \includegraphics[width=\linewidth]{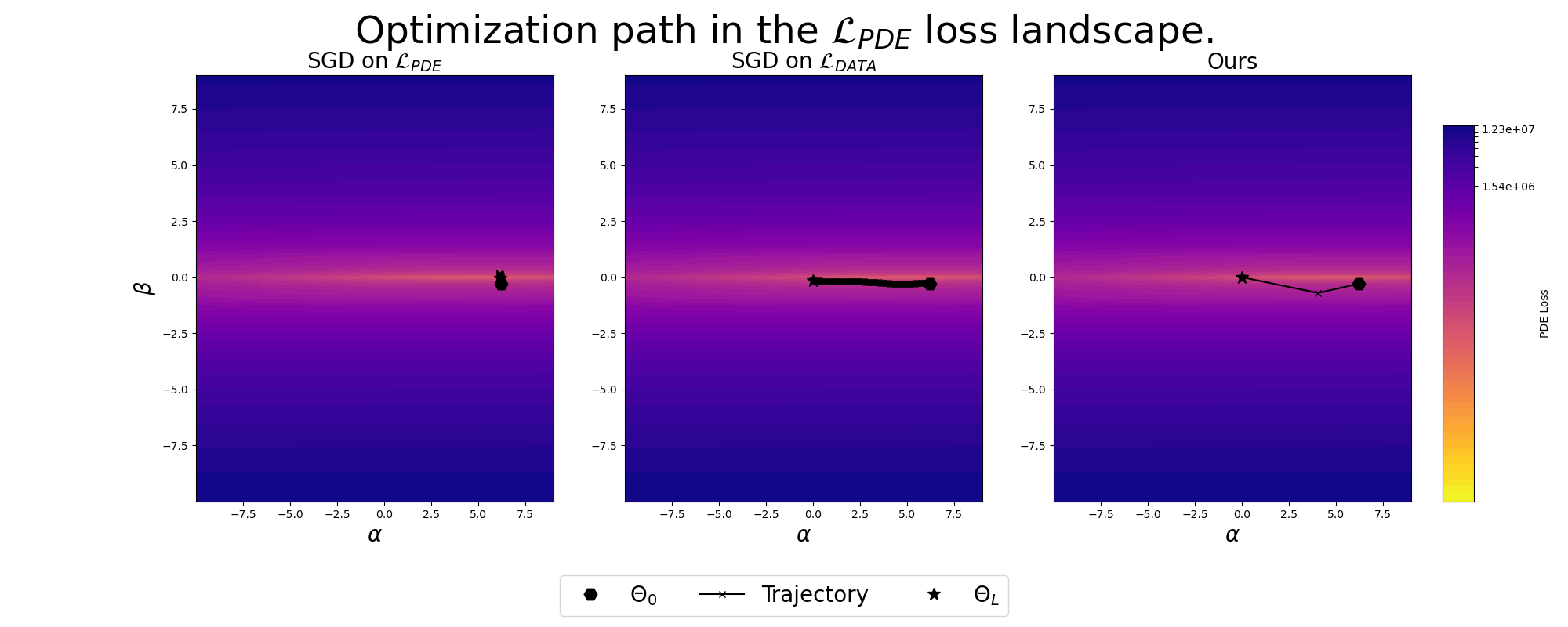}
    \caption{Loss landscapes and optimization trajectories on an instance of the Helmholtz PDE. These optimization trajectories are computed using SGD on the physical loss (left), the data loss (center), and our method (right). The background represents the PDE loss $\mathcal{L}_{\text{PDE}}$. }
    \label{fig:losslandscapesPHYeig}
\end{figure}

\Cref{fig:losslandscapesPHYeig}, clearly illustrates the characteristic shape of the ill-conditioned function $\mathcal{L}_{\text{PDE}}$. The two directions, extracted from the highest and lowest eigenvalues, are clearly visible on this PDE loss landscape. This highlights the difficulty of this optimization problem for standard descent methods.

\newpage
\section{Qualitative results}
\label{app:visu}
This section is dedicated to visualization of the results of our model, baselines and optimizers, presented in \cref{sec:expe}. 
For each PDE considered, we chose $2$ samples in the test sets and compute the solutions with our model and the different baselines. We provide visualization samples with $L=5$ \textit{i.e.} results proposed in \cref{tab:test-loss-stat1d,tab:test-loss-stat-2d} to detail more precisely the evolution of the solution at several steps of optimization. $3$ datasets are shown with $L=2$: \textit{Heat} for computational reasons (training with $L=5$ is much more expensive when the dimension of the problem increases) and the $2$ additional datasets trained only with $L=2$ (\textit{Advections} and \textit{NLRDIC}). Then, we show the evolution of the reconstruction of the solution with our method \textit{i.e.} we plot the solution at each step of the optimization (\cref{fig:ngdhelmholtz1,fig:ngdhelmholtz2,fig:ngdpoisson1,fig:ngdpoisson2,fig:ngdnlrd1,fig:ngdnlrd2,fig:ngddarcy1,fig:ngddarcy2,fig:ngdheat1,fig:ngdheat2,fig:ngdadvections1,fig:ngdadvections2,fig:ngdnlrdics1,fig:ngdnlrdics2}) and we compare the final prediction with baselines' (\cref{fig:baselinehelmholtz1,fig:baselinehelmholtz2,fig:baselinepoisson1,fig:baselinepoisson2,fig:baselinenlrd1,fig:baselinenlrd2,fig:baselinedarcy1,fig:baselinedarcy2,fig:baselineheat1,fig:baselineheat2,fig:baselineadvections1,fig:baselineadvections2,fig:baselinenlrdics1,fig:baselinenlrdics2}). Finally, we chose $20$
($6$ for \textit{Heat}) PDEs and we reproduce \cref{fig:test-time-opt} for every dataset. More precisely, we optimize one PINN per PDE using Adam, we fit our basis using several optimizers (GD, ADAM, L-BFGS and our learned optimization process) and we fine-tune the learned PINO for $10,000$ steps and visualize the evolution of the MSE (averaged at each step on the selected PDE). These figures show the relevance of learning the optimizer when using physics-informed losses (\cref{fig:testopthelmholtz,fig:testoptpoisson,fig:testoptnlrd,fig:testoptdarcy,fig:testoptheat,fig:testoptadvections,fig:testoptnlrdics}).

\subsection{Helmholtz}
\begin{figure}[htbp]
\begin{subfigure}{0.48\textwidth}
    \centering
    \includegraphics[width=\textwidth]{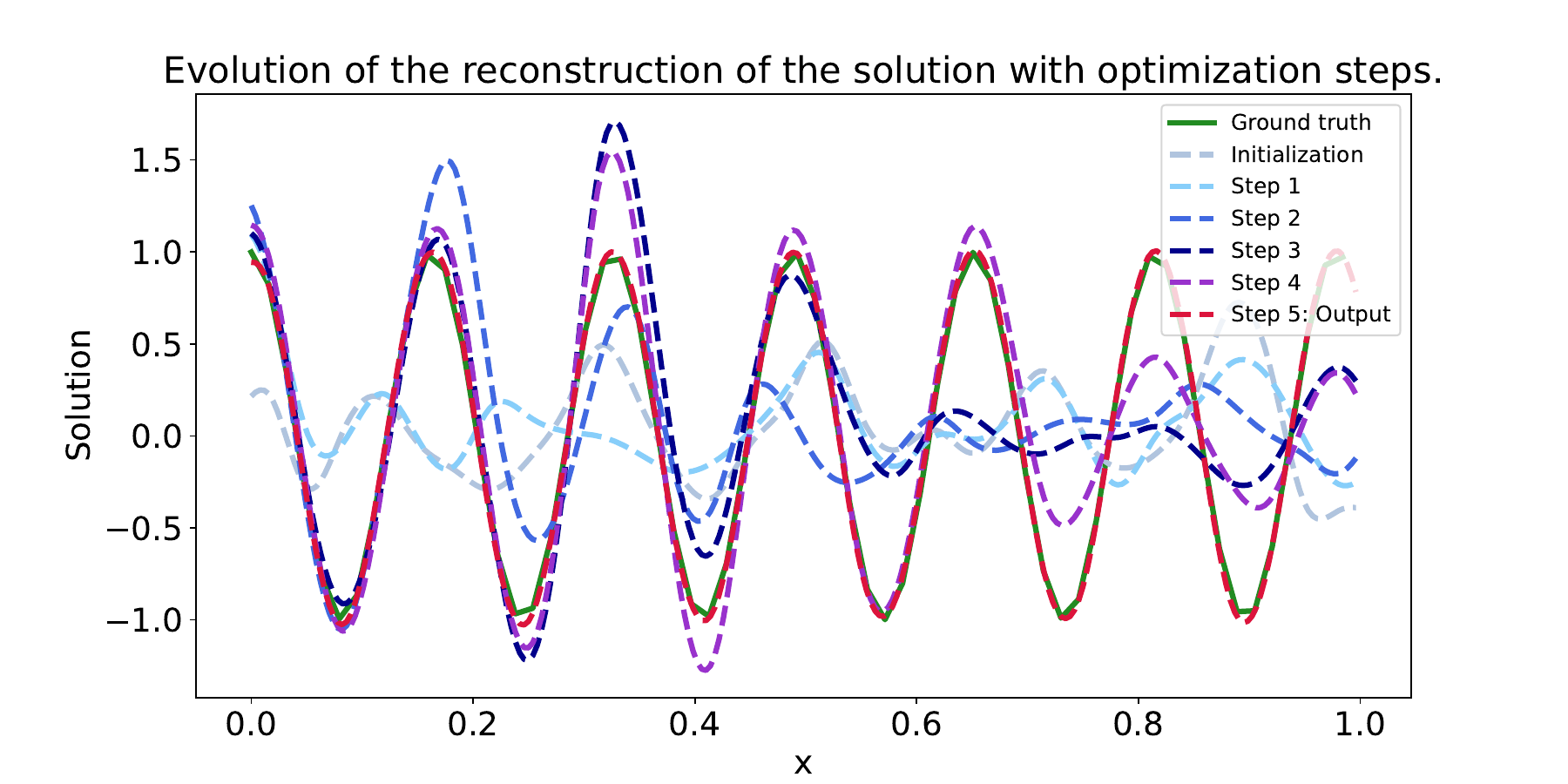}
    \caption{}
    \label{fig:ngdhelmholtz1}
\end{subfigure}
\hfill
\begin{subfigure}{0.48\textwidth}
    \centering
    \includegraphics[width=\textwidth]{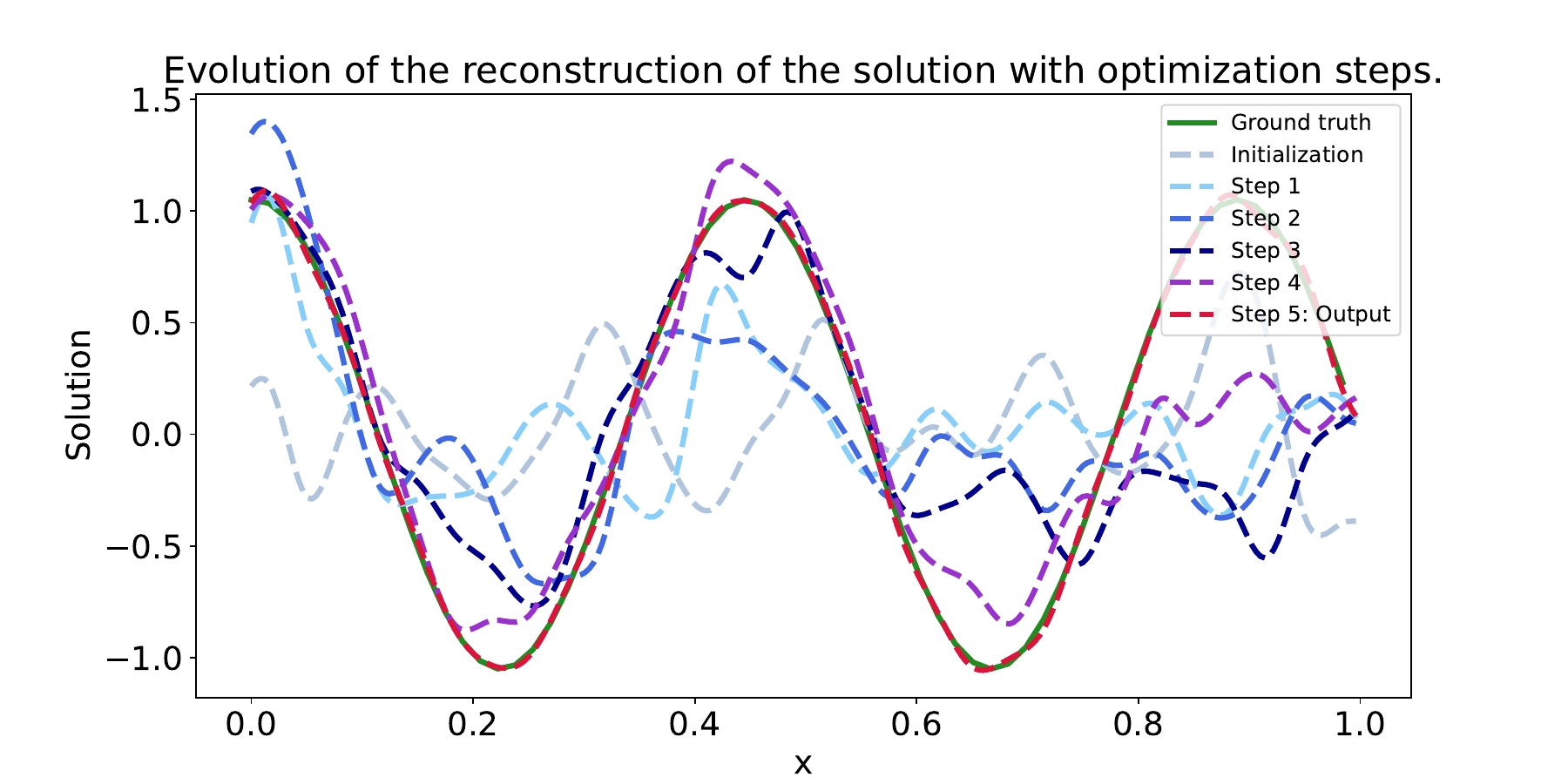}
    \caption{}
    \label{fig:ngdhelmholtz2}
\end{subfigure}
\caption{Reconstruction of the solution using our optimizer on the Helmholtz dataset. }
\end{figure}
\vfill
\begin{figure}[htbp]
\begin{subfigure}{0.48\textwidth}
    \centering
    \includegraphics[width=\textwidth]{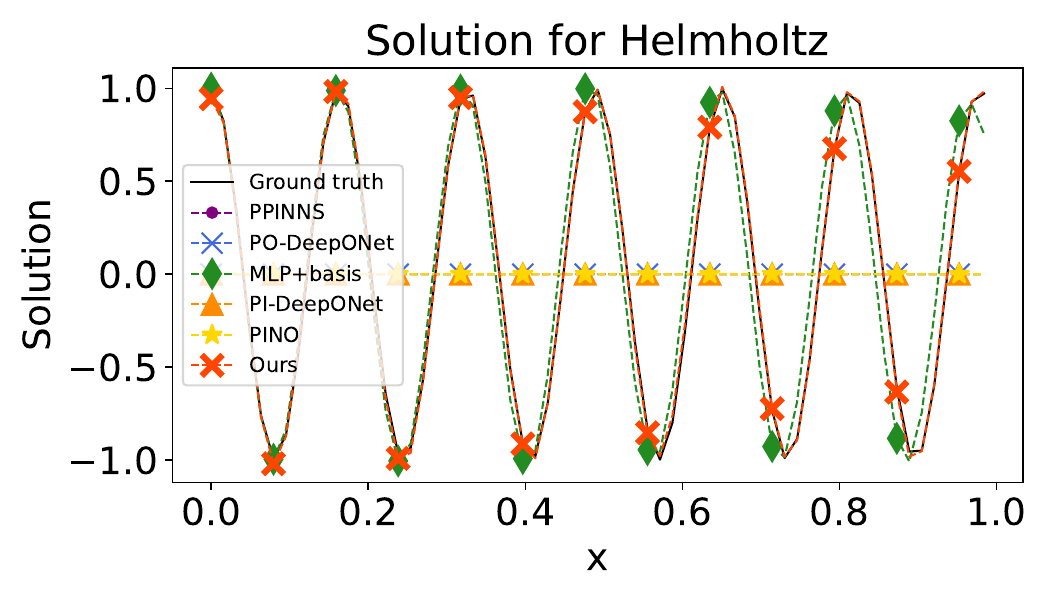}
    \caption{}
    \label{fig:baselinehelmholtz1}
\end{subfigure}
\hfill
\begin{subfigure}{0.48\textwidth}
    \centering
    \includegraphics[width=\textwidth]{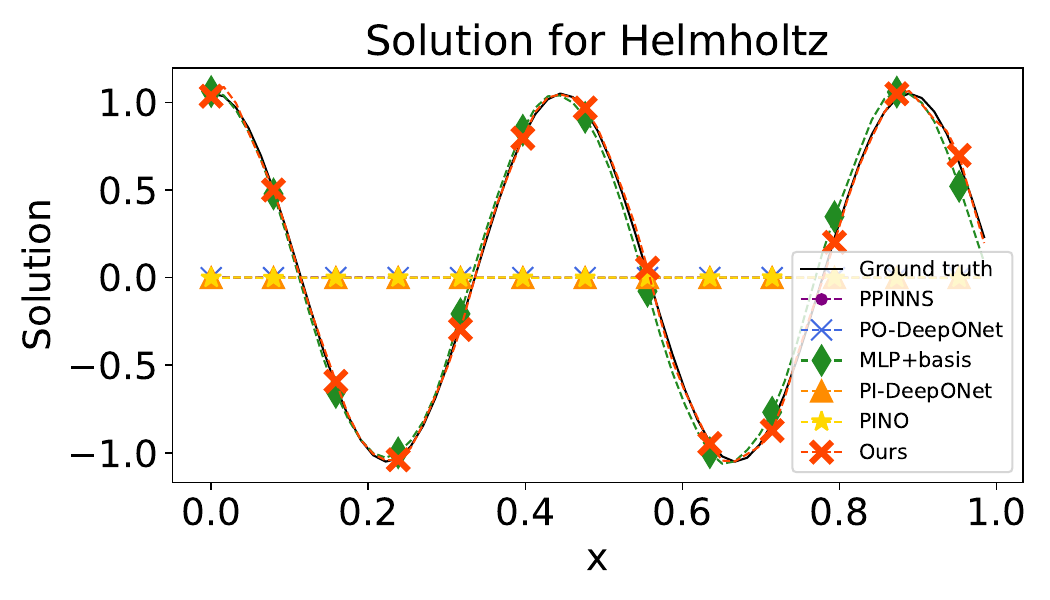}
    \caption{}
    \label{fig:baselinehelmholtz2}
\end{subfigure}
\caption{Visual comparison of the solutions for the Helmholtz equation. }
\end{figure}
\vfill
\begin{figure}[htbp]
    \centering
    \includegraphics[width=\textwidth]{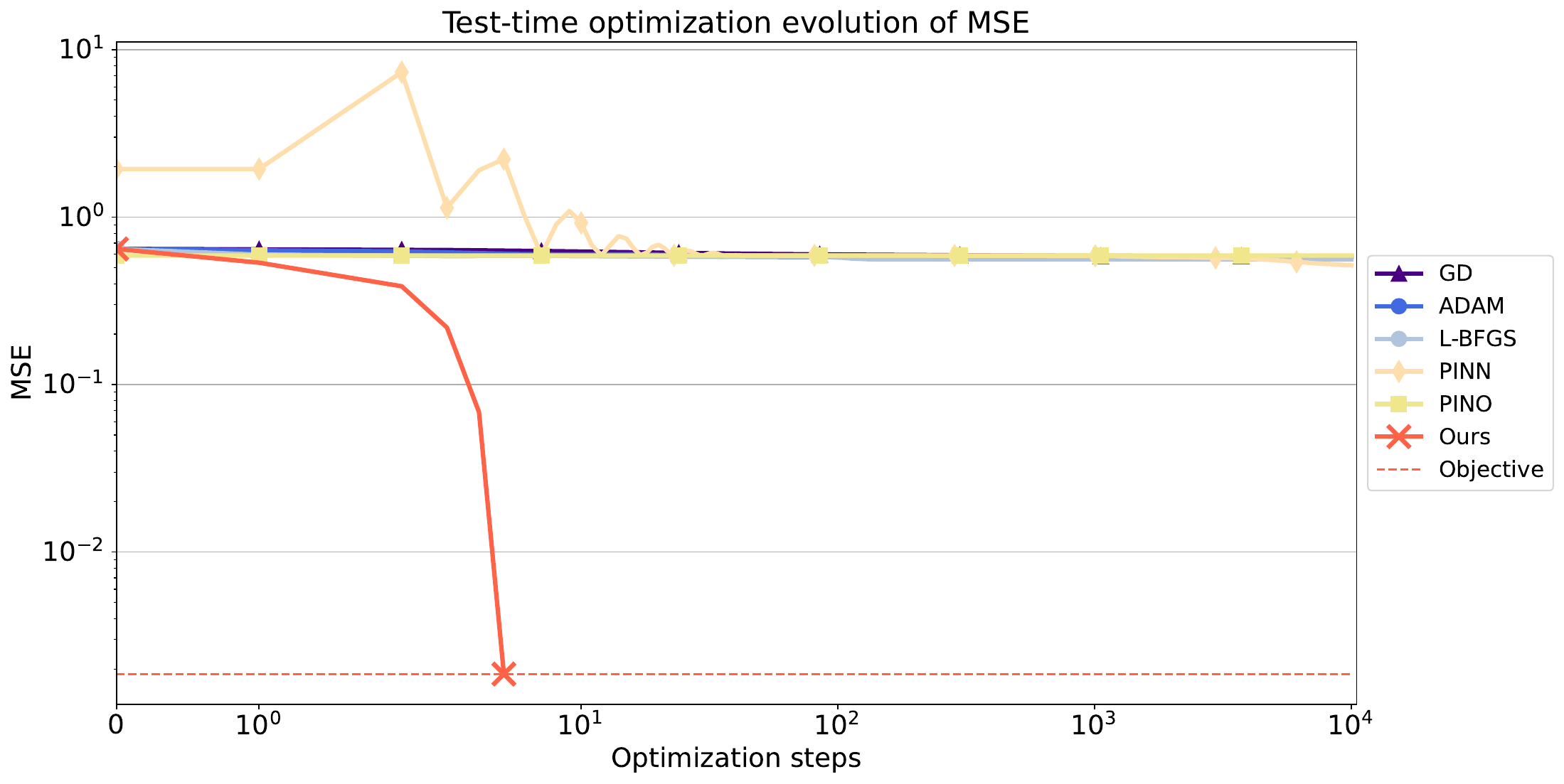}
    \caption{Test-time optimization based on the physical residual loss $\mathcal{L}_{\textnormal{PDE}}$ on \textit{Helmholtz}. Note that, even though hardly visible on this figure, the optimization is running very slowly and the PINN MSE (orange) decreases for the last steps. This dataset will probably need even more steps before convergence. }
    \label{fig:testopthelmholtz}
\end{figure}

\vspace{10cm}
\subsection{Poisson}
\begin{figure}[htbp]
\begin{subfigure}{0.48\textwidth}
    \centering
    \includegraphics[width=\textwidth]{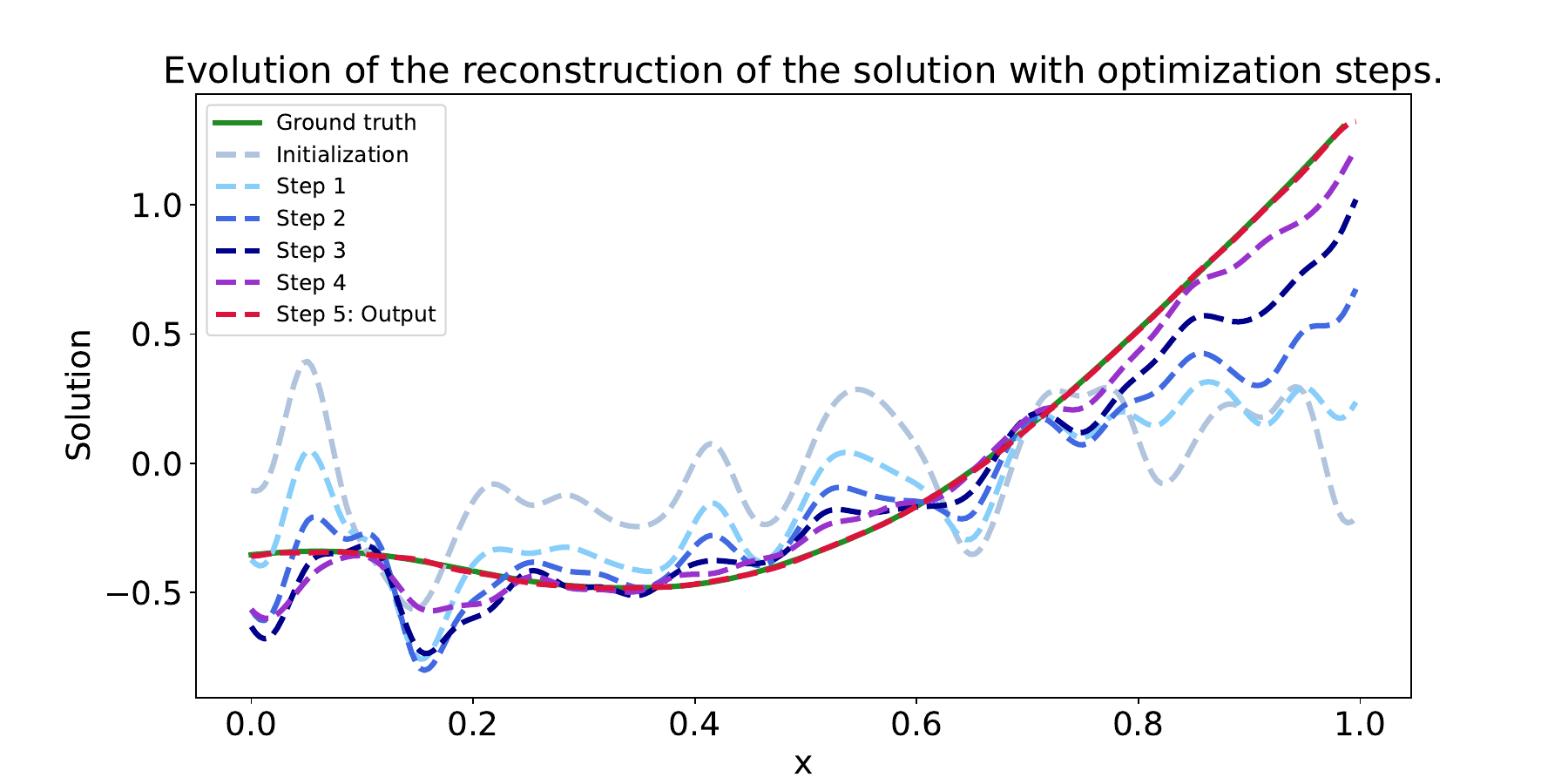}
    \caption{}
    \label{fig:ngdpoisson1}
\end{subfigure}
\hfill
\begin{subfigure}{0.48\textwidth}
    \centering
    \includegraphics[width=\textwidth]{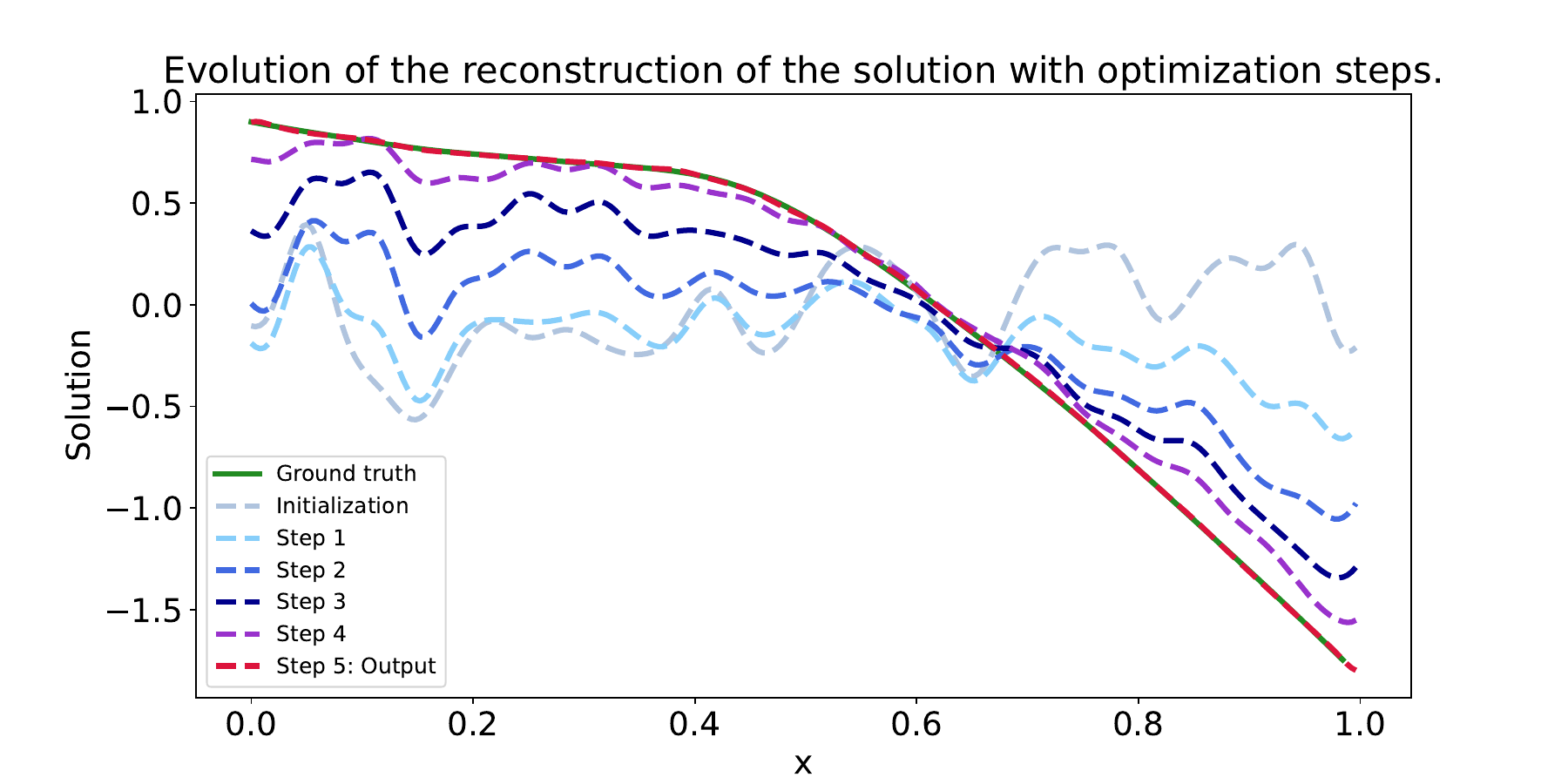}
    \caption{}
    \label{fig:ngdpoisson2}
\end{subfigure}
\caption{Reconstruction of the solution using our optimizer on the Poisson dataset. }
\end{figure}
\vfill
\begin{figure}[htbp]
\begin{subfigure}{0.48\textwidth}
    \centering
    \includegraphics[width=\textwidth]{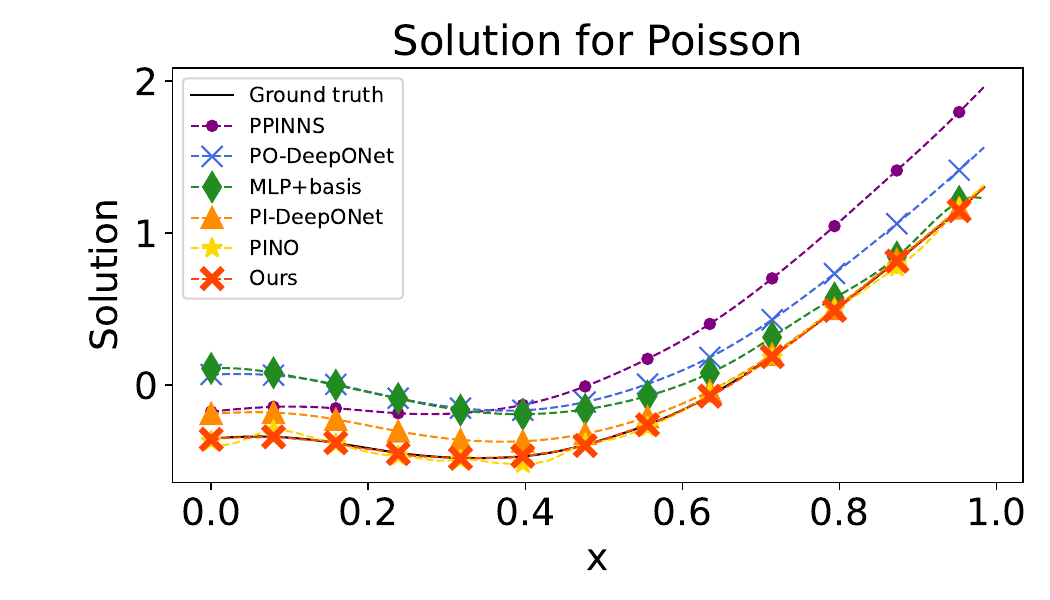}
    \caption{}
    \label{fig:baselinepoisson1}
\end{subfigure}
\hfill
\begin{subfigure}{0.48\textwidth}
    \centering
    \includegraphics[width=\textwidth]{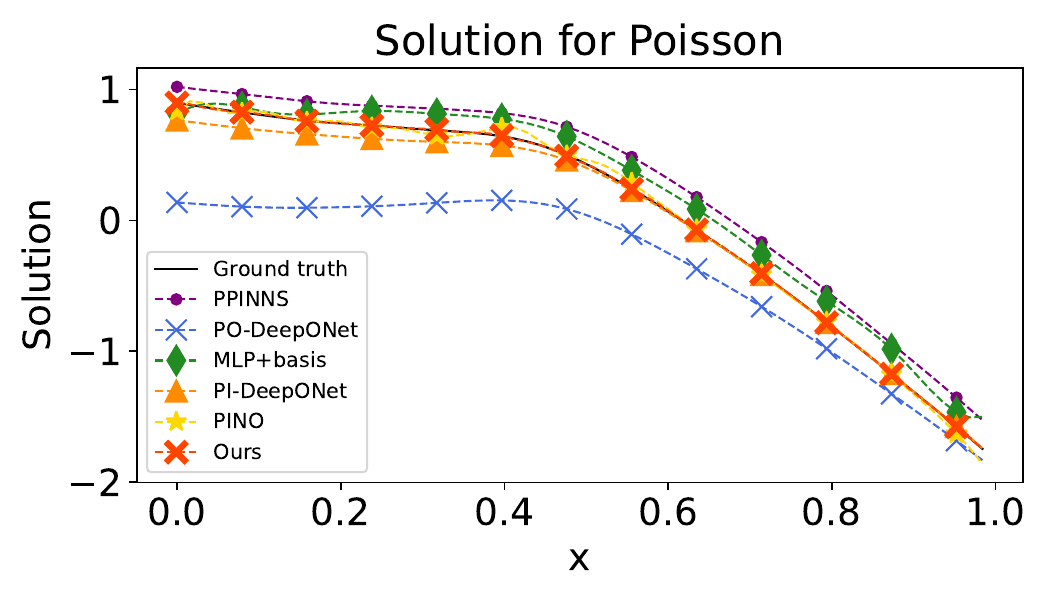}
    \caption{}
    \label{fig:baselinepoisson2}
\end{subfigure}
\caption{Visual comparison of the solutions for the Poisson equation. }
\end{figure}
\vfill
\begin{figure}[htbp]
    \centering
    \includegraphics[width=\textwidth]{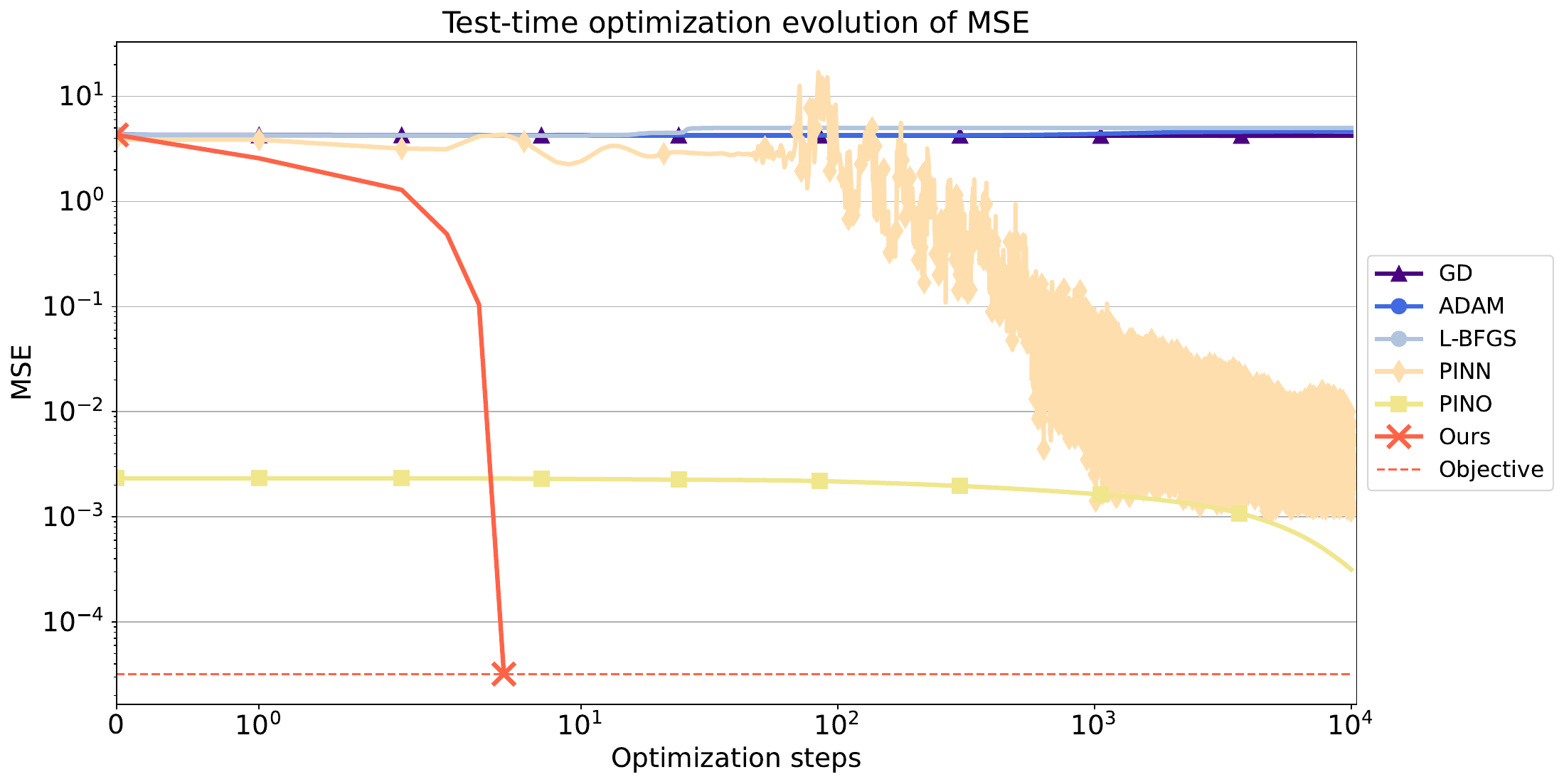}
    \caption{Test-time optimization based on the physical residual loss $\mathcal{L}_{\textnormal{PDE}}$ on \textit{Poisson}. }
    \label{fig:testoptpoisson}
\end{figure}

\vspace{4cm}
\subsection{Reaction-Diffusion}
\begin{figure}[htbp]
\begin{subfigure}{\textwidth}
    \centering
    \includegraphics[width=\textwidth]{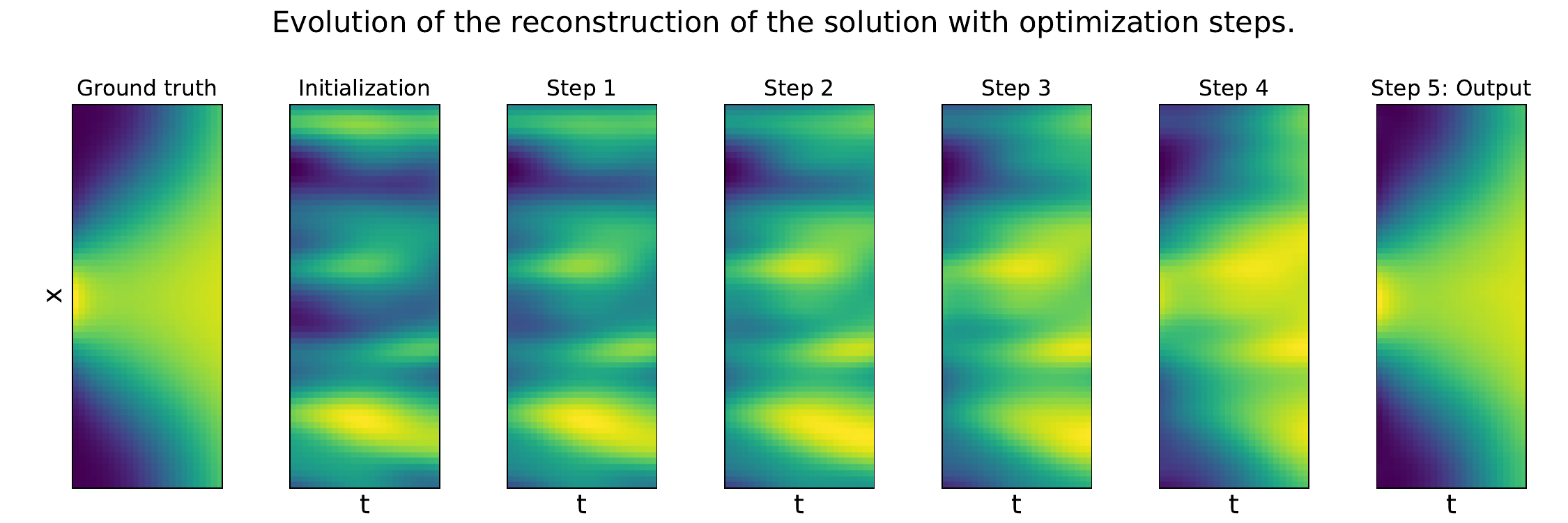}
    \caption{}
    \label{fig:ngdnlrd1}
\end{subfigure}
\vfill
\begin{subfigure}{\textwidth}
    \centering
    \includegraphics[width=\textwidth]{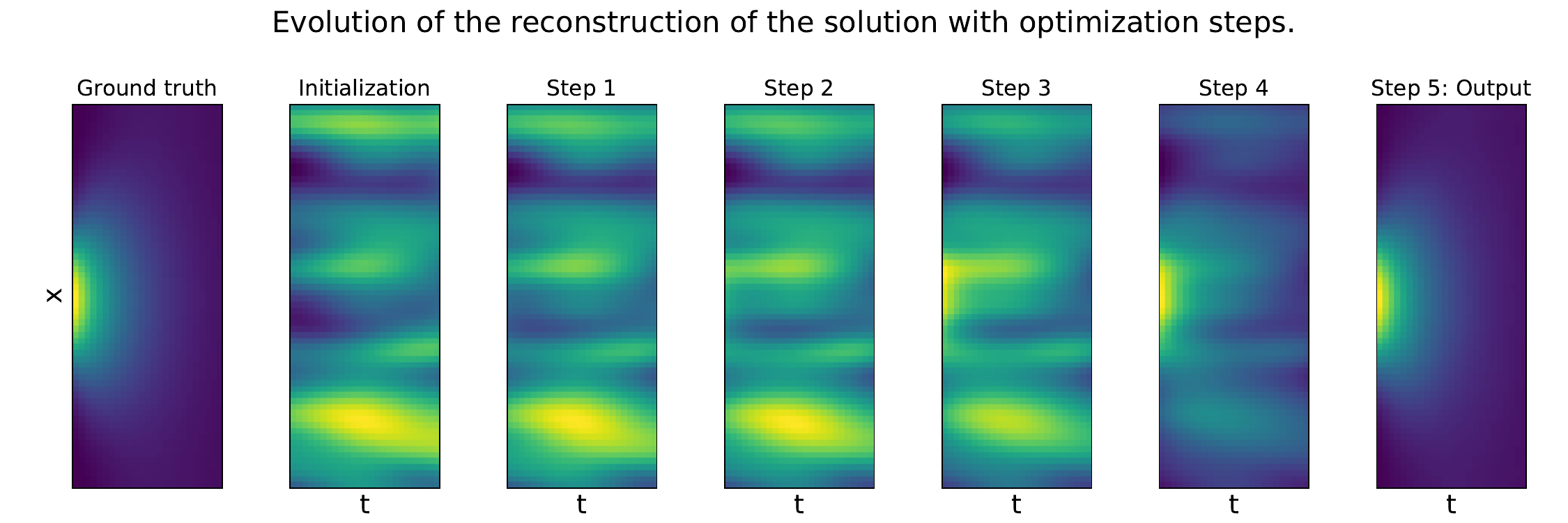}
    \caption{}
    \label{fig:ngdnlrd2}
\end{subfigure}
\caption{Reconstruction of the solution using our optimizer on the Reaction-Diffusion dataset. }
\end{figure}
\vfill
\begin{figure}[htbp]
\begin{subfigure}{\textwidth}
    \centering
    \includegraphics[width=\textwidth]{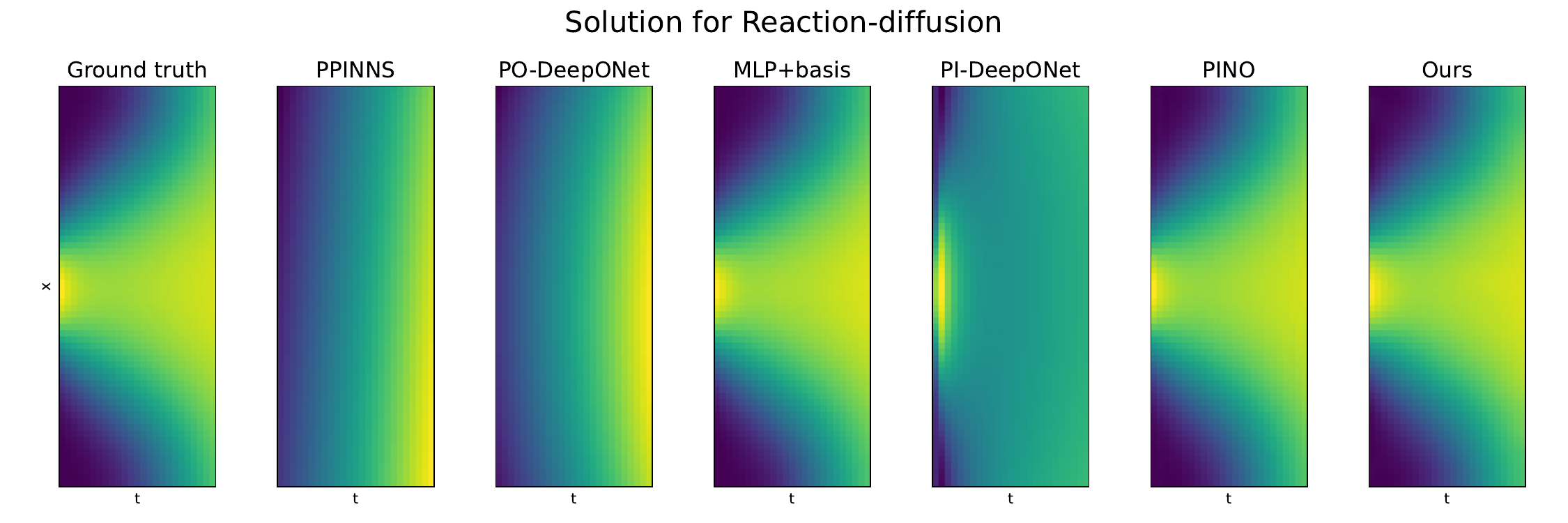}
    \caption{}
    \label{fig:baselinenlrd1}
\end{subfigure}
\vfill
\begin{subfigure}{\textwidth}
    \centering
    \includegraphics[width=\textwidth]{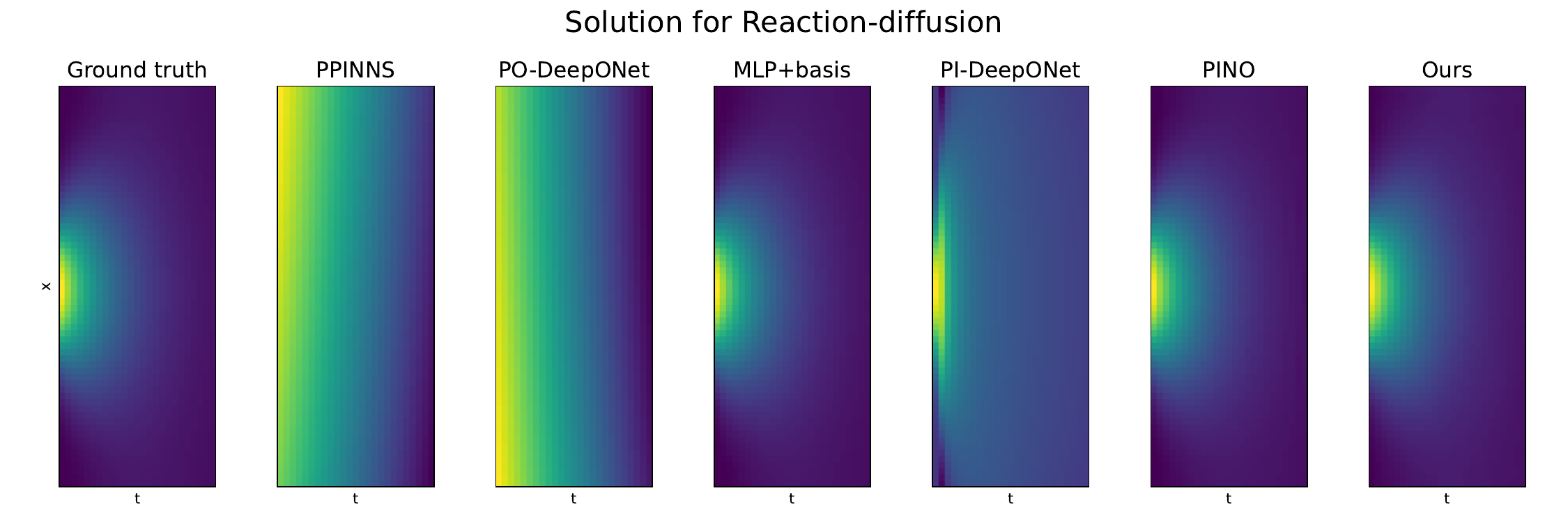}
    \caption{}
    \label{fig:baselinenlrd2}
\end{subfigure}
\caption{Visual comparison of the solutions for the Reaction-Diffusion equation.}
\end{figure}
\vfill
\begin{figure}[htbp]
    \centering
    \includegraphics[width=\textwidth]{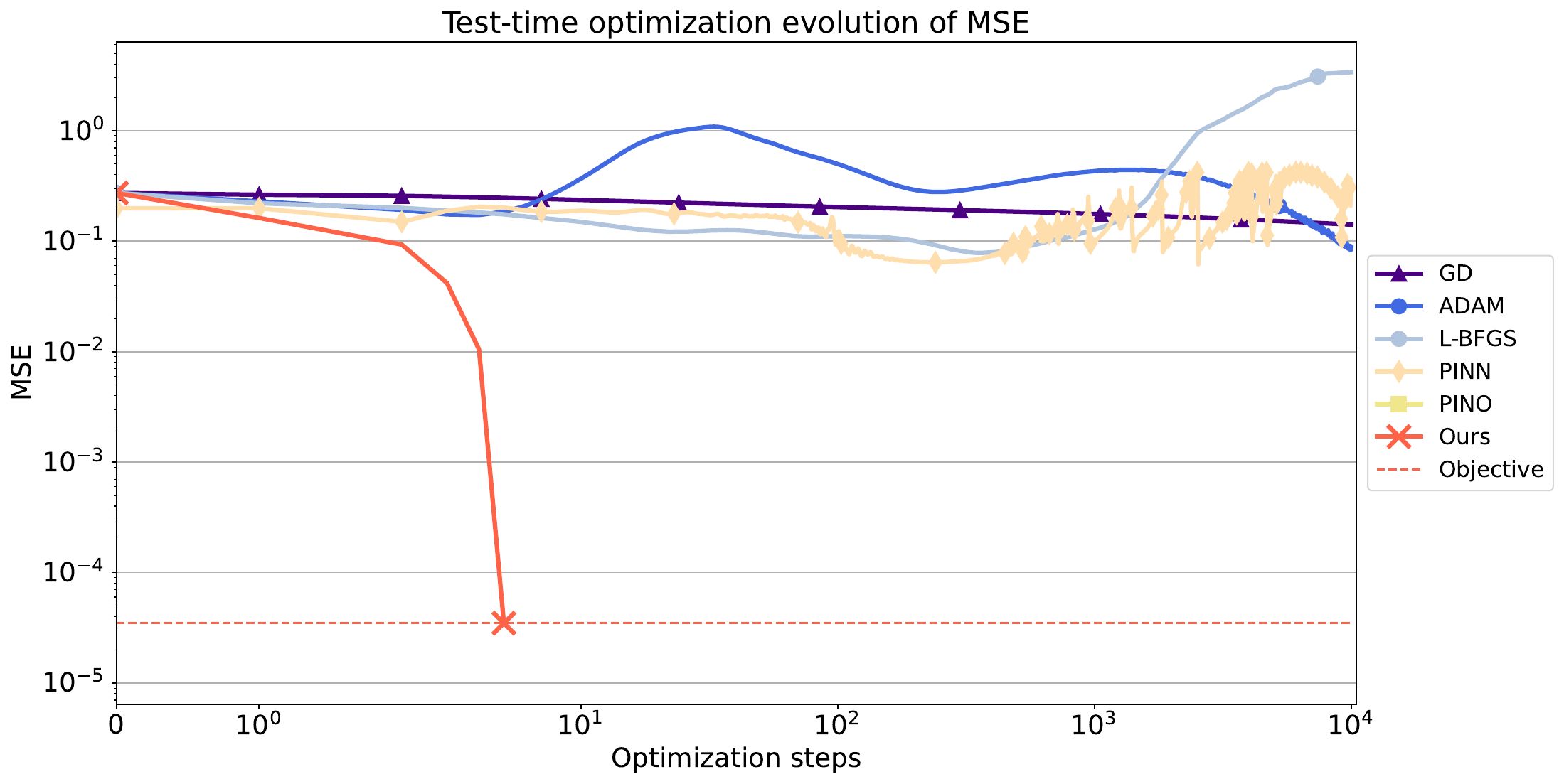}
    \caption{Test-time optimization based on the physical residual loss $\mathcal{L}_{\textnormal{PDE}}$ on \textit{NLRD}. }
    \label{fig:testoptnlrd}
\end{figure}

\clearpage
\subsection{Darcy}
\begin{figure}[htbp]
\begin{subfigure}{\textwidth}
    \centering
    \includegraphics[width=\textwidth, trim={0cm 4cm 0cm 0cm}, clip]{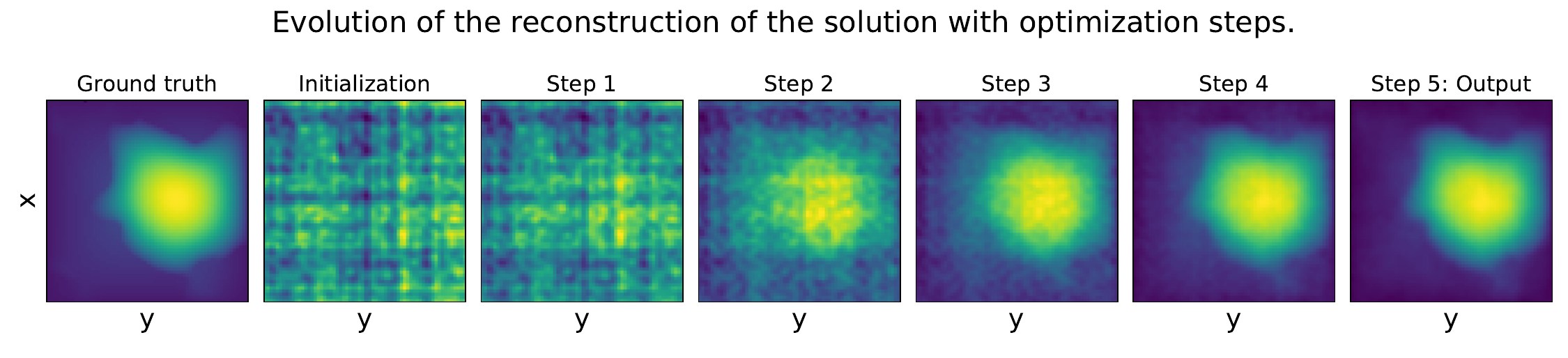}
    \caption{}
    \label{fig:ngddarcy1}
\end{subfigure}
\vfill
\begin{subfigure}{\textwidth}
    \centering
    \includegraphics[width=\textwidth, trim={0cm 4cm 0cm 0cm}, clip]{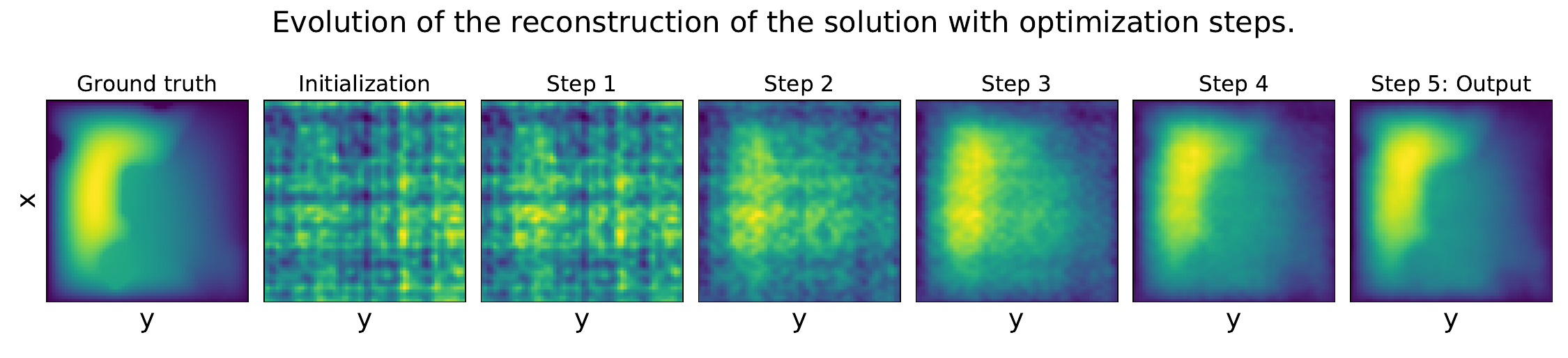}
    \caption{}
    \label{fig:ngddarcy2}
\end{subfigure}
\caption{Reconstruction of the solution using our optimizer on the Darcy dataset. }
\end{figure}
\vfill
\begin{figure}[htbp]
\begin{subfigure}{\textwidth}
    \centering
    \includegraphics[width=\textwidth, trim={0cm 2cm 0cm 0cm}, clip]{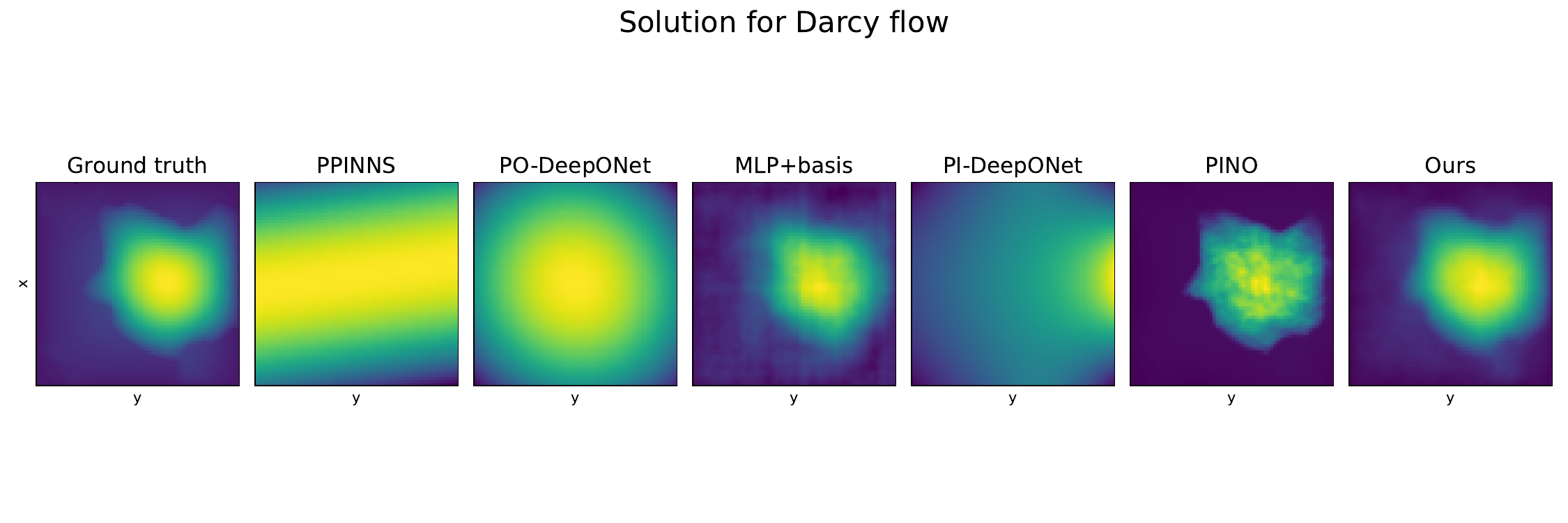}
    \caption{}
    \label{fig:baselinedarcy1}
\end{subfigure}
\vfill
\begin{subfigure}{\textwidth}
    \centering
    \includegraphics[width=\textwidth, trim={0cm 2cm 0cm 0cm}, clip]{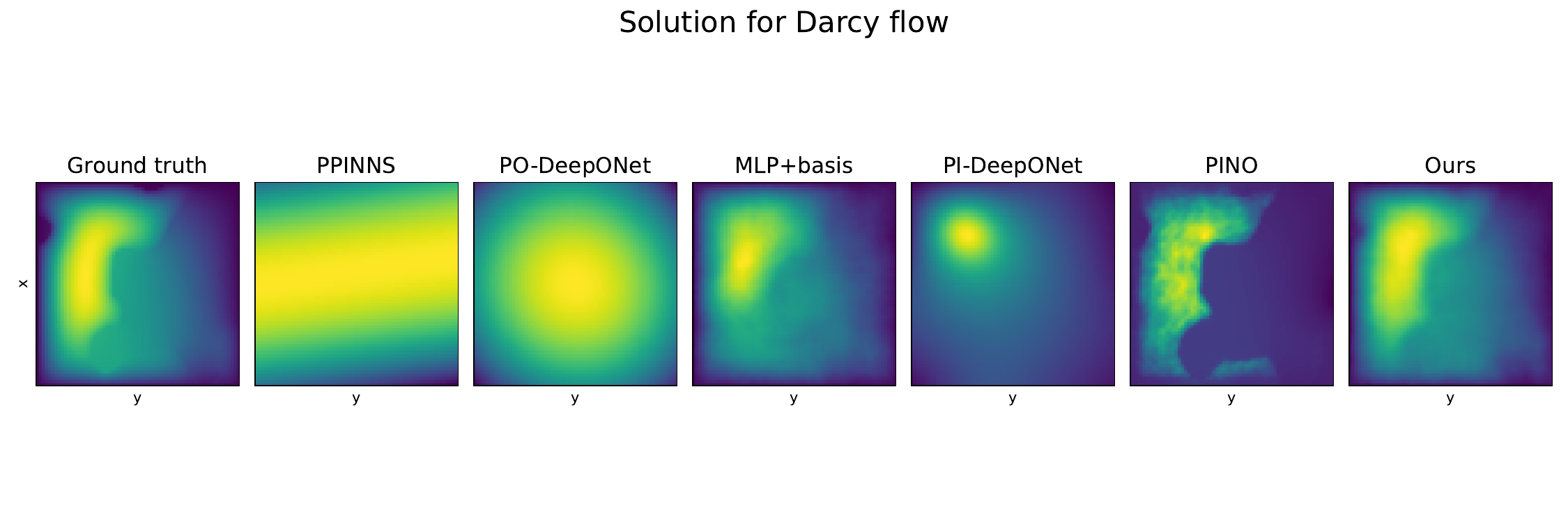}
    \caption{}
    \label{fig:baselinedarcy2}
\end{subfigure}
\caption{Visual comparison of the solutions for the Darcy equation.}
\end{figure}
\vfill
\begin{figure}[htbp]
    \centering
    \includegraphics[width=\textwidth]{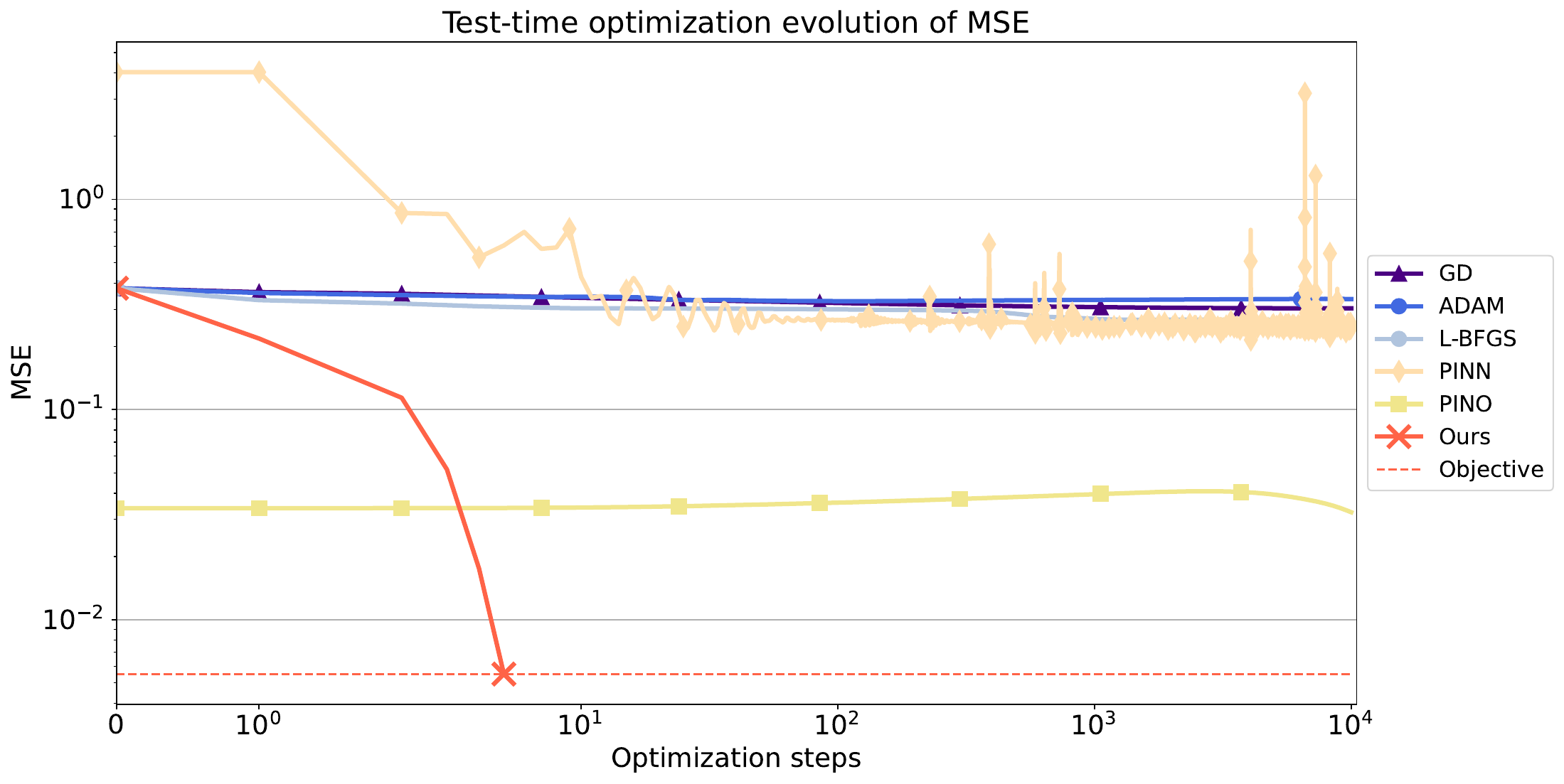}
    \caption{Test-time optimization based on the physical residual loss $\mathcal{L}_{\textnormal{PDE}}$ on \textit{Darcy}. }
    \label{fig:testoptdarcy}
\end{figure}

% \clearpage
% \subsection{Helmholtz 2d}
% \begin{figure}[htbp]
% \begin{subfigure}{\textwidth}
%     \centering
%     \includegraphics[width=\textwidth, trim={0cm 4cm 0cm 0cm}, clip]{img/.pdf}
%     \caption{}
%     \label{fig:ngdh2d1}
% \end{subfigure}
% \vfill
% \begin{subfigure}{\textwidth}
%     \centering
%     \includegraphics[width=\textwidth, trim={0cm 4cm 0cm 0cm}, clip]{img/visualization_NGD_8_h2d_lm.pdf}
%     \caption{}
%     \label{fig:ngdh2d2}
% \end{subfigure}
% \caption{Reconstruction of the solution using our optimizer on the Helmholtz 2d dataset. }
% \end{figure}
% \vfill
% \begin{figure}[htbp]
% \begin{subfigure}{\textwidth}
%     \centering
%     \includegraphics[width=\textwidth, trim={0cm 2cm 0cm 0cm}, clip]{img/vis_baselines_h2d_6_lm.pdf}
%     \caption{}
%     \label{fig:baselineh2d1}
% \end{subfigure}
% \vfill
% \begin{subfigure}{\textwidth}
%     \centering
%     \includegraphics[width=\textwidth, trim={0cm 2cm 0cm 0cm}, clip]{img/vis_baselines_h2d_8_lm.pdf}
%     \caption{}
%     \label{fig:baselineh2d2}
% \end{subfigure}
% \caption{Visual comparison of the solutions for the Helmholtz 2d equation.}
% \end{figure}
% \vfill
% \begin{figure}[htbp]
%     \centering
%     \includegraphics[width=\textwidth]{img/compare_test_opt_h2d_10000_MSE_large_2CR.pdf}
%     \caption{Test-time optimization based on the physical residual loss $\mathcal{L}_{\textnormal{PDE}}$ on \textit{Helmholtz 2d}. }
%     \label{fig:testopth2d}
% \end{figure}

\clearpage
\subsection{Heat}
\begin{figure}[htbp]
\begin{subfigure}{\textwidth}
    \centering
    \includegraphics[width=\textwidth, trim={4cm 0cm 4cm 0cm}, clip]{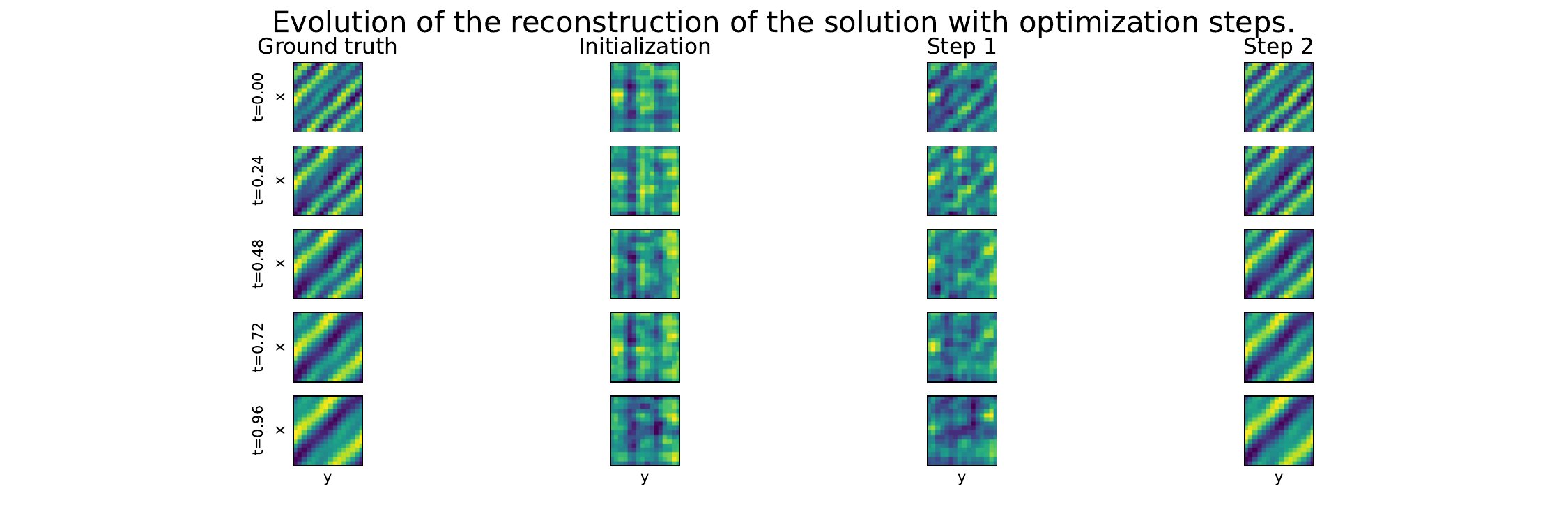}
    \caption{}
    \label{fig:ngdheat1}
\end{subfigure}
\vfill
\begin{subfigure}{\textwidth}
    \centering
    \includegraphics[width=\textwidth, trim={4cm 0cm 4cm 0cm}, clip]{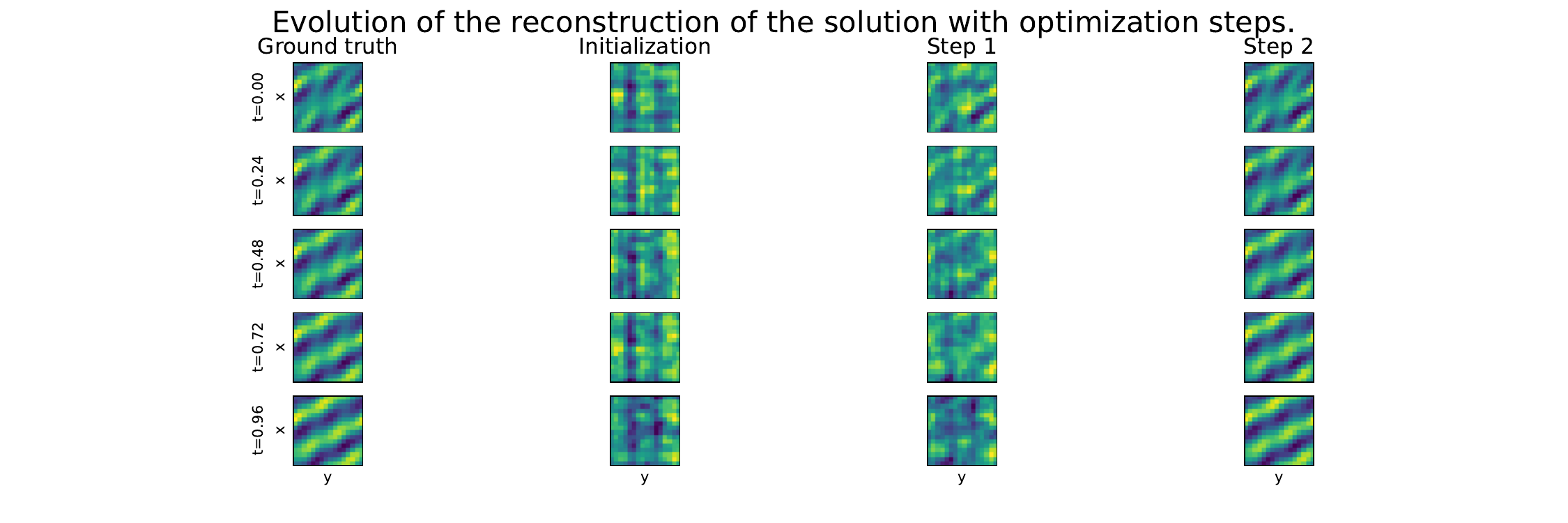 }
    \caption{}
    \label{fig:ngdheat2}
\end{subfigure}
\caption{Reconstruction of the solution using our optimizer on the Heat dataset. }
\end{figure}
\vfill
\begin{figure}[htbp]
\begin{subfigure}{\textwidth}
    \centering
    \includegraphics[width=\textwidth, trim={0cm 0cm 0cm 0cm}, clip]{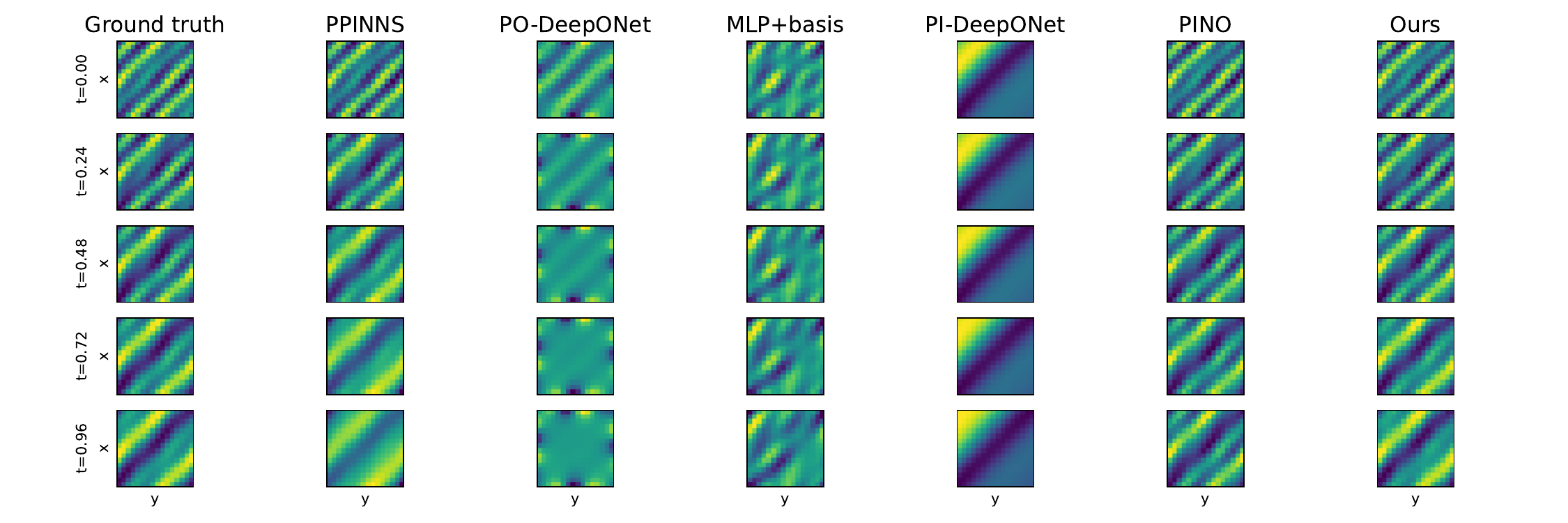}
    \caption{}
    \label{fig:baselineheat1}
\end{subfigure}
\vfill
\begin{subfigure}{\textwidth}
    \centering
    \includegraphics[width=\textwidth, trim={0cm 0cm 0cm 0cm}, clip]{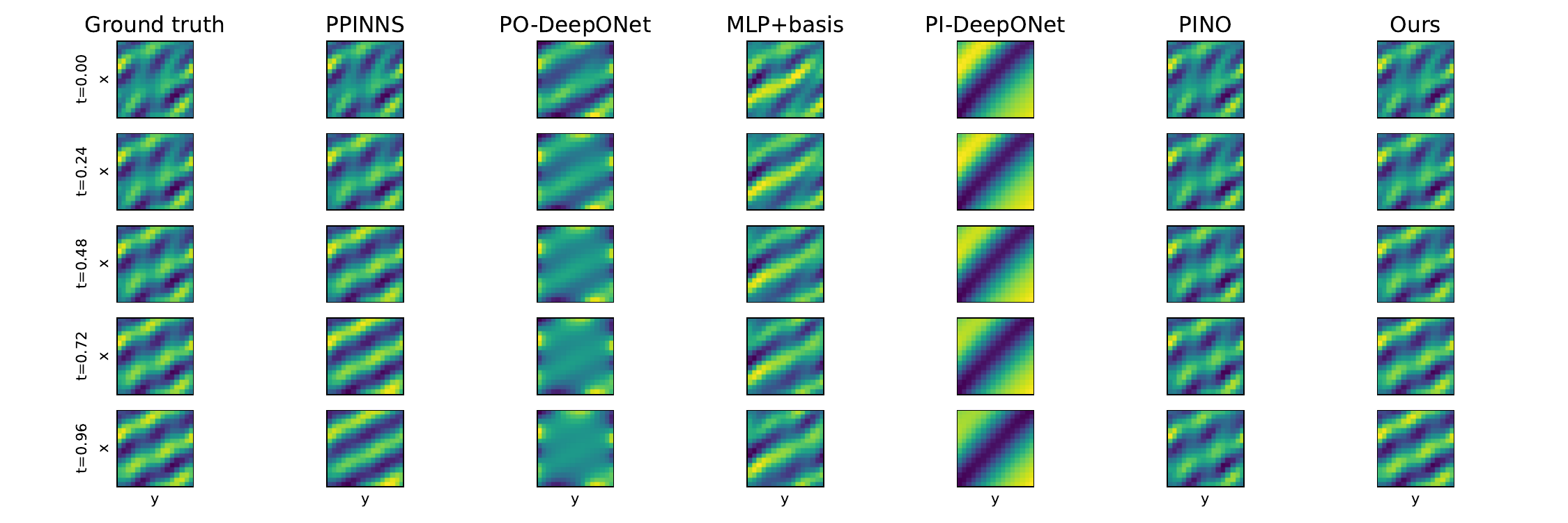}
    \caption{}
    \label{fig:baselineheat2}
\end{subfigure}
\caption{Visual comparison of the solutions for the Heat equation.}
\end{figure}
\vfill
\begin{figure}[htbp]
    \centering
    \includegraphics[width=\textwidth]{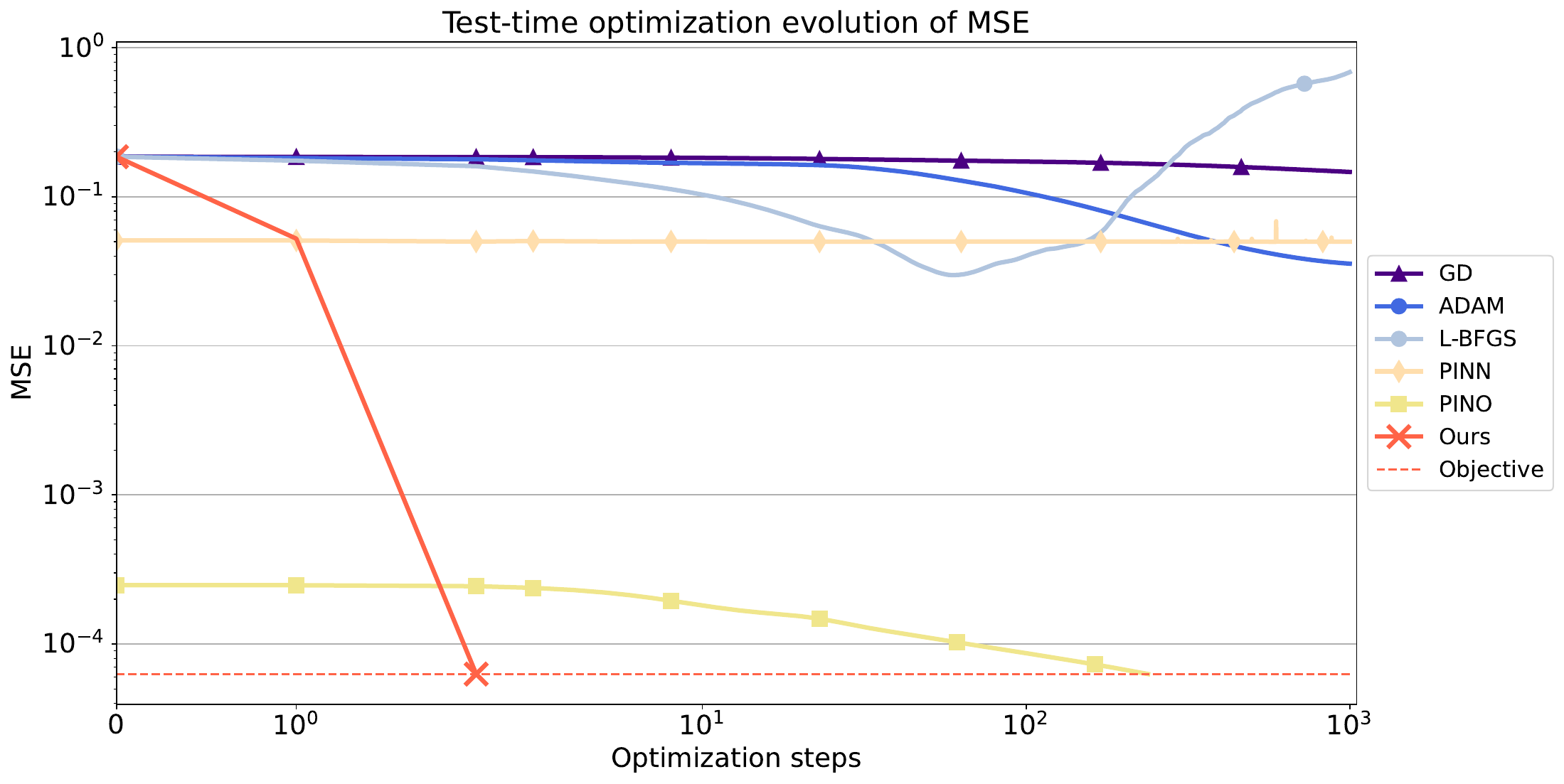}
    \caption{Test-time optimization based on the physical residual loss $\mathcal{L}_{\textnormal{PDE}}$ on \textit{Heat}. For computational reasons, this experiment has been conducted on $1,000$ steps only. }
    \label{fig:testoptheat}
\end{figure}

\clearpage
\subsection{Advection}
\begin{figure}[htbp]
\begin{subfigure}{\textwidth}
    \centering
    \includegraphics[width=\textwidth]{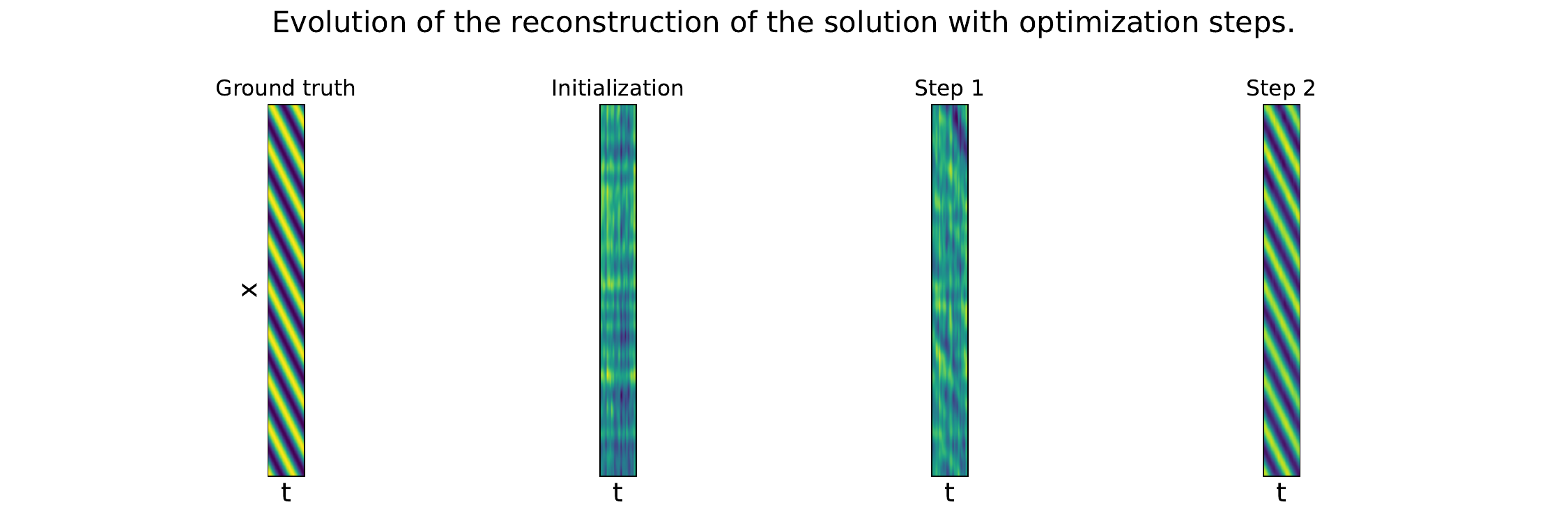}
    \caption{}
    \label{fig:ngdadvections1}
\end{subfigure}
\vfill
\begin{subfigure}{\textwidth}
    \centering
    \includegraphics[width=\textwidth]{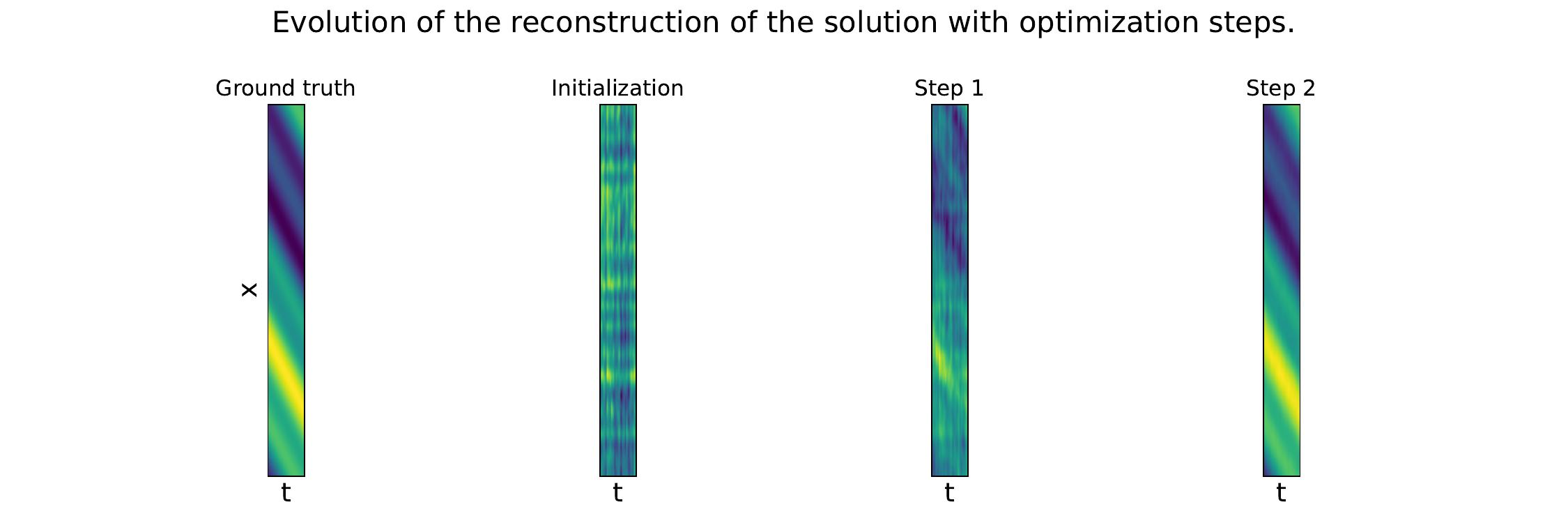}
    \caption{}
    \label{fig:ngdadvections2}
\end{subfigure}
\caption{Reconstruction of the solution using our optimizer on the Advection dataset. }
\end{figure}
% \vfill
\begin{figure}[htbp]
\begin{subfigure}{\textwidth}
    \centering
    \includegraphics[width=\textwidth]{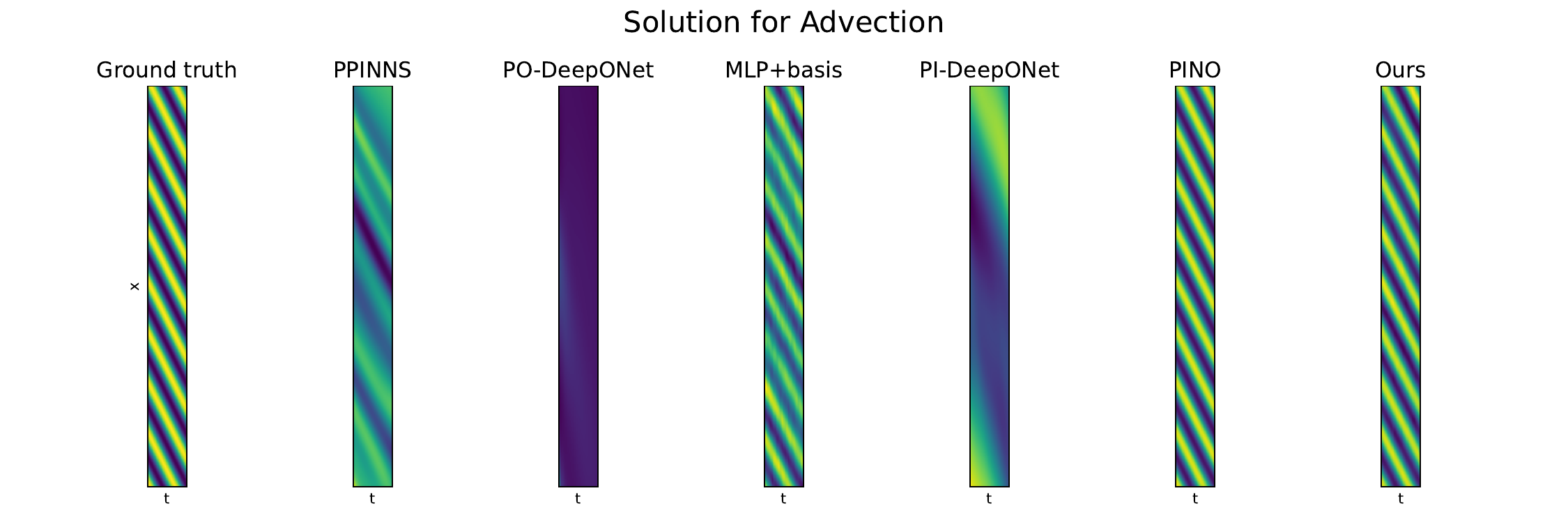}
    \caption{}
    \label{fig:baselineadvections1}
\end{subfigure}
\vfill
\begin{subfigure}{\textwidth}
    \centering
    \includegraphics[width=\textwidth]{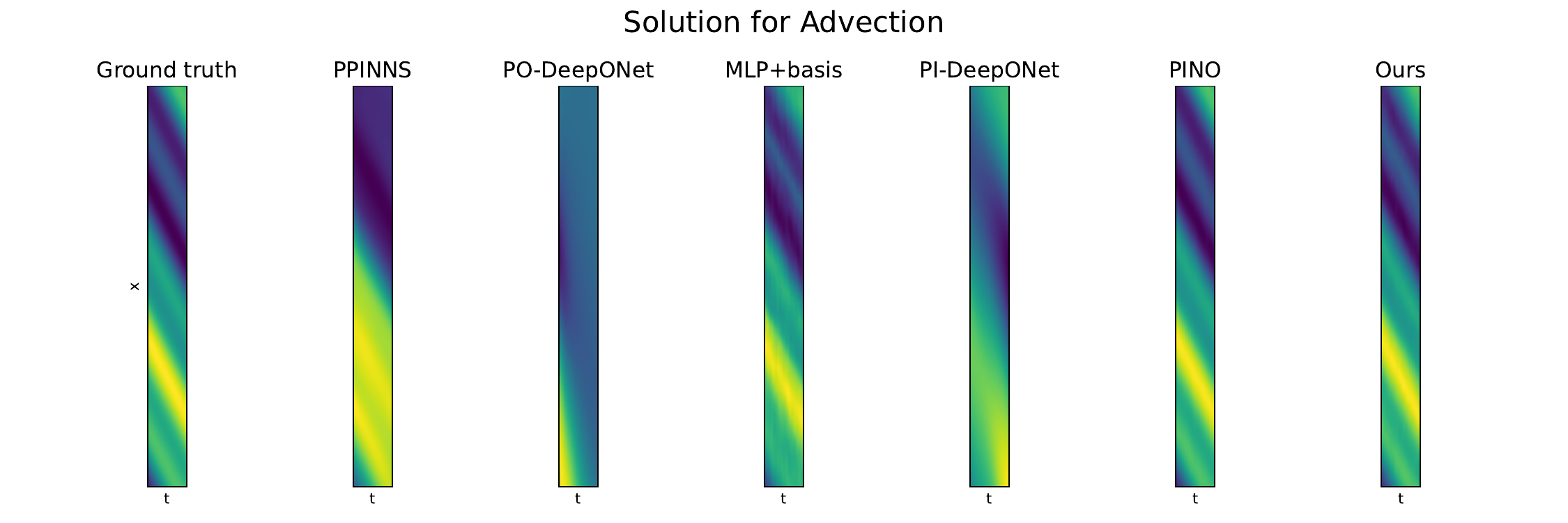}
    \caption{}
    \label{fig:baselineadvections2}
\end{subfigure}
\caption{Visual comparison of the solutions for the Advection equation.}
\end{figure}
\vfill
\begin{figure}[htbp]
    \centering
    \includegraphics[width=\textwidth]{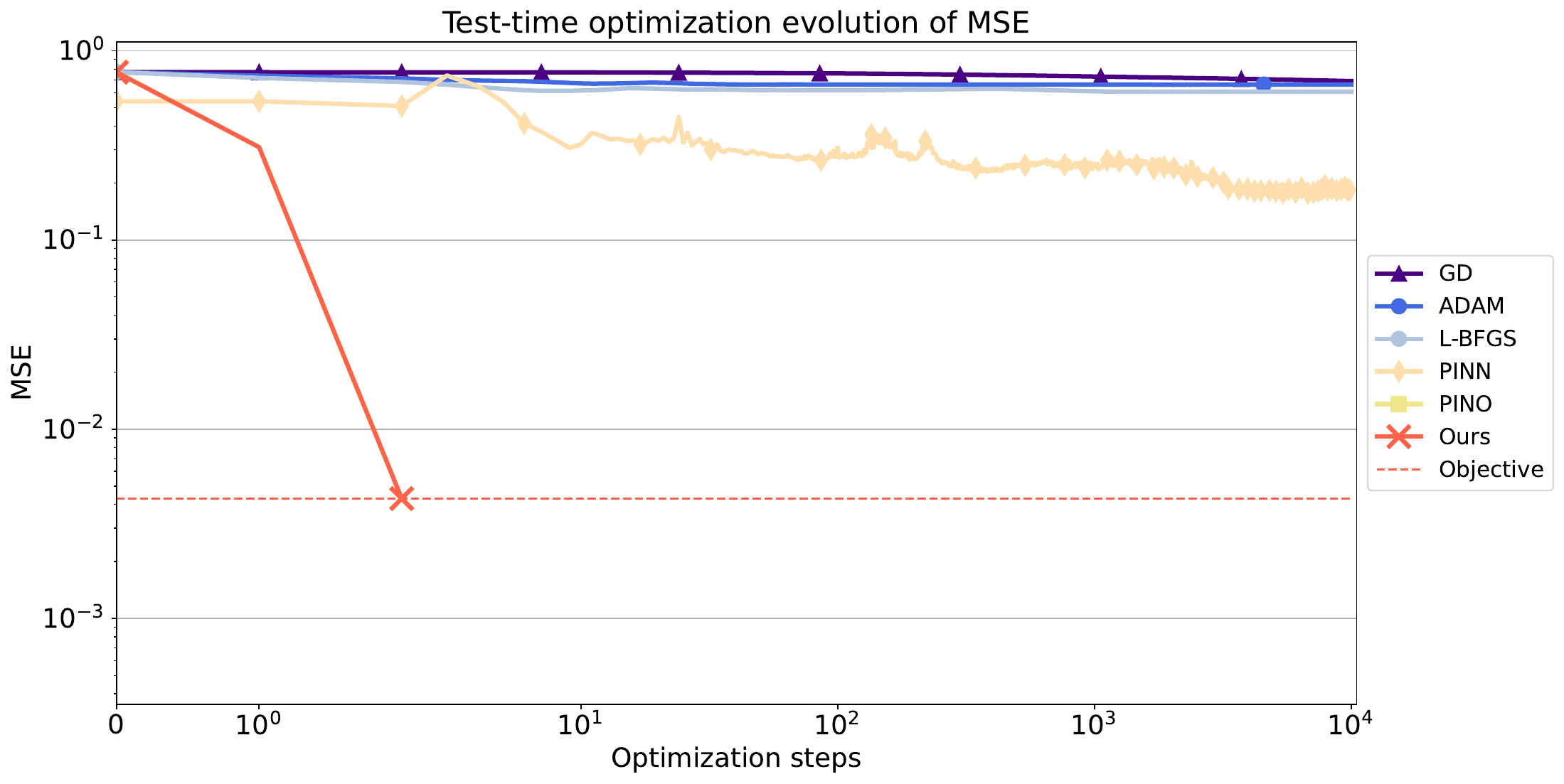}
    \caption{Test-time optimization based on the physical residual loss $\mathcal{L}_{\textnormal{PDE}}$ on \textit{Advection}. }
    \label{fig:testoptadvections}
\end{figure}

\clearpage
\subsection{Non-Linear Reaction-Diffusion with Initial Conditions}
\begin{figure}[htbp]
\begin{subfigure}{\textwidth}
    \centering
    \includegraphics[width=\textwidth]{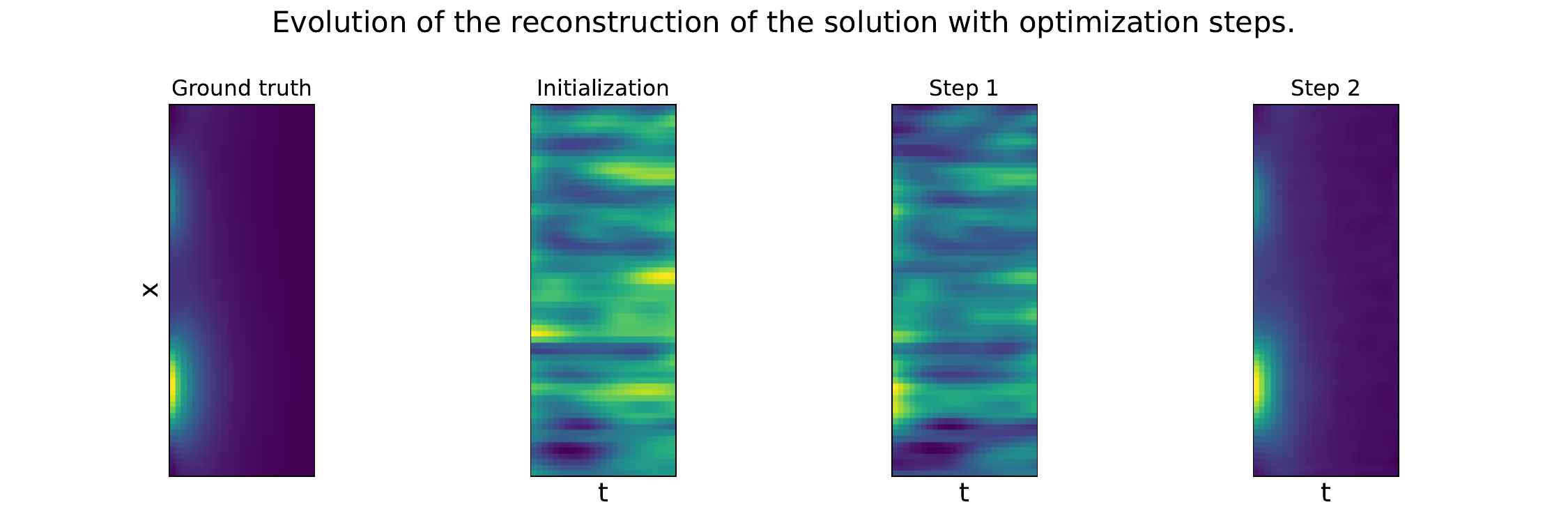}
    \caption{}
    \label{fig:ngdnlrdics1}
\end{subfigure}
\vfill
\begin{subfigure}{\textwidth}
    \centering
    \includegraphics[width=\textwidth]{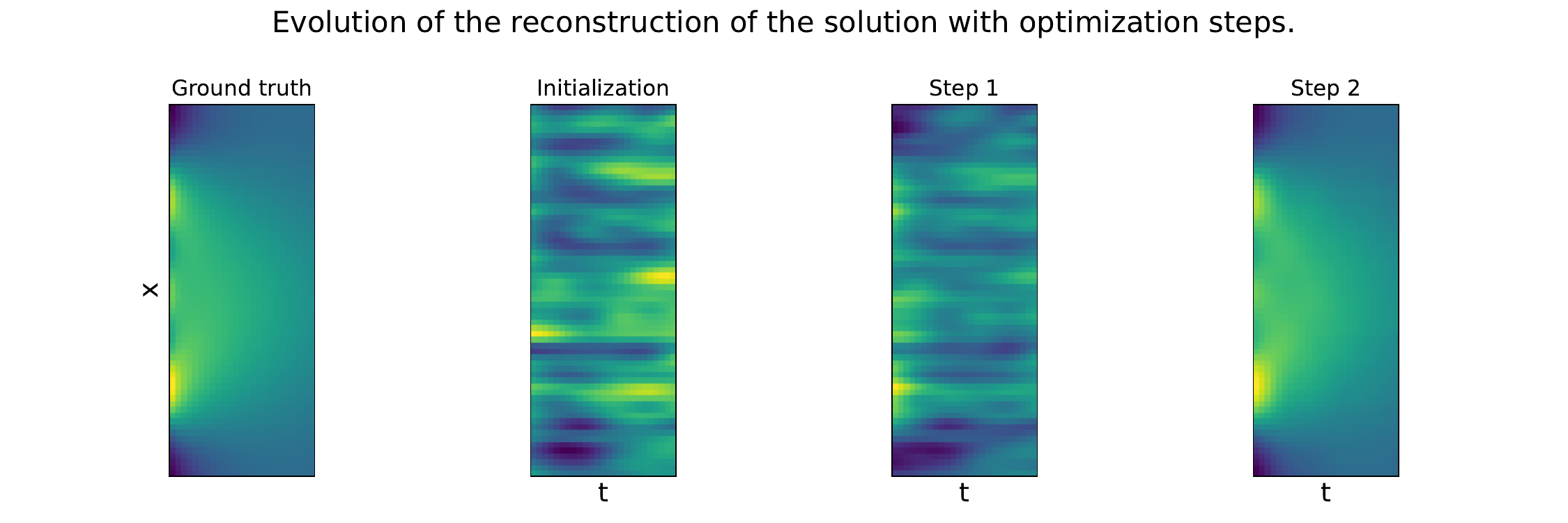}
    \caption{}
    \label{fig:ngdnlrdics2}
\end{subfigure}
\caption{Reconstruction of the solution using our optimizer on the NLRDIC dataset. }
\end{figure}
% \vfill
\begin{figure}[htbp]
\begin{subfigure}{\textwidth}
    \centering
    \includegraphics[width=\textwidth]{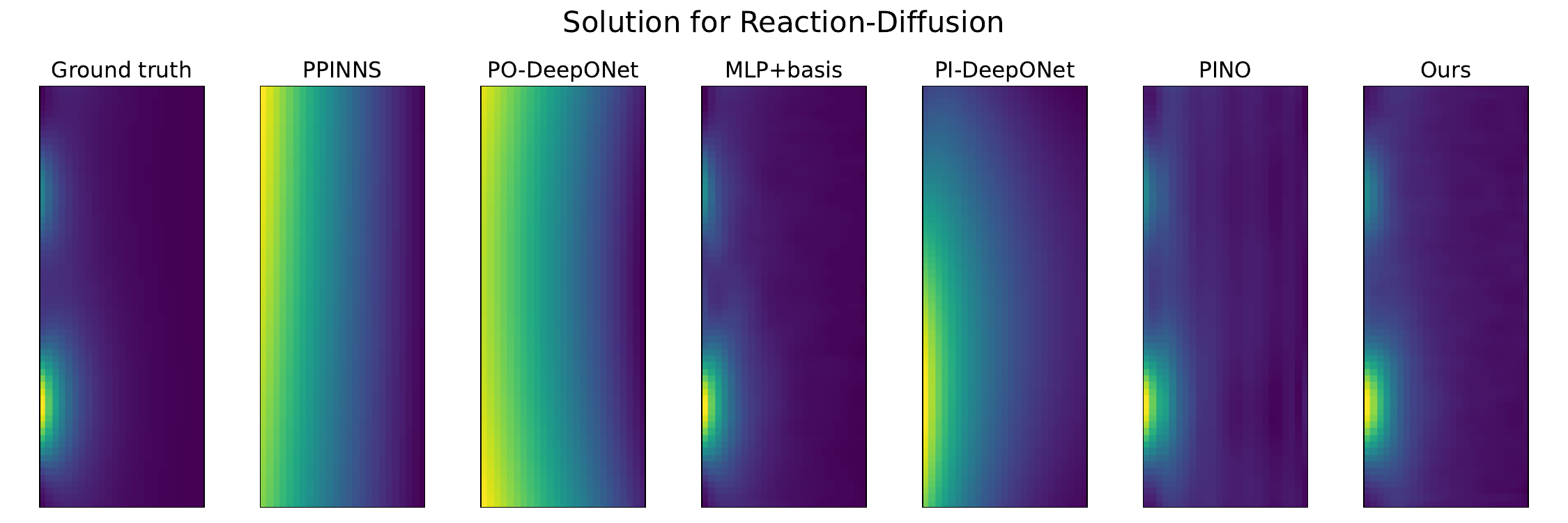}
    \caption{}
    \label{fig:baselinenlrdics1}
\end{subfigure}
\vfill
\begin{subfigure}{\textwidth}
    \centering
    \includegraphics[width=\textwidth]{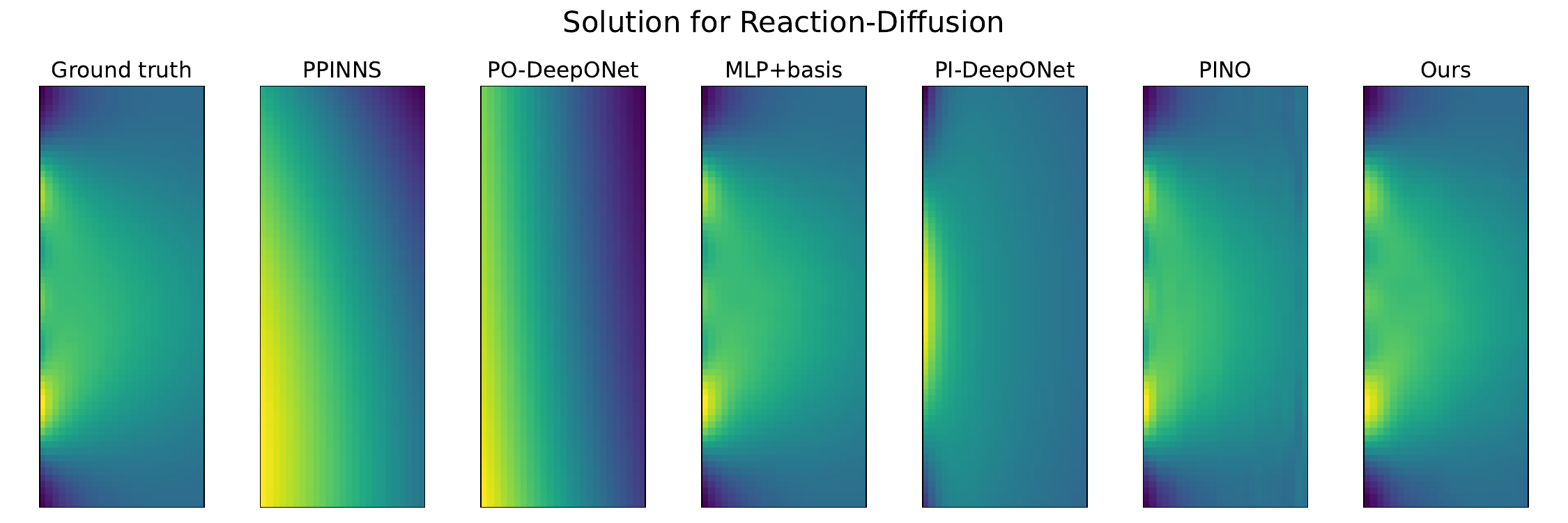}
    \caption{}
    \label{fig:baselinenlrdics2}
\end{subfigure}
\caption{Visual comparison of the solutions for the NLRD with varying IC equation.}
\end{figure}
\vfill
\begin{figure}[htbp]
    \centering
    \includegraphics[width=\textwidth]{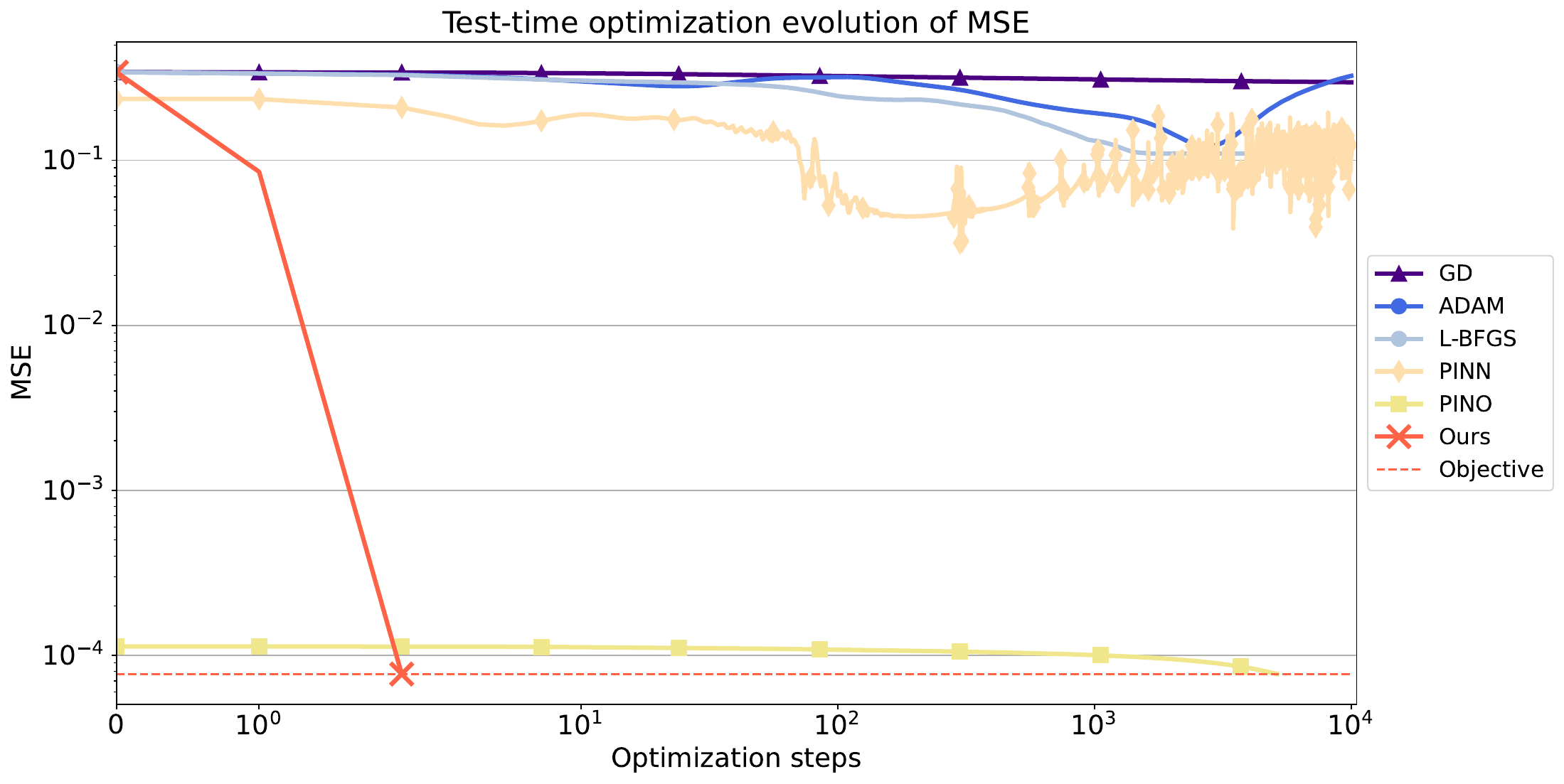}
    \caption{Test-time optimization based on the physical residual loss $\mathcal{L}_{\textnormal{PDE}}$ on \textit{NLRD} with varying IC. }
    \label{fig:testoptnlrdics}
\end{figure}

\end{document}